\documentclass{article}

\usepackage{microtype}

\usepackage{amssymb}
\usepackage{pifont}

\usepackage{graphicx}
\usepackage{subcaption}
\usepackage{booktabs} % for professional tables
\usepackage{tikz}

\usepackage[colorlinks=true]{hyperref}
\usepackage{tabularx,colortbl}
\usepackage{xcolor}

\usepackage[accepted]{icml2026}

\usepackage{url} 
\hypersetup{colorlinks=true,linkcolor=blue,citecolor=blue,linktocpage,urlcolor=blue}

\usepackage{amsmath}
\usepackage{mathtools}
\usepackage{amsthm}
\usepackage{enumitem}
\setlist{nosep} % Applies to all itemize, enumerate, and description

\usepackage[capitalize,noabbrev]{cleveref}

\usepackage{balance}

\theoremstyle{plain}
\newtheorem{theorem}{Theorem}[section]
\newtheorem{proposition}[theorem]{Proposition}
\newtheorem{lemma}[theorem]{Lemma}

\theoremstyle{definition}
\newtheorem{definition}[theorem]{Definition}
\newtheorem{assumption}[theorem]{Assumption}
\theoremstyle{remark}
\newtheorem{remark}[theorem]{Remark}

\newcommand{\R}{\mathbb{R}}

\newcommand{\E}{\mathbb{E}}
\newcommand{\Prob}{\mathbb{P}}

\newcommand{\cM}{\mathcal{M}}

\newcommand{\cT}{\mathcal{T}}

\newcommand{\norm}[1]{\left\| #1 \right\|}

\newcommand{\dimE}{\mathrm{dim}_E}

\newcommand{\argmax}{\mathop{\mathrm{argmax}}}

\newcommand{\muref}{\mu^{\mathrm{ref}}}

\icmltitlerunning{Minimax-Optimal Policy Regret in Partially Observable Markov Games}

\begin{document}
% --- Global vertical-space tightening (proceedings layout) ---
\setlength{\abovedisplayskip}{5pt plus 1pt minus 3pt}
\setlength{\belowdisplayskip}{5pt plus 1pt minus 3pt}
\setlength{\abovedisplayshortskip}{1pt plus 1pt minus 1pt}
\setlength{\belowdisplayshortskip}{3pt plus 1pt minus 2pt}
\setlength{\jot}{4pt}
\setlength{\parskip}{4pt plus 1pt minus 2pt}
\setlength{\textfloatsep}{9pt plus 2pt minus 4pt}
\setlength{\abovecaptionskip}{4pt}

\twocolumn[
  \icmltitle{Minimax-Optimal Policy Regret in Partially Observable Markov Games}

  % It is OKAY to include author information, even for blind submissions: the
  % style file will automatically remove it for you unless you've provided
  % the [accepted] option to the icml2026 package.

  % List of affiliations: The first argument should be a (short) identifier you
  % will use later to specify author affiliations Academic affiliations
  % should list Department, University, City, Region, Country Industry
  % affiliations should list Company, City, Region, Country

  % You can specify symbols, otherwise they are numbered in order. Ideally, you
  % should not use this facility. Affiliations will be numbered in order of
  % appearance and this is the preferred way.
  \icmlsetsymbol{equal}{*}

  \begin{icmlauthorlist}
    \icmlauthor{Raman Arora}{yyy}
    
    %\icmlauthor{}{sch}
    %\icmlauthor{}{sch}
  \end{icmlauthorlist}

  \icmlaffiliation{yyy}{Department of Computer Science, Johns Hopkins University, Baltimore, MD, USA}
  
  \icmlcorrespondingauthor{Raman Arora}{arora@cs.jhu.edu}
  
  % You may provide any keywords that you find helpful for describing your
  % paper; these are used to populate the ``keywords'' metadata in the PDF but
  % will not be shown in the document
  \icmlkeywords{Machine Learning, ICML}

  \vskip 0.3in
]

% this must go after the closing bracket ] following \twocolumn[ ...

% This command actually creates the footnote in the first column listing the
% affiliations and the copyright notice. The command takes one argument, which
% is text to display at the start of the footnote. The \icmlEqualContribution
% command is standard text for equal contribution. Remove it (just {}) if you
% do not need this facility.

% Use ONE of the following lines. DO NOT remove the command.
% If you have no special notice, KEEP empty braces:
\printAffiliationsAndNotice{}  % no special notice (required even if empty)
% Or, if applicable, use the standard equal contribution text:
% \printAffiliationsAndNotice{\icmlEqualContribution}

% ============================================================================

% ============================================================================

\begin{abstract}
We study sequential decision-making in partially observable environments against strategic, adaptive opponents, modeled as partially observable Markov games (POMGs). The central challenge is to learn latent dynamics from partial observations while facing an adversary whose behavior depends on the learner's strategy, making standard regret notions inadequate. We prove that an epoch-based optimistic maximum-likelihood algorithm achieves \(\tilde O(\sqrt{T})\) policy regret for fixed problem parameters, with explicit dependence on the horizon, adversary memory, confidence radius, and the aggregate Eluder dimension of the observable-operator error classes. The algorithm deploys one policy per epoch, with geometrically capped epoch lengths, confidence sets built cumulatively from past data, and a statistical termination test that ends an epoch as soon as the data refute the deployed optimistic model. We also prove a matching lower bound,  and extend the framework to horizon-adaptive guarantees and geometrically fading adversary memory.

\end{abstract}

\vspace{-15pt}
\section{Introduction}
\label{sec:intro}
% ============================================================================
\vspace{-2pt}

Sequential decision-making under uncertainty lies at the heart of reinforcement learning \citep{sutton2018reinforcement, bertsekas1995dynamic}. While remarkable progress has been made in understanding single-agent learning \citep{azar2017minimax, jaksch2010near, jin2018q, dann2017unifying}, many applications of practical importance involve multiple agents whose interests may conflict. A self-driving car navigates among human drivers who adapt their behavior in response to the autonomous vehicle's actions \citep{sadigh2016planning}. An algorithmic trader operates in markets where other participants learn and respond to observed trading patterns \citep{hendershott2011does}. A cybersecurity system defends against attackers who probe and adapt to defensive measures \citep{zhu2015game}. In each case, the learner faces an environment that is not only partially observable but also strategically responsive to the learner's own behavior.

\vspace{-2pt}
These problems go beyond the standard MDP framework \citep{puterman2014markov, bertsekas1995dynamic} in two ways. First, under partial observability the agent cannot directly observe the latent state, requiring belief-based reasoning and history-dependent strategies \citep{kaelbling1998planning, spaan2005perseus}. Second, with an adaptive opponent the environment's response depends on the learner's past actions \citep{littman1994markov, shapley1953stochastic}, introducing counterfactual dependence that standard regret notions do not capture.

\vspace{-2pt}
We develop a unified framework for learning in partially observable Markov games against adaptive adversaries, with the goal of characterizing when efficient learning is possible and providing algorithms with provable guarantees.

\vspace*{-8pt}
\paragraph{The Challenge of Adaptive Adversaries.} Classical online-learning regret benchmarks implicitly assume the environment would behave the same under alternative learner actions \citep{cesa2006prediction, hazan2016introduction}, which is appropriate for oblivious adversaries \citep{auer2002nonstochastic, freund1997decision} but fails under adaptation. \emph{Policy regret} \citep{arora2012bandit} instead evaluates performance against the counterfactual outcome had the learner committed to a fixed policy $\pi^*$ from the start. %, inducing an eventual stationary response $R_\infty(\pi^*)$. 
Sublinear policy regret is impossible against adversaries with unbounded memory, so attention restricts to memory-bounded adversaries, for which \citet{arora2012bandit} achieve sublinear policy regret via mini-batching. Their results are limited to bandit settings, 
and extending them to structured partially observable dynamics requires new ideas.

\vspace*{-8pt}
\paragraph{From POMDPs to POMGs.}
Recent progress on sample-efficient learning in partially observable MDPs (POMDPs) shows that tractability depends on how informative observations are about latent states \citep{liu2022partially, liu2023optimistic}. In particular, \emph{weakly revealing} POMDPs exclude degenerate instances and admit optimistic maximum-likelihood methods with polynomial sample complexity. These approaches rely on observable operator models (OOMs) \citep{jaeger2000observable}, which represent observable dynamics via operators indexed by observation-action pairs, enabling tractable likelihood computation even with large latent state spaces; related spectral methods were developed in \citet{boots2011online, hsu2012spectral}. See Appendix~\ref{sec:related} for extended discussion.

\vspace*{-2pt}
Our work extends this operator-based viewpoint to the game-theoretic setting. In a POMG, transitions depend on both learner and adversary actions, and the adversary's behavior is itself a strategic response to the learner's policy. This coupling creates challenges beyond the single-agent case. The observable operators now entangle world dynamics with the adversary response, and disentangling these components is essential for learning and regret analysis.

\vspace*{-7pt}
\subsection{Our Contributions}
\vspace*{-3pt}

This paper provides a complete theoretical treatment of policy regret minimization in POMGs. % Our contributions can be summarized as follows.

\vspace*{-1pt}
\textbf{Formalization and structural assumptions.} We formalize learning in POMGs
against posterior-Lipschitz adversaries, whose responses vary smoothly with the
learner's policy. This captures structured
strategic behavior while ruling out discontinuous responses to infinitesimal policy
changes. We also introduce a causal decomposition separating world dynamics from adversary behavior, allowing separate likelihood-based estimation of both. 
%We also introduce a causal decomposition separating world dynamics from adversary behavior, which allows likelihood-based world estimation and % posterior-Lipschitz 
%adversary control to be handled separately.

\vspace*{-1pt}
\textbf{Upper bound via epoch-based optimistic MLE.}
Our main algorithmic result establishes that policy regret has the desired \(\sqrt{T}\) dependence up to problem-dependent and logarithmic factors. In particular,
we show that
\vspace*{-3pt}
\[
PR(T)
=
\tilde O\left(H\sqrt{\bar\beta_T d_E T}+mH\log T\right),
\]
\vspace*{-10pt}

for $\bar\beta_T=H\beta_T$.
Here \(H\) is the horizon, \(m\) is the adversary's memory bound, \(d_E\) is the
uniform aggregate Eluder dimension, and \(\beta_T\) is the MLE confidence radius
after \(T\) episodes. The algorithm computes an optimistic MLE confidence set at
the start of each epoch and executes a single optimistic policy throughout the
epoch; epoch lengths are capped geometrically at \(T_e=2^e\), and an epoch is
terminated early as soon as the accumulated data refute the deployed optimistic
model. This statistical termination test makes committing to one policy
per epoch safe, and its warm-up overhead is dominated by the leading term; a
companion proposition shows the test cannot be dropped
(Appendix~\ref{app:necessity}).

\vspace*{-1pt}
\textbf{Matching lower bound.}
We prove that any algorithm must incur \(\Omega(\sqrt{d_E T})\) policy regret, even in a fully revealing POMG with a fixed memoryless adversary. % Thus our upper bound is optimal up to problem-dependent and log factors.

\vspace*{-1pt}
\textbf{Extensions.} 
We show that our algorithm can adapt to unknown time horizons and handle adversaries with unbounded but geometrically fading memory. Finally, we instantiate our bounds for tabular~settings.

% ============================================================================ ============================================================================

\vspace*{-8pt}
\paragraph{Paper Organization.}
\Cref{sec:formulation} presents the POMG model, \Cref{sec:algorithm} the algorithm and upper bound, \Cref{sec:lower} the lower bound, and \Cref{sec:extensions} the extensions. Proofs and supplementary results are deferred to the appendices.
%The paper is organized as follows. \Cref{sec:formulation} introduces the POMG
%model and structural assumptions. \Cref{sec:algorithm} presents the epoch-based
%algorithm and policy-regret upper bound. \Cref{sec:lower} gives the lower bound,
%and \Cref{sec:extensions} covers horizon-adaptivity and fading-memory adversaries.
%Proofs, two summary tables, and an empirical validation are deferred to the appendices.

% ============================================================================
\vspace*{-7pt}
\section{Problem Setup and Structural Results}
\label{sec:formulation}
\vspace*{-3pt}
% ============================================================================

We first define the POMG model and learning objective. 

% We begin by defining partially observable Markov games and the associated learning objective.

% We begin by formally defining partially observable Markov games and the learning objective. Our formulation captures two-player games where a learner faces a strategic opponent, both operating under partial observability.

\vspace*{-7pt}
\subsection{Partially Observable Markov Games}
\label{sec:pomg}
\vspace*{-2pt}

A partially observable Markov game is specified by \(\cM = (S,A,B,O_A,O_B,T,E^A,E^B,r,H,\rho_0)\). Here, $S$ is the state space, $A$ and $B$ 
are the learner's and adversary's action spaces, and $O_A$ and $O_B$ are their observation spaces that reveal partial information about the state. At each step $h\in[H]$, the dynamics are given by transition kernels $T_h:S\times A\times B\to\Delta(S)$ and emission kernels $E_h^A:S\to\Delta(O_A)$ and $E_h^B:S\to\Delta(O_B)$. The learner receives rewards $r_h:O_A\to[0,1]$, $\rho_0\in\Delta(S)$ is the initial state distribution, and all spaces are finite. Throughout, $\log$ denotes the natural logarithm and $\log_2$ the base-2 logarithm.

%A partially observable Markov game is specified by a tuple \(\cM = (S,A,B,O_A,O_B,T,E^A,E^B,r,H,\rho_0)\). Here $S$ denotes the state space, $A$ the learner's action space, and $B$ the adversary's action space. Both players receive partial information about the state through observations. The learner observes signals from $O_A$ and the adversary from $O_B$. The dynamics are governed by transition kernels $T_h\!: S \times A \times B \to \Delta(S)$ and emission kernels $E^A_h\!: S \to \Delta(O_A)$, $E^B_h\!: S \to \Delta(O_B)$ for each step $h \in [H]$, where $H$ is the episode horizon. The learner receives rewards according to $r_h: O_A \to [0,1]$, and $\rho_0 \in \Delta(S)$ specifies the initial state distribution. All component spaces $S, A, B, O_A, O_B$ are finite. %; throughout, $\log$ denotes the natural logarithm and $\log_2$ the base-2 logarithm.

The game proceeds in episodes. At the beginning of each episode, an initial state $s_1$ is drawn from $\rho_0$. At each step $h$, both players simultaneously observe signals $o^A_h \sim E^A_h(\cdot|s_h)$ and $o^B_h \sim E^B_h(\cdot|s_h)$, drawn conditionally independently given $s_h$, then select actions $a_h$ and $b_h$ based on their respective histories. The state transitions according to $s_{h+1} \sim T_h(\cdot|s_h, a_h, b_h)$, and the learner receives reward $r_h(o^A_h)$. The episode ends after $H$ steps.

A learner policy $\pi = (\pi_1, \ldots, \pi_H)$ maps observation-action histories to distributions over actions: $\pi_h: \cT^A_{h-1} \times O_A \to \Delta(A)$, where $\cT^A_h$ denotes the set of learner post-action histories of length $h$, that is, sequences $(o^A_1,a_1,\ldots,o^A_h,a_h)$, with $\cT^A_0$ containing the empty history; the adversary history spaces $\cT^B_h$ are defined analogously, and for either player $P\in\{A,B\}$ we write $\cT^P_{\leq H}:=\bigcup_{h=0}^{H}\cT^P_h$. Similarly, an adversary response \(g=(g_1,\ldots,g_H)\) maps adversary decision
histories to action distributions:
\(g_h\!:\!\cT^B_{h-1}\!\times\! O_B\!\to\!\Delta(B)\). For notational convenience,
we write \(\eta^B_h\!=\!(\tau^B_{h-1},o^B_h)\) and denote the response by
\(g_h(\cdot\!\mid\! \eta^B_h)\). Given a learner policy $\pi$ and adversary response $g$, we write $V^\pi(g)$ for the expected cumulative reward over an episode.

\vspace*{-7pt}
\subsection{Adaptive Adversaries and Policy Regret}
\vspace*{-2pt}
\label{sec:adapt}

We consider adversaries who adapt their behavior based on the learner's recent policies. Formally, an adversary is \emph{$m$-memory bounded} if there exists a fixed, time-invariant map $R: \Pi^m \to \mathcal{G}$, where $\Pi$ and $\mathcal{G}$ denote the learner's and adversary's policy spaces, respectively, such that the adversary's policy in episode $t$ is $g^t = R(\pi^{t-m+1}, \ldots, \pi^t)$. Consequently, once the learner has played a fixed policy $\pi$ for $m$ consecutive episodes, the response equals the \emph{stationary response} $R_\infty(\pi) := R(\pi, \ldots, \pi)$. 
% We consider adversaries who adapt their behavior based on the learner's recent policies. Formally, an adversary is \emph{$m$-memory bounded} if there is a fixed, time-invariant reaction rule $R$ such that its response at episode $t$ depends only on the learner's policies from the most recent $m$ episodes, $g^t = R(\pi^{t-m+1}, \ldots, \pi^t)$. Consequently, once the learner has played a fixed policy $\pi$ for $m$ consecutive episodes, the response equals the \emph{stationary response} $R_\infty(\pi) := R(\pi, \ldots, \pi)$.

The standard notion of external regret compares the learner's performance to the best fixed policy assuming the observed adversary sequence would have been the same. 
This notion is inadequate for adaptive adversaries. Policy regret instead accounts for counterfactual adversary responses.

\begin{definition}[Policy Regret]
The policy regret over $T$ episodes is
\vspace*{-7pt}
\[
PR(T) = \max_{\pi^* \in \Pi} \sum_{t=1}^T \left[V^{\pi^*}(R_\infty(\pi^*)) - V^{\pi^t}(g^t)\right],
\]

\vspace*{-5pt}
\noindent where $\pi^1, \ldots, \pi^T$ are the learner's policies, $g^1, \ldots, g^T$ are the adversary's responses, $R_\infty(\pi^*)$ is the stationary response to $\pi^*$, and $\Pi$ is the learner's policy class, a fixed set of history-dependent policies mapping learner histories to action distributions.
\end{definition}
\vspace*{-2pt}

Policy regret compares the learner to the best fixed policy~under the adversary
response that policy would have induced. It is the natural benchmark for
strategic adaptive opponents.

\vspace*{-7pt}
\subsection{Model Parameterization and Assumptions}
\label{sec:model}
\vspace*{-3pt}

To enable efficient learning, we parameterize the world dynamics by
\(\theta \in \Theta\) and the adversary's response function by
\(\Phi \in \Psi\), under which the stationary response to a deployed policy \(\pi\) selects actions according to a response model \(g^\Phi_h(\cdot\mid\eta^B_h,\pi)\). We write
\(\xi=(\theta,\Phi)\) for the joint parameter and 
\(\Xi=\Theta\times\Psi\) for the joint parameter space. Our algorithm maintains
confidence sets over these parameters based on observed data. Throughout, we write \(\tau_B\) for the adversary's decision history \(\eta^B_h\) defined in Section~\ref{sec:pomg}, dropping the step index when it is immaterial. The key structural
assumption on the adversary is that its response varies smoothly with the learner's
policy. To make this precise, we introduce a decoupled formulation that avoids
subtle circularity issues.

\begin{definition}[Reference Posterior-Predictive]\label{def:sref}
Fix a reference adversary policy $\muref$ independent of the learner's strategy, e.g., uniform over $B$. For a learner policy $\pi$ and fixed world model $\theta \in \Theta$, let $P^{\pi,\muref,\theta}$ denote the joint history law under $\pi$ and $\muref$, and $P^{\pi,\muref,\theta}(\cdot\mid\tau_B)$ its conditional law given the adversary's decision history $\tau_B$. The reference posterior-predictive policy at $\tau_B$ is
\vspace*{-1pt}
\begin{equation}
\label{eq:sref-def}
S_{\tau_B,h}^{\text{ref},\theta}(\pi)
:=
\E_{\tau_A \sim P^{\pi, \muref, \theta}(\cdot|\tau_B)}
\left[\pi_h(\cdot|\tau_A)\right],
\end{equation}
the expected step-$h$ learner action distribution given $\tau_B$ under the reference dynamics. Note that \(\pi\) enters through the conditioning measure, so \(S_{\tau_B,h}^{\text{ref},\theta}\) depends on the whole policy and not only on \(\pi_h\).
\end{definition}

\vspace*{-4pt}
The reference measure \(\muref\) provides a fixed baseline for computing how the
learner's policy appears from the adversary's perspective. This decoupling is
essential because the adversary's actual response depends on the learner's policy,
but the smoothness condition constraining the adversary must be stated in terms of
a quantity that does not itself depend on the adversary. Since
\(S_{\tau_B,h}^{\mathrm{ref},\theta}(\pi)\) depends on the world model \(\theta\),
the condition below must also specify how this dependence on \(\theta\) is handled.

\vspace*{-1pt}
\begin{assumption}[Posterior-Lipschitz Adversary]\label{ass:pl} 
There exists $L \geq 0$ such that, \emph{uniformly over all} $\theta \in \Theta$, for all
policy blocks $\pi^{1:m}:=(\pi^1,\ldots,\pi^m)\in\Pi^m$ and $\nu^{1:m}:=(\nu^1,\ldots,\nu^m)\in\Pi^m$, where $m$ is the adversary's memory bound, all steps $h \in [H]$, and all adversary
histories $\tau_B$:
\vspace*{-2pt}
\begin{eqnarray}
\label{eq:posterior-lipschitz}
\norm{
g_h(\cdot|\tau_B,\pi^{1:m})
-
g_h(\cdot|\tau_B,\nu^{1:m})
}_1 \qquad\qquad\qquad \nonumber\\
\qquad\qquad\qquad
\le
L \max_{i \in [m]}
\norm{
S_{\tau_B,h}^{\text{ref},\theta}(\pi^i)
-
S_{\tau_B,h}^{\text{ref},\theta}(\nu^i)
}_1 .
\end{eqnarray}
Here $g_h(\cdot|\tau_B,\pi^{1:m})$ denotes the reaction rule's response to the
policy block; it determines the response model by stationarity, $g^\Phi_h(\cdot|\eta^B,\pi) = g_h(\cdot|\tau_B,\pi,\ldots,\pi)$.
\end{assumption}

Requiring \eqref{eq:posterior-lipschitz} uniformly over \(\theta\in\Theta\) keeps the adversary class well-defined without knowledge of the true \(\theta^*\) and is the form used in the analysis. %; Remark~\ref{rem:sref-theta} in Appendix~\ref{app:upper} discusses this choice and gives examples of adversaries satisfying the assumption.}
See Remark~\ref{rem:sref-theta} for examples of adversaries satisfying the assumption.
% See Remark~\ref{rem:sref-theta} for further discussion and examples of adversaries satisfying the assumption.

\vspace*{-5pt}
Our second structural assumption concerns the informativeness of observations.
It is stated through the \emph{world channel}, which for each step \(h\) and learner
observation-action~pair \((o^A,a)\) is a linear map
\(W_h^\theta(o^A,a)\colon \mathbb R^{|\bar S^+|}\!\to\!
\mathcal V:=\mathbb R^{|\bar S^+|\times B}\)
determined by the transition and emission kernels of \(\theta\) alone; \(\bar S^+ \!:=\! S \!\times\! \cT^A_{\leq H} \!\times\! \cT^B_{\leq H}\) is the augmented state space pairing a latent state with both players\textquoteright\ histories, and a predictive state is a probability vector over \(\bar S^+\). For a predictive state~\(q\), the \(b\)-indexed component of \(W_h^\theta(o^A,a)q\) is
the unnormalized successor of \(q\) that results if the adversary plays~\(b\). 
Theorem~\ref{thm:decomp} constructs \(W_h^\theta\) explicitly and shows that 
observable dynamics factor through it. The
weak-revealing~condition below requires that discrepancies in this world
channel are detectable from short windows of learner observations, after the
adversary's private observations and actions are marginalized according to the
response model $g^\Phi$.

\begin{assumption}[Multi-step weak revealing]\label{ass:reveal}
There exist an integer \(\kappa\ge 1\) and a constant \(\alpha_\kappa>0\) such that
the following holds. For any two world parameters \(\theta,\theta'\in\Theta\), any
learner policy \(\pi\), any adversary response model \(\Phi\), any step \(h\le H-\kappa\), any
learner observation-action pair \((o^A,a)\), and any normalized predictive state
\(q\), let $\mathcal O^{\theta,\Phi,\pi}_{h+1:h+\kappa}(\cdot\mid q,o^A,a)$
denote the conditional distribution of the next \(\kappa\) learner observations
\(o^A_{h+1},\ldots,o^A_{h+\kappa}\),
generated from \(q\), conditional on observing \(o^A\) and taking action \(a\) at
step \(h\). Then
\vspace*{-3pt}
\begin{align}
\label{eq:weak-revealing}
    \left\|
\mathcal O^{\theta,\Phi,\pi}_{h+1:h+\kappa}(\cdot\mid q,o^A,a)
-
\mathcal O^{\theta',\Phi,\pi}_{h+1:h+\kappa}(\cdot\mid q,o^A,a)
\right\|_1 & \nonumber \\
& \hspace*{-170pt} 
\ge
\alpha_\kappa
\left\|
W_h^\theta(o^A,a)q
-
W_h^{\theta'}(o^A,a)q
\right\|_{1,\mathcal V},
\end{align}

\vspace*{-7pt}
\noindent where \(\|\cdot\|_{1,\mathcal V}\) is the \(\ell_1\) norm on the
intermediate space \(\mathcal V\). Equivalently, world-channel differences
affecting the controlled predictive state are detectable from \(\kappa\)-step
learner-observation windows with signal strength at least
\(\alpha_\kappa\). For the final steps \(h>H-\kappa\) the condition is not
imposed; the standard convention pads the episode with \(\kappa\)
fully revealing dummy steps \citep{liu2022partially}.
\end{assumption}

\vspace*{-5pt}
This condition, adapted from \citet{liu2022partially}, rules out POMGs where observations are so uninformative that exponentially many samples would be needed to identify the dynamics. The parameter $\kappa$ represents the window length needed for observations to become revealing.

% % ============================================================================

% ============================================================================

\vspace*{-8pt}
\subsection{OOM and Causal Decomposition}
\label{sec:oom}
\vspace*{-1pt}

Our analysis builds on the observable operator model (OOM) framework. For a POMG with
parameters $\xi = (\theta, \Phi)$ and learner policy $\pi$, the dynamics of the learner's
observation process can be represented by operators $J^{\xi,\pi}_h(o^A,a) \in \R^{d \times d}$,
where $d := |\bar S^+|$ is the dimension of the augmented state space (Section~\ref{sec:model}), such that
\vspace*{-2pt}
\begin{small}
\[
\Prob^{\xi,\pi}(o^A_1, \ldots,\! o^A_h\! \mid\! a_1,\ldots,\!a_h)
\!=\!
\mathbf{1}^\top\!\!  J^{\xi,\pi}_h\!(o^A_h\!,\! a_h)\! \cdots \!J^{\xi,\pi}_1(o^A_1\!,\! a_1)\, q_0.
\]
\end{small}

\vspace*{-5pt}
\noindent Here $q_0$ is an initial predictive state. The operators describe the \emph{controlled} conditional process, i.e., the probability of the learner's observations given its actions. The adversary is not visible in this representation; its private observations and actions are marginalized out according to the response model $g^\Phi(\cdot\mid\eta^B,\pi)$. 
Learner action probabilities under $\pi$ are handled separately.
A crucial insight is that these operators admit a causal decomposition, a key structural ingredient of our analysis. It separates the world-channel component from the adversary-aggregation component of the observable operator.

% The following POMG-specific decomposition is a key structural ingredient of our analysis. It separates the world-channel component from the adversary-aggregation component of the observable operator.

\begin{theorem}[Causal Decomposition]\label{thm:decomp}
% \textcolor{red}{Consider the augmented state space $\bar{S}^+ = S \times \cT^A_{\leq H} \times \cT^B_{\leq H}$ and the intermediate space $\mathcal{V} = \R^{|\bar{S}^+| \times B}$ of Section~\ref{sec:model}.}
For any adversary response model $g^\Phi_h(\cdot\mid\eta^B_h,\pi)$ and for each step $h$, learner observation $o^A$, and learner
action $a$, there exist linear maps
$W^\theta_h(o^A,a) : \R^{|\bar{S}^+|} \to \mathcal{V},$ and 
$G^{\Phi,\pi}_h(o^A,a) : \mathcal{V} \to \R^{|\bar{S}^+|}$
%\vspace*{-2pt}
%\[
%W^\theta_h(o^A,a) : \R^{|\bar{S}^+|} \to \mathcal{V},~~~~
%G^{\Phi,\pi}_h(o^A,a) : \mathcal{V} \to \R^{|\bar{S}^+|}
%\]
%\vspace*{-7pt}
% \noindent 
such that the controlled observable operator satisfies

\vspace*{-14pt}
\[
J^{\xi,\pi}_h(o^A,a) \;=\; G^{\Phi,\pi}_h(o^A,a) \;\circ\; W^\theta_h(o^A,a).
\]
The map $W^\theta_h$ depends only on the world kernels $(T_h, E^A_h,$ $ E^B_h)$,
while $G^{\Phi,\pi}_h$ depends only on $g^\Phi_h$ and $\pi$, not on $\theta$.
Both maps are entrywise nonnegative, and $G^{\Phi,\pi}_h$ is
$\ell_1$-nonexpansive for arbitrary signed inputs (each output coordinate is a
single input coordinate scaled by a response probability, so
$\|G^{\Phi,\pi}_h u\|_1\le\|u\|_{1,\mathcal V}$).
\end{theorem}

\vspace*{-3pt}
The decomposition holds for any adversary response function. 
% Assumption~\ref{ass:pl} is not needed for the factorization and enters later only to control how \(G^{\Phi,\pi}_h\) varies with \(\pi\). 
The world channel \(W^\theta_h\) maps a
predictive state, using only world kernels, to a \(B\)-indexed family of
hypothetical unnormalized successors, one for each adversary action~\(b\).
The adversary aggregation \(G^{\Phi,\pi}_h\) mixes these successors using
\(g^\Phi_h(\cdot|\eta^B_h,\pi)\). The next augmented state stores the~post-action history
\(\tau^B_h=(\eta^B_h,b_h)\). The proof tracks both players' full
observation-action histories, uses \(\mathcal V\) for type-correctness, and shows
that \(G^{\Phi,\pi}_h\) is \(\theta\)-free. Details in
Appendix~\ref{app:decomposition}.

\vspace*{-7pt}
\subsection{Uniform Eluder Dimension}
\label{sec:eluder}
\vspace*{-4pt}

We next define the complexity measure used in the regret bound, the aggregate
Eluder dimension of the channel error classes introduced below, and record its
scaling for several model families. For a joint parameter \(\xi=(\theta,\Phi)\)
and a learner policy \(\pi\), write \(P^\pi_\xi\) for the law of the learner's
observation-action trajectory when the world follows \(\theta\) and the
adversary plays the stationary response \(g^\Phi(\cdot|\cdot\,,\pi)\) to
\(\pi\), and \(V^\pi(\xi) := \E_{P^\pi_\xi}\bigl[\sum_{h=1}^H
r_h(o^A_h)\bigr]\) for the corresponding value; for the true parameter
\(\xi^*\), \(V^\pi(\xi^*)=V^\pi(R_\infty(\pi))\), the value entering the
policy-regret benchmark. Write \(q^{\xi}_{h-1}\) for the normalized predictive
state under \(\xi\), the conditional distribution over the augmented state
given the learner's observation-action history through step \(h-1\), so that
\(\|q^{\xi}_{h-1}\|_1=1\). The classes below are built from these predictive
states and the observable operators \(J^{\xi,\pi}_h\). %\textcolor{red}{ of Section~\ref{sec:oom}}.

The Eluder dimension, introduced by \citet{russo2013eluder}, measures the sequential
complexity of a function class. For a class \(\mathcal F\) of nonnegative functions
and a scale \(\alpha>0\), say a query \(x\) is \emph{\(\alpha\)-dependent} on a
finite set \(S\) of queries if every \(f\in\mathcal F\) with
\(\sum_{z\in S}f(z)^2\le\alpha^2\) also satisfies \(f(x)\le\alpha\);
\(\dimE(\mathcal F,\alpha)\) is the length of the longest sequence of queries
such that, for some single \(\alpha'\ge\alpha\), every element is
\(\alpha'\)-independent of its predecessors; under this standard convention
\citep{russo2013eluder}, \(\dimE(\mathcal F,\cdot)\) is nonincreasing in the
scale. For our regret
analysis, we track the two channels of the causal factorization
% \(J^{\xi,\pi}_h = G^{\Phi,\pi}_h\circ W^{\theta}_h\) 
(Theorem~\ref{thm:decomp})
separately, through the \emph{channel error classes}
$\mathcal{F}^{W,\xi^*}_{\rm agg}:=\{F^W_\theta:\theta\in\Theta\}$ and
$\mathcal{F}^{G,\xi^*}_{\rm agg}:=\{F^G_\Phi:\Phi\in\Psi\}$, where (writing
\(\E^\pi:=\E_{\tau\sim P^\pi_{\xi^*}}\) for brevity)

\vspace*{-12pt}
\begin{align*}
F^{W}_{\theta}(\pi) &:= \textstyle\sum_{h=1}^H\E^{\pi}\bigl[\|(W^{\theta}_h - W^{\theta^*}_h)\, q^{\xi^*}_{h-1}\|_{1,\mathcal V}\bigr], \\
F^{G}_{\Phi}(\pi) &:= \textstyle\sum_{h=1}^H\E^{\pi}\bigl[\|(G^{\Phi,\pi}_h - G^{\Phi^*,\pi}_h)\, W^{\theta^*}_h q^{\xi^*}_{h-1}\|_1\bigr].
\end{align*}
\vspace*{-12pt}

Here and throughout, operators at step $h$ are evaluated at the realized
observation-action pair of the trajectory, and both classes depend on the true
$\xi^*=(\theta^*,\Phi^*)$ through the reference predictive states
$q^{\xi^*}_{h-1}$ and the trajectory distribution $P^{\pi}_{\xi^*}$. We adopt
\emph{uniform} definitions by taking suprema over $\xi^*$, evaluated at the
resolution scale \(1/(2T)\):
\vspace*{-3pt}
\begin{align*}
d^{W}_{\rm agg} &:= \sup\nolimits_{\xi^*\in\Xi}\dimE\!\left(\mathcal{F}^{W,\xi^*}_{\rm agg}, 1/(2T)\right)\!,\\
d^{G}_{\rm agg} &:= \sup\nolimits_{\xi^*\in\Xi}\dimE\!\left(\mathcal{F}^{G,\xi^*}_{\rm agg}, 1/(2T)\right)\!.
\end{align*}

\vspace*{-4pt}
\noindent 
% Our complexity parameter is the channel sum $d_E := d^{W}_{\rm agg} + d^{G}_{\rm agg}$.
We use the channel sum $d_E := d^{W}_{\rm agg} + d^{G}_{\rm agg}$ as the complexity parameter.  
% The complexity parameter $d_E\! :=\! d^{W}_{\rm agg} \!+\! d^{G}_{\rm agg}$ that we use is the channel sum.
% The complexity parameter of the paper is the channel sum $d_E := d^{W}_{\rm agg} + d^{G}_{\rm agg}$.
The \emph{composed} aggregate error
$F_{\xi}(\pi):=\sum_{h}\mathbb E_{\tau\sim P^\pi_{\xi^*}}[\|(J^{\xi,\pi}_h -
J^{\xi^*,\pi}_h)q^{\xi^*}_{h-1}\|_1]$, which is what the regret decomposition
measures, is controlled by the two channels. Since $G^{\Phi,\pi}_h$ is
$\ell_1$-nonexpansive, 
%(each output coordinate is one intermediate coordinate scaled by a response probability $g^\Phi_h(b\mid\eta^B,\pi)$; Theorem~\ref{thm:decomp}), 
the exact identity $(J^{\xi,\pi}_h-J^{\xi^*,\pi}_h)q = G^{\Phi,\pi}_h\bigl((W^\theta_h-W^{\theta^*}_h)q\bigr)
+ (G^{\Phi,\pi}_h-G^{\Phi^*,\pi}_h)\bigl(W^{\theta^*}_hq\bigr)$ yields, for every policy $\pi$ and model $\xi=(\theta,\Phi)$, the pointwise \emph{channel-triangle} inequalities
\begin{equation}\label{eq:channel-triangle}
F_{\xi}\le F^{W}_{\theta}+F^{G}_{\Phi},
\qquad
F^{G}_{\Phi}\le F_{\xi}+F^{W}_{\theta}.
\end{equation}
The analysis in Appendix~\ref{app:upper} runs the Eluder
summability argument once per channel and combines the two via \eqref{eq:channel-triangle}. Per-step (stepwise) variants of these classes can be defined
analogously; see Remark~\ref{app:stepwise} in Appendix~\ref{app:upper}.

\begin{proposition}[Uniform Eluder Dimension Bounds]\label{prop:eluder}
Consider model classes satisfying Assumptions~\ref{ass:pl} and~\ref{ass:reveal}, using the augmented state space $\bar{S}^+_{\rm red} = S\times O_A^{\kappa-1}\times O_B^{\kappa-1}$ (dimension $d_W = |S|\cdot|O_A|^{\kappa-1}\cdot|O_B|^{\kappa-1}$). Assume that the learner policy class and the adversary response class depend on histories only through the $\kappa$-window statistics retained in $\bar{S}^+_{\rm red}$, so that the causal factorization descends to the reduced space (see the discussion in Appendix~\ref{app:decomposition}). Assume the world model class $\Theta$ and adversary model class $\Psi$ are parameterized such that the observable operators $\xi\mapsto J^{\xi,\pi}_h$ are Lipschitz in $\xi$ with respect to the $\ell_\infty$ parameter norm, and that each model class has a finite $\epsilon$-covering number in operator norm, with logarithm polynomial in the stated parameter counts. Then, 
\vspace*{-4pt}
\begin{itemize}[leftmargin=8pt]
\item For \textbf{tabular} settings, $\!d_E\! =\! \tilde{O}(H(|S|^2 |A| |B| |O_A|^\kappa |O_B|^\kappa \!\!$ $+ d_{\rm adv} |B|))$, where $d_{\rm adv}$ is the feature dimension of the (linear) adversary-response model. 
\item For \textbf{linear world model}, $d_E \!=\! \tilde{O}(H(d_{\rm lin} |O_A|^\kappa |O_B|^\kappa +$ $\! d_{\rm adv} |B|))$, where $d_{\rm lin}$ is the linear dimension.
\item For \textbf{low-rank transitions}, $d_E\!=\! \tilde{O}(H(r^2(|A|{+}|B|) |O_A|^\kappa$ $ |O_B|^\kappa+d_{\rm adv} |B|))$, where $r$ is the rank. 
\end{itemize}
\end{proposition}

\vspace*{-7pt}
The factor \(H\) arises from summing the per-step span dimensions over
\(h\in[H]\). The factor \(|O_B|^\kappa\) arises because the augmented representation
must retain enough adversary-side observation information to describe the adversary
aggregation operator \(G^{\Phi,\pi}_h\) % (defined in Section~\ref{sec:oom})
independently of the world parameter \(\theta\). For the linear and low-rank
families we additionally assume the parameterization is
\emph{coordinate-structured}, i.e., each entry of $W^\theta_h$ is an affine function
of the parameters, so that the world-channel difference family
$\{W^\theta_h-W^{\theta'}_h\}$ spans a subspace of the stated dimension.
Bilinear parameterizations with unknown factor interactions need not satisfy
this (their difference families can span the full ambient space), and the
corresponding bounds then fail.

\vspace*{-8pt}
\begin{proof}[Proof sketch]
Since \(d_E=d^{W}_{\rm agg}+d^{G}_{\rm agg}\) by definition, it suffices to
bound each channel separately; the channel-triangle inequalities \textcolor{red}{\eqref{eq:channel-triangle}} let the regret
analysis consume the two bounds directly (Appendix~\ref{app:upper}), so the
composed class never needs its own complexity bound. The argument for each
channel has two steps. First, each channel error class consists of \(\ell_1\) norms of linear images
of the parameter. For the world channel, the operator difference
\(W^\theta_h-W^{\theta^*}_h\) is a fixed matrix, independent of the query
policy, and each entry is an affine function of the transition and emission
parameters; on the reduced window space the difference family therefore lies in
a span of dimension \(O(|S|^2|A||B|\cdot|O_A|^\kappa|O_B|^\kappa)\) per step
(the single-agent analogue of this count is in \citet{liu2022partially}). For
the adversary channel, \(G^{\Phi,\pi}_h\) is linear in \(\Phi\), so
\(\|(G^{\Phi,\pi}_h-G^{\Phi^*,\pi}_h)W^{\theta^*}_hq\|_1
=\|U_h\,\mathrm{vec}(\Phi-\Phi^*)\|_1\) for a design matrix \(U_h\)
computable from the query, with \((|B|-1)\,d_{\rm adv}\) free parameters per
step.

Second, the sign-witness lemma (Lemma~\ref{lem:sign-witness},
Appendix~\ref{app:eluder-lemma}) bounds the Eluder dimension of any such class
by its span dimension times logarithm of the evaluation scale. 
% ; part (a), \textcolor{red}{applied in expectation form as in the proof of Proposition~\ref{prop:eluder} (Appendix~\ref{app:eluder-lemma})}, covers the fixed-design world channel, and part (b) covers the query-dependent adversary design.
Summing the per-step span dimensions over \(h\in[H]\) and adding the two
channels gives the tabular bound. Under the coordinate-structure hypothesis,
the linear and low-rank families replace the world-channel span count
\(|S|^2|A||B|\) by \(d_{\rm lin}\) and \(r^2(|A|+|B|)\) respectively; the
adversary channel is unchanged. The full proof is in
Appendix~\ref{app:eluder-lemma}. %, with the reduced-space preliminaries in
%Appendix~\ref{app:decomposition}.
\end{proof}

% ============================================================================
\vspace*{-15pt}
\section{Algorithm and Upper Bound}
\label{sec:algorithm}
\vspace*{-3pt}
% ============================================================================

We now present our main algorithm and establish its policy regret guarantee. Algorithm~\ref{alg:omle} follows the optimistic MLE paradigm, maintaining confidence sets over the joint parameter space and selecting policies that are optimal under the most favorable parameters. 
%\vspace*{-7pt}
%\subsection{Algorithmic Framework}
%\vspace*{-4pt}
The algorithm divides episodes into epochs of geometrically increasing length
\(T_e=2^e\). Each epoch uses a single policy selected optimistically from a
confidence set built cumulatively from all previous epochs' data. 
This structure
keeps the number of \emph{scheduled} policy switches logarithmic in \(T\), with
data-dependent terminations adding a controlled number more
(Lemma~\ref{lem:switch-count}); limiting switches matters because the
adversary's response changes with the learner's policy and every switch costs
warm-up episodes. In epoch \(e\), the algorithm proceeds as follows:
%, with steps 1--3 at the epoch start and steps 4--5 during the epoch:

\vspace*{-3pt}
\begin{enumerate}[leftmargin=12pt,itemsep=1pt]%,labelsep=2pt]

\item \textbf{MLE computation:} Compute the maximum likelihood estimate
\(\hat{\xi}_{e-1}\in\Xi\) of the joint parameter over all trajectories in the cumulative dataset
\(\mathcal D_{e-1}\), which contains data from epochs \(1,\ldots,e-1\).

\item \textbf{Confidence set construction:} Form the confidence set $\mathcal{C}_{e-1}$ consisting of all parameters $\xi$ whose log-likelihood on $\mathcal{D}_{e-1}$ is within $\beta_e$ of the MLE, where $\beta_e = c(\log \mathcal{N}(\epsilon_e;\Xi) + \log(1/\delta_e) + mB_0)$, with $c$ an absolute constant ($c\ge16$ suffices), $\mathcal{N}(\epsilon;\Xi)$ the $\epsilon$-covering number of the model class in operator norm~(\ref{ass:reg}(b)), $B_0$ the bound on per-trajectory log-likelihood ratios~(\ref{ass:reg}(c)), discretization scale $\epsilon_e = 2^{-\lceil\log_2(|\mathcal D_{e-1}|+2)\rceil}\in[\tfrac{1}{2}\cdot\tfrac{1}{|\mathcal D_{e-1}|+2},\,\tfrac{1}{|\mathcal D_{e-1}|+2}]$ (the dyadic grid keeps the candidate covering nets deterministic; see Lemma~\ref{lem:conc}), and failure budget $\delta_e = \delta/(2e^2)$. All three quantities are keyed to the data volume and the epoch counter, so the confidence radius requires no knowledge of $T$; the $mB_0$ floor is what prices policy switches (Lemma~\ref{lem:switch-count}).

\item \textbf{Optimistic planning:} Select the policy-parameter pair $(\pi_e, \xi_e)$ that maximizes the expected value within the confidence set $\mathcal{C}_{e-1}$. This optimism drives exploration and is the mechanism behind the no-regret
guarantee. The selected policy \(\pi_e\) is used 
throughout the epoch.

\item \textbf{Warm-up and execution with statistical termination:} Execute
policy \(\pi_e\) for \(m-1\) warm-up episodes, discarding data to allow the
adversary to stabilize; then execute \(\pi_e\) for at most \(T_e\)
data-collection episodes. After every data-collection episode \(t\), recompute the
maximum-likelihood estimate on all data collected so far; if the deployed model
falls more than \(\tau_e(t) := 2\beta_e(t) + C_0B_0\log(2t(t{+}1)/\delta_e)\)
below it in log-likelihood, \emph{terminate the epoch immediately}, retaining
the epoch's data. Here \(\beta_e(t)\) is the radius of step 2 recomputed at
the current dataset size (on the same dyadic grid) and confidence budget
\(\delta_e/(2t(t{+}1))\), and
\(C_0\) is an absolute constant from the martingale concentration bound behind
the test (Appendix~\ref{app:upper}). Let \(n_e\le T_e\) denote the number of
data-collection episodes the epoch actually runs. The deviation-aware
threshold separates statistical fluctuation from genuine misfit.

\item \textbf{Dataset update:} Add the \(n_e\le T_e\) data-collection trajectories
(including those of terminated epochs) to the cumulative dataset:
$\mathcal D_e
=
\mathcal D_{e-1}
\cup
\{(\pi_e,\tau_1^{(e)}),\ldots, (\pi_e,\tau_{n_e}^{(e)})\}.$

\end{enumerate}

\vspace*{-3pt}
The key feature distinguishing this approach is that we use
\emph{one policy per epoch}, with epochs ending either at the geometric cap
\(T_e=2^e\) or at the statistical termination test. The learner thus deploys
\(S_T = O(\log T) + \#\{\text{terminations}\}\) policies, one per epoch, over
\(T\) data-collection episodes, and we show
(Lemma~\ref{lem:switch-count}) that terminations are rare enough for their
total warm-up cost to be dominated by the leading regret term. This structure
serves several purposes:

\vspace*{-3pt}
\textbf{Cumulative learning:} Confidence sets are built from \emph{all historical
data}, not just the current epoch. Uniform weighting of these trajectories enables
standard concentration and Eluder dimension arguments. This ensures that information accumulates across epochs, with later epochs benefiting from increasingly precise parameter estimates.

\vspace*{-3pt}
\textbf{Statistical termination:} The termination test is what makes
committing to a single optimistic policy safe. Optimism deliberately favors
directions the data have not yet constrained, and without the test a single
ill-informed commitment can run unchecked for \(\Theta(T)\) episodes
(Prop.~\ref{prop:fixed-fails}); with it, each epoch's contribution is
capped at \(\tilde O(H\beta_T)\) (Lemma~\ref{lem:epoch-cap}), with \(\beta_T\)
the confidence radius at horizon \(T\) as in Theorem~\ref{thm:upper}. The test
is the maximum-likelihood analogue of the count-doubling and
determinant-doubling triggers in low-switching-cost RL
\citep{bai2019provably, qiao2022sample}, though here rare switching is also
forced by the adversary, since every policy change costs \(m-1\) warm-up
episodes.

\vspace*{-3pt}
\textbf{Automatic adaptation:} Each epoch \(e\) sets its length based only on
\(T_e=2^e\), requiring no advance knowledge of the final horizon \(T\). The same
analysis applies to any final horizon without modifying the algorithm.

\vspace*{-3pt}
\textbf{Geometric efficiency:} The data-collection episodes satisfy
\(\sum_{e} n_e\le T\) by definition of the horizon, and at most
\(\log_2(2T)\) epochs end by exhausting their caps, since a cap-hit at index
\(e\) requires \(2^e\le T\); every further epoch ending, except possibly the final horizon-truncated
one, is a statistical termination, priced by Lemma~\ref{lem:switch-count}. The capped summability
bound controls prediction errors over all epochs, and Cauchy--Schwarz yields
the \(\sqrt T\) dependence.

\begin{algorithm}[t]
\begin{small}
\caption{Epoch-Based Optimistic MLE for POMGs}
\label{alg:omle}
\begin{algorithmic}[1]
\REQUIRE Model class $\Xi$, policy class $\Pi$, confidence level $\delta$, memory bound $m$, increment bound $B_0\ge1$ (Assumption~\ref{ass:reg}(c)); absolute constants $c\ge16$, $C_0$
\STATE Initialize cumulative dataset $\mathcal{D}_0 \leftarrow \emptyset$
\FOR{epoch $e = 1, 2, \ldots$}
    \STATE Set epoch length $T_e \leftarrow 2^e$
    \STATE Set $\delta_e \leftarrow \delta/(2e^2)$, ~$\epsilon_e \leftarrow 2^{-\lceil\log_2(N_{e-1}+2)\rceil}$ where $N_{e-1} \leftarrow |\mathcal D_{e-1}|$, ~$\beta_e \leftarrow c\bigl(\log \mathcal{N}(\epsilon_e;\Xi) + \log(1/\delta_e) + mB_0\bigr)$
    \hfill\COMMENT{keyed to data volume and epoch counter; no knowledge of $T$ required}
    \STATE \textbf{MLE Estimation:} 
    $\hat{\xi}_{e-1} \leftarrow \argmax_{\xi \in \Xi} L_{e-1}(\xi)$ where 
    $L_{e-1}(\xi)=\sum_{(\pi, \tau) \in \mathcal{D}_{e-1}} \log P^\pi_\xi(\tau)$
    \STATE \textbf{Confidence Set:} $\mathcal{C}_{e-1} \leftarrow \{\xi \in \Xi : L_{e-1}(\xi) \geq L_{e-1}(\hat{\xi}_{e-1}) - \beta_e\}$
    \\
    \STATE \textbf{Optimistic Planning:} \\
    $(\pi_e, \xi_e) \leftarrow \argmax_{(\pi,\xi) \in \Pi \times \mathcal{C}_{e-1}} V^\pi(\xi)$
    \STATE \textbf{Warm-up:} Execute $\pi_e$ for $m-1$ episodes (discard data)
    \STATE \textbf{Execution with termination test:} $n_e \leftarrow T_e$; \textbf{for} $t=1,\ldots,T_e$: execute $\pi_e$, collect $\tau_t^{(e)}$; recompute the MLE $\hat\xi_t$ on $\mathcal D_{e-1}\cup\{(\pi_e,\tau_1^{(e)}),\ldots,(\pi_e,\tau_t^{(e)})\}$; \textbf{if} $L_t(\xi_e) < L_t(\hat\xi_t) - \tau_e(t)$ \textbf{then} set $n_e \leftarrow t$ and end the epoch, where $\tau_e(t) := 2\beta_e(t) + C_0B_0\log\bigl(2t(t{+}1)/\delta_e\bigr)$ and $\beta_e(t)$ is the radius of line 4 recomputed with scale $2^{-\lceil\log_2(|\mathcal D_{e-1}|+t+2)\rceil}$ and budget $\delta_e/(2t(t{+}1))$
    \hfill\COMMENT{statistical termination}
    \STATE \textbf{Update Dataset:} $\mathcal{D}_e \leftarrow \mathcal{D}_{e-1} \cup \{(\pi_e, \tau_1^{(e)}),\ldots,$ $  (\pi_e, \tau_{n_e}^{(e)})\}$
\hfill\COMMENT{data-collection trajectories only, terminated epochs included}
\ENDFOR
\end{algorithmic}
\end{small}
\end{algorithm}

%\vspace*{-3pt}
%Algorithm~\ref{alg:omle} is an oracle-style statistical algorithm. It assumes access to a realizable model class \(\Xi=\Theta\times\Psi\) and to optimization oracles for MLE and optimistic planning over this class. Thus our results should be viewed as sample-complexity and regret guarantees under model-class realizability, rather than as computational efficiency guarantees.

\vspace*{-7pt}
\subsection{Main Result: Upper Bound}
\vspace*{-4pt}

In addition to Assumptions~\ref{ass:pl} and~\ref{ass:reveal}, the proof uses the
following statistical regularity conditions on the model class. These are standard
for the finite/tabular classes considered in Section~\ref{sec:extensions} under
positivity, smoothing/discretization, and the weak-revealing conversion described
below.

\vspace*{-3pt}

\begin{assumption}[Statistical Regularity]\label{ass:reg}
% The joint parameter space 
$\!\Xi \!= \Theta \times \Psi$ satisfies:
\vspace*{-12pt}
\begin{enumerate}[label=\alph*),leftmargin=10pt,itemsep=0pt]%,labelsep=1pt]
\item \textbf{Well-specification:} $\!\!$The true parameters $\xi^*\! =\! (\theta^*,\!\Phi^*) \!\in\! \Xi$.
\vspace*{-10pt}

\item \textbf{Finite covering:} For every $\epsilon > 0$, the $\epsilon$-covering number $\mathcal{N}(\epsilon;\Xi)$ of $\Xi$ in operator norm is finite; $\log \mathcal{N}(\epsilon;\Xi)$ grows at most polynomially in $1/\epsilon$ and the problem dimensions. Moreover, per-trajectory log-likelihoods $\xi \mapsto \log P^\pi_\xi(\tau)$ are Lipschitz in this norm, uniformly over policies and trajectories (given the positivity in (c)).
\vspace*{2pt}

\item \textbf{Bounded log-likelihood increments:} Log-likelihood ratios $\log(P^\pi_\xi(\tau)/P^\pi_{\xi^*}(\tau))$ are almost surely bounded by a constant $B_0$ (which may depend on the smoothing or discretization scale of the model class, and requires all trajectory probabilities to be bounded below, e.g., via smoothing). In particular, all trajectory probabilities are uniformly bounded below over $\Xi$; we refer to this property as \emph{positivity}. We normalize $B_0$ so that $B_0\ge1$ and $B_0$ also dominates the Lipschitz constant in (b); this is without loss of generality, since only an upper bound on either quantity enters the analysis. 
\vspace*{2pt}

\item \textbf{Well-defined predictive states:} For every $\xi\in\Xi$, policy $\pi$, and trajectory prefix with positive probability under $\xi^*$, the prefix probability under $\xi$ is positive, so the normalized predictive state $q^{\xi}_{h-1}$ is well-defined and satisfies $\|q^{\xi}_{h-1}\|_1 = 1$; this follows from the positivity in (c). 
\vspace*{2pt}

\item \textbf{KL-to-operator conversion (stepwise):} For all $\xi=(\theta,\Phi)\in\Xi$ and policies $\pi\in\Pi$, both
\vspace*{-4pt}
\[
\textstyle 
\mathrm{KL}(P^\pi_{\xi^*}\|P^\pi_\xi)
\!\geq\!
c_{\mathrm{KL}} \!\sum_{h=1}^H \mathbb{E}\!\left[\|J^{\xi,\pi}_h q^{\xi^*}_{h-1} \!-\! J^{\xi^*,\pi}_h q^{\xi^*}_{h-1}\|_1^2\right]
\]

\vspace*{-2pt}
\noindent\emph{and} the world-channel form
\[
\textstyle 
\mathrm{KL}(P^\pi_{\xi^*}\|P^\pi_\xi)
\geq
c_{\mathrm{KL}} \sum_{h=1}^H \mathbb{E}\!\left[\|(W^{\theta}_h - W^{\theta^*}_h) q^{\xi^*}_{h-1}\|_{1,\mathcal V}^2\right]\!\!,
\]

\vspace*{-1pt}
\noindent where $c_{\mathrm{KL}}$ depends on $\alpha_\kappa$. The first clause is the standard stepwise conversion proved for POMDPs via weak revealing
\citep{liu2022partially}, imposed on the joint-operator class (well-defined by
Theorem~\ref{thm:decomp}); the second is the quantitative content of
Assumption~\ref{ass:reveal}, which makes world-channel discrepancies detectable
in $\kappa$-step observation windows with strength $\alpha_\kappa$, so that a
chain-rule/Pinsker argument converts window detectability into trajectory KL.
For tabular classes both clauses are verified this way, with
$c_{\mathrm{KL}}=\Omega(\alpha_\kappa^2/\kappa)$, using the positivity in
(c). Together the two clauses also control the adversary channel; by the pointwise
channel-triangle inequality \eqref{eq:channel-triangle} together with
$(a+b)^2\le2a^2+2b^2$,
$\mathrm{KL}(P^\pi_{\xi^*}\|P^\pi_\xi)\ge (c_{\mathrm{KL}}/4)\sum_h
\mathbb{E}[\|(G^{\Phi,\pi}_h-G^{\Phi^*,\pi}_h)W^{\theta^*}_hq^{\xi^*}_{h-1}\|_1^2]$. Under this joint conversion no posterior-Lipschitz transport cost enters
the analysis; see Remark~\ref{rem:transport-role} (and
Lemma~\ref{lem:transport} for the weakened, world-channel-only variant).
\vspace*{2pt}

\item \textbf{Exact warm-up stationarity:} The reaction rule \(R\) is
time-invariant (as in Section~\ref{sec:adapt}); consequently, after
\(m-1\) warm-up episodes of policy \(\pi\), the \(m\)-memory adversary's
response is exactly \(R_\infty(\pi)\) for all data-collection episodes in that
epoch, with no residual transient.

\end{enumerate}
\end{assumption}

\vspace*{-3pt}
\begin{remark}
For finite/tabular classes, these conditions hold under standard positivity, namely
trajectory probabilities are bounded below after smoothing or discretization if
needed. Condition (e) is a
stepwise KL-to-operator bound; applying Cauchy--Schwarz over \(h\) gives $\Delta_e^2
=
\left(\sum_h \mathbb E[\|\cdots\|_1]\right)^2
\le
H\sum_h \mathbb E[\|\cdots\|_1^2],$ which yields the \(H\beta_T\) term in
\(\bar\beta_T=H\beta_T\). A stronger aggregate version of (e) would
remove this factor, giving \(\bar\beta_T=\beta_T\); see
Remark~\ref{rem:aggregate-kl}.
\end{remark}

\begin{theorem}[Policy Regret Upper Bound]\label{thm:upper}
Under Assumptions~\ref{ass:pl}, \ref{ass:reveal}, and~\ref{ass:reg}, Algorithm~\ref{alg:omle} achieves
\vspace*{-3pt}
\[
\textstyle
PR(T) \leq C(H\sqrt{\bar\beta_T \cdot d_E T}\,\log T + mH\, S_T)
\]
with probability at least $1\!-\!2\delta$, where $\!d_E \!=\! d^{W}_{\rm agg}\!+\!d^{G}_{\rm agg}$
is the channel-sum aggregate Eluder dimension at scale $1/(2T)$, $\bar\beta_T := H\beta_T$ with
$\beta_T=O(\log \mathcal{N}(\epsilon_T;\Xi)+$
$\log T+\log(1/\delta)+mB_0)$ for $\epsilon_T := 1/(2T+4)$,
and $S_T$ is the number of policies deployed by the algorithm, which satisfies
$S_T \le 1 + \log_2(2T) + \tilde O\bigl(\tfrac{B_0}{1+mB_0}\sqrt{d_E\,\bar\beta_T\,T}\bigr)$
(Lemma~\ref{lem:switch-count}). Consequently $PR(T) = \tilde O\bigl(H\sqrt{\bar\beta_T d_E T} + mH\log T\bigr)$, with no condition relating $m$, $B_0$, and $T$. 
\end{theorem}

\vspace*{-2pt}
The theorem gives the desired $\sqrt{T}$ dependence on the number of episodes.
The effective confidence parameter $\bar\beta_T = H\beta_T$ arises from applying
Cauchy--Schwarz over $H$ steps to convert the stepwise KL-to-operator bound
(Assumption~\ref{ass:reg}(e)) into a past-consistency radius for the aggregate
class; under the stronger aggregate KL conversion discussed in
Remark~\ref{rem:aggregate-kl}, the $H\beta_T$ factor improves to $\beta_T$.
No transport term enters $\bar\beta_T$, even though the observable operators
depend on the deployed policy; Remark~\ref{rem:transport-role} explains why
this dependence costs nothing here. The $mB_0$ floor in $\beta_e$ is what buys the unconditional domination of the warm-up term. The constant $C$ depends on $\alpha_\kappa$ (through $c_{\rm KL}$) and on $B_0$, with additional logarithmic dependence on $|B|$ and $H$; the posterior-Lipschitz constant $L$ enters only through $d_E$ and $\mathcal N(\cdot;\Xi)$. This rate is optimal in its \(T\) and
\(d_E\) dependence, as shown in Section~\ref{sec:lower}. The full proof is in Appendix~\ref{app:upper}, and an extended sketch in
Appendix~\ref{app:sketch}; here we record the skeleton in four steps.

% \begin{proof}[Proof sketch]
%\vspace*{-7pt}
%\subsection{Proof Sketch}
%\label{sec:sketch}
\vspace*{-1pt}
\emph{Optimism reduces regret to simulation error.} After its warm-up
episodes, an \(m\)-memory adversary plays its stationary response, so
\(V^{\pi_e}(\xi^*)=V^{\pi_e}(R_\infty(\pi_e))\) throughout epoch \(e\)'s
data collection. On the confidence event \(\xi^*\in\mathcal C_{e-1}\),
optimism gives
\(V^{\pi^*}(\xi^*)-V^{\pi_e}(\xi^*)\le V^{\pi_e}(\xi_e)-V^{\pi_e}(\xi^*)\),
and the simulation lemma bounds the right side by \(H\Delta_e\), where
\(\Delta_e\) is the aggregate observable-operator prediction error of the
optimistic model along the trajectories of \(\pi_e\).

\emph{Likelihood fit controls past errors.} Membership
\(\xi_e\in\mathcal C_{e-1}\) means the optimistic model nearly matches the
true likelihood on all past data; the KL-to-operator conversion
(\ref{ass:reg}(e)), available under weak revealing, turns this into
small prediction error at the historical policies
\(\pi_1,\ldots,\pi_{e-1}\).

\emph{The termination test caps fresh errors.} Past consistency says nothing
about \(\Delta_e\) at the newly chosen \(\pi_e\), and a fixed epoch
schedule can commit to an ill-informed policy for \(\Theta(T)\) episodes
(Proposition~\ref{prop:fixed-fails}). If \(\Delta_e\) is large, the epoch's
own data refute the deployed model at a per-episode rate proportional to the
squared stepwise error, so the test fires; hence every epoch, terminated or
complete, satisfies \(n_e\Delta_e^2\le\tilde O(H\beta_T/c_{\rm KL})\)
(Lemma~\ref{lem:epoch-cap}).

\emph{Eluder summability assembles the bound.} Past consistency at radius
\(\bar\beta_T=H\beta_T\) and the per-epoch cap feed the capped summability
argument, giving \(\sum_e n_e\Delta_e^2\le\tilde O(d_E\bar\beta_T)\)
channel by channel (Lemma~\ref{lem:eluder-total});
Cauchy--Schwarz with \(\sum_e n_e\le T\) then yields
\(\sum_e n_eH\Delta_e\le\tilde O(H\sqrt{\bar\beta_T d_E T})\), and the
warm-up cost of \((m-1)H\) per epoch over \(S_T\) epochs is controlled by
the switch-count lemma, cap-hit epochs being at most \(\log_2(2T)\) and each
termination being priced in statistical evidence
(Lemma~\ref{lem:switch-count}).

% \end{proof}

% \vspace*{-6pt}
% \paragraph{Proof idea.} Optimism reduces epoch-\(e\) regret to the
% simulation error \(H\Delta_e\) of the optimistic model at the deployed
% policy. Likelihood consistency converts, via Assumption~\ref{ass:reg}(e),
% into small prediction error at past policies, while the statistical
% termination test caps each epoch's fresh error
% (Lemma~\ref{lem:epoch-cap}); Eluder summability then bounds
% \(\sum_e n_e\Delta_e^2\) by \(\tilde O(d_E\bar\beta_T)\) and
% Cauchy--Schwarz gives the rate, with warm-up priced by the switch count.
% An extended sketch is in Appendix~\ref{app:sketch} and the full proof in
% Appendix~\ref{app:upper}.
% % ============================================================================
\vspace*{-10pt}

\section{Lower Bound}%: Hard Instance Construction}
\label{sec:lower}
\vspace*{-3pt}
% ============================================================================

%\vspace*{-5pt}
We now establish that the upper bound in Theorem~\ref{thm:upper} is optimal in its \(\sqrt{T}\) and aggregate-Eluder-dimension dependence, up to
problem-dependent and logarithmic factors. 
Our lower bound demonstrates that even under highly favorable conditions, i.e., fully
revealing observations, a memoryless adversary, and a smooth Lipschitz response
class, any algorithm must incur \(\Omega(\sqrt{d_E T})\) policy regret.

\vspace*{-6pt}
\subsection{Hard Instance Construction}
\vspace*{-2pt}

We embed the standard $d$-armed stochastic bandit in the simplest possible POMG, one in which observations fully reveal the state, the adversary is trivial, and the horizon is $H=2$, so that all remaining difficulty is statistical. The learner's first action routes the game to a latent state, the observation there carries a Bernoulli reward signal, and identifying the best route is exactly the embedded bandit problem. The family is indexed by the location $i^*\in[d]$ of the best route, with gap $\Delta = c_0\sqrt{d/T}$ for a small constant $c_0>0$. The blocks below give the construction.

%\vspace*{-2pt}
%\textcolor{red}{The lower bound comes from a bandit embedded in the simplest possible POMG, one in which observations fully reveal the state, the adversary is trivial, and the horizon is $H=2$, so that all remaining difficulty is statistical. The learner's first action routes the game to a latent state, the observation there carries a Bernoulli reward signal, and identifying the best route is exactly a $d$-armed bandit problem. The family is indexed by the location $i^*\in[d]$ of the best route, with gap $\Delta = c_0\sqrt{d/T}$ for a small constant $c_0>0$. The blocks below give the construction.}

\vspace*{-2pt}
\textbf{Dynamics.} The POMG has $d$ latent states $s_1,\ldots,s_d$ and $d$ learner actions $a_1,\ldots,a_d$. The episode starts in a fixed, known state (say $s_1$; step 1 serves only to route the learner's choice). At step $h\!=\!1$, the learner observes a fixed observation $o_0$ (the same in all states, so $h\!=\!1$ carries no reward), chooses action $a_i$, and the state transitions deterministically to $s_i$. At step $h\!=\!2$, in state $s_i$ the learner observes $o=(i,y)$ where $y\sim\mathrm{Bernoulli}(\mu_i)$ with $\mu_{i^*}\!=\!1/2\!+\!\Delta$ and $\mu_i=1/2$ for $i\!\neq\! i^*$; reward $r_2(i,y)\!=\!y$, and $r_1\!=\!0$. The first coordinate \(i\) of \(o\) reveals the latent state at step \(h\!=\!2\);
together with the fixed and known initial state at step \(h\!=\!1\), this makes the
POMG fully revealing with \(\kappa\!=\!1\) and a constant signal parameter independent
of \(\Delta\). 

\vspace*{-1pt}
\textbf{Adversary.} There is a single adversary action $b_{\text{coop}}$, played in every episode, so the adversary is memoryless ($m=1$) and its response $g \equiv b_{\text{coop}}$ satisfies Assumption~\ref{ass:pl} with $L=0$.

\vspace*{-1pt}
\textbf{Policy regret.} The best fixed policy $\pi^* = a_{i^*}$ achieves value $1/2+\Delta$ per episode. Any algorithm's policy regret equals its bandit regret against arm $a_{i^*}$.

\vspace*{-2pt}
\textbf{Model class.} Define the ambient class as $\{\mu \in [1/2,\, 1/2+\Delta]^d\}$ where $\mu_i$ is the mean reward of arm $a_i$, and take the learner policy class to be the $d$ deterministic routing policies (the learner may still randomize its choices across episodes). The hard instances form the subclass $\mu_{i^*} = 1/2+\Delta$, $\mu_i = 1/2$ for $i\neq i^*$. With this policy class the ambient class has Eluder dimension $\Theta(d)$ whenever $dT$ exceeds an absolute constant (Appendix~\ref{app:lower}); the lower-bound subclass is contained in the ambient class, and therefore establishes a lower bound for a model class with \(d_E=\Theta(d)\).

\vspace*{-2pt}
\begin{theorem}[Lower Bound]\label{thm:lower}
For any $d, T \in \mathbb{N}$ with $T\! \geq\! d\! \geq\! 2$ and $dT\! \geq\! C_{\rm lb}$
(an absolute constant determined by $c_0$;\! $C_{\rm lb}\!=\!\lfloor1/(4c_0^2)\rfloor\!+\!1$
suffices), the ambient model class of the construction above, embedded as a POMG, satisfies Assumptions~\ref{ass:pl}, \ref{ass:reveal}, and~\ref{ass:reg} and has $d_E = \Theta(d)$. For any algorithm, there exists an instance in this class such that $\mathbb{E}[PR(T)] \!\geq\! c \sqrt{d T}$ for a universal constant $c > 0$.
\end{theorem}

\vspace*{-2pt}
The construction uses a constant adversary response, which is memoryless and
trivially satisfies the posterior-Lipschitz condition. Thus the hardness comes not
from adversarial memory or partial observability, but from the intrinsic statistical
difficulty of identifying the best policy. 
The detailed proof is in Appendix~\ref{app:lower}.

Since the ambient class has \(d_E=\Theta(d)\), Theorem~\ref{thm:lower}
establishes the claimed \(\Omega(\sqrt{d_E T})\) dependence. We do not claim a
complete minimax characterization of all horizon and memory factors. The outer
factor \(H\) in the upper bound comes from the simulation lemma, and the
\(mH\log T\) term is the warm-up cost needed to let an \(m\)-memory adversary
stabilize at the beginning of each epoch. Refined hard instances can force horizon
and stabilization costs. % , but the main lower bound is intended to establish the
% unavoidable \(\sqrt{T}\) and \(d_E\) dependence. 
Closing the remaining constants and
problem-dependent factors in \(H\) and \(m\) is an interesting direction for future
work.

% % ============================================================================

\section{Extensions}
\label{sec:extensions}
% ============================================================================
\vspace*{-3pt}

Next, we study several natural extensions of our framework. % that broaden its applicability.

%============================================================================

\vspace*{-7pt}
\subsection{Adaptive Horizon Property}
\vspace*{-2pt}

Algorithm~\ref{alg:omle} naturally handles unknown time horizons $T$ through its epoch-based structure. Since epochs grow geometrically and each epoch builds its confidence set from all historical data (not just the current epoch), the algorithm requires no prior knowledge of when it will terminate.

\begin{theorem}[Adaptivity]\label{thm:adaptive}
Without knowledge of $T$, and for all horizons $T$ simultaneously, Algorithm~\ref{alg:omle} achieves $PR(T) = \tilde{O}\bigl(H\sqrt{\bar\beta_T d_E T} + mH\log T\bigr)$ with probability at least $1-2\delta$, where $\bar\beta_T$ and $d_E$ are as in Theorem~\ref{thm:upper}.
\end{theorem}

\vspace*{-4pt}
The result follows from Theorem~\ref{thm:upper}. Each epoch uses only information about its own length to determine how many episodes to run, requiring no knowledge of the total horizon $T$. The regret bound holds for all $T$ by a union bound over epochs with confidence parameters $\delta_e = \delta/(2e^2)$, ensuring $\sum_{e=1}^\infty \delta_e \le \delta$; a second, identical budget covers the per-epoch cap and switch-count events, giving the stated $1-2\delta$. 
The key observation is that the confidence and Eluder summability arguments are
prefix-stable. They apply to any completed prefix of epochs and to the
possibly partial final epoch. Since \(\sum_{e}n_e\le T\) for every horizon
and the deployed-policy count is bounded by Lemma~\ref{lem:switch-count}, the
same Cauchy--Schwarz step as in Theorem~\ref{thm:upper} gives the desired
bound for every final horizon \(T\). This natural adaptivity arises from the epoch-based structure combined with cumulative confidence sets, requiring no modification to the base algorithm. The detailed proof is in Appendix~\ref{app:adaptive}.

% ============================================================================
% ============================================================================

\vspace*{-6pt}
\subsection{Fading Memory Adversaries}
\vspace*{-2pt}

Our base model assumes the adversary's response depends on the learner's most recent $m$ policies through a hard cutoff window. A natural generalization considers adversaries with exponentially decaying memory. Formally, a $\gamma$-fading memory adversary maintains an internal state using the \emph{normalized} recursion $z_t = \gamma z_{t-1} + (1-\gamma)\phi(\pi^t),$ where $\phi: \Pi \to \mathbb{R}^{d_z}$ is a feature map with $\|\phi(\pi)\|_2 \leq 1$, and its response is $g^t = g(z_t)$ for some $L_g$-Lipschitz function $g$. The normalization $(1-\gamma)$ ensures $\|z_t\|_2 \leq 1$ uniformly and the stationary state under constant policy $\pi$ is exactly $z_\infty = \phi(\pi)$, with $\|z_t - z_\infty\|_2 \leq \gamma^k \cdot 2$ after $k$ steps.

\begin{theorem}[Fading Memory]\label{thm:fading}
Against a $\gamma$-fading memory adversary with $L_g$-Lipschitz response
function, run Algorithm~\ref{alg:omle} with memory parameter
$m := m_{\text{warm}} := \lceil \log(2T)/(1-\gamma) \rceil$. Then, with
probability at least $1-2\delta$,
\[
PR(T)
= \tilde{O}\!\left(H\sqrt{\bar\beta^\gamma_T\, d_E T}
\;+\; \frac{H S_T \log T}{1-\gamma}\right),
\]
where $\bar\beta^\gamma_T := H\beta^\gamma_T$ with
$\beta^\gamma_T := O\bigl(\log\mathcal N(\epsilon_T;\Xi)+\log T+\log(1/\delta)
+ B_0\log(2T)/(1-\gamma) + B_0 H L_g\bigr)$, and $S_T$ is the
deployed-policy count of Theorem~\ref{thm:upper}.
\end{theorem}

%\begin{theorem}[Fading Memory]\label{thm:fading}
%Against a $\gamma$-fading memory adversary with $L_g$-Lipschitz response
%function, run Algorithm~\ref{alg:omle} with memory parameter
%$m := m_{\text{warm}} = \lceil \log(2T) / (1-\gamma) \rceil$, entering the
%$mB_0$ floor of $\beta_e$ and making the warm-up step run
%$m_{\text{warm}}-1$ episodes, which suffices because every data-collection
%episode is then at least the $m_{\text{warm}}$-th consecutive episode of the
%deployed policy. Then, with probability at
%least $1\!-\!2\delta$,
%$\textstyle PR(T) = \tilde{O}(H\textstyle\sqrt{\bar\beta^\gamma_T\, d_E T} + H S_T\log T/(1-\gamma))$,
%where $\bar\beta^\gamma_T\! :=\! H\beta^\gamma_T$ with
%$\beta^\gamma_T \!:=\! O(\log\mathcal N(\epsilon_T;\Xi)\!+\!\log T\!+\!$ $\log(1/\delta)\!+\!$ 
%$B_0\log(2T)/(1-\gamma)\!+\!B_0HL_g)$, and $S_T$ is the deployed-policy count
%of Theorem~\ref{thm:upper}, whose warm-up cost is unconditionally dominated by
%the leading term plus $O(H\log^2T/(1-\gamma))$ (Appendix~\ref{app:fading}); in
%particular the second term is $\tilde O(H\log^2T/(1-\gamma))$ whenever no early
%terminations occur.
%\end{theorem}

\vspace*{-3pt}

The algorithmic modification is simple. At the beginning of each epoch $e$, the algorithm
selects an optimistic policy $\pi_e$ and plays it for $m_{\mathrm{warm}}-1$
warm-up episodes, whose data are discarded, followed by at most $T_e$
data-collection episodes. The analysis shows that after $m_{\text{warm}}$ episodes of playing a fixed policy $\pi$, the adversary's state is within $O(\gamma^{m_{\text{warm}}}) \leq O(1/T)$ of the stationary state $z_\infty = \phi(\pi)$. 
%By the first data-collection episode, the adversary
%has observed $m_{\mathrm{warm}}$ consecutive plays of $\pi_e$, so its internal
%state is within
%$O(\gamma^{m_{\mathrm{warm}}}) \leq O(1/T)$
%of the stationary state $z_\infty=\phi(\pi_e)$.

%The algorithmic modification is straightforward. At the beginning of each epoch $e$, after selecting policy $\pi_e$ optimistically, execute $\pi_e$ for $m_{\text{warm}}-1$ warm-up episodes (discarding data) before beginning the at most $T_e$ data-collection episodes. This extended warm-up allows the adversary's internal state to stabilize. The analysis shows that after $m_{\text{warm}}$ episodes of playing a fixed policy $\pi$, the adversary's state is within $O(\gamma^{m_{\text{warm}}}) \leq O(1/T)$ of the stationary state $z_\infty = \phi(\pi)$. 
% Specifically, under the normalized recursion with $\|\phi(\pi)\|_2 \leq 1$, starting from any $z_0$ with $\|z_0\|_2 \leq 1$, the distance to stationarity decays as $\|z_k - z_\infty\|_2 \leq 2\gamma^k$. Setting $k = m_{\text{warm}} = \lceil\log(2T)/(1-\gamma)\rceil$ ensures $\gamma^k \leq 1/(2T)$, making the approximation error negligible.

% \vspace*{-1pt}
Additionally, when computing posterior-predictive quantities for the Lipschitz condition, we truncate the memory to the most recent $O(\log T / (1-\gamma))$ policies, incurring approximation error $O(\gamma^{m_{\text{warm}}}) = O(1/T)$ that does not affect the asymptotic regret rate. As stated, the fading variant fixes $T$ through $m_{\text{warm}}$; an epoch-indexed warm-up restores the horizon-free property at the same rate (Appendix~\ref{app:fading}).

%\vspace*{-1pt}
This result shows that fading memory adversaries admit efficient learning with regret depending on the effective memory $1/(1-\gamma)$ rather than requiring a hard cutoff. The extreme cases are instructive. Setting $\gamma = 0$ recovers a memoryless adversary, for which the warm-up may be set to zero, while $\gamma \to 1$ approaches a long-memory regime in which the effective memory \(1/(1-\gamma)\) diverges. 

\vspace*{-1pt}
The epoch-based structure with one policy per epoch handles fading memory naturally. Within each epoch, we use the extended warm-up phase once, then execute the policy for the epoch's data-collection episodes. Since the adversary sees the same policy repeatedly within the epoch, its state remains stable after warm-up. With $S_T$ epochs started, the warm-up cost is $O(m_{\text{warm}} \cdot S_T \cdot H) = O(\frac{H\,S_T\log T}{1-\gamma})$, which is $O(H\log^2T/(1-\gamma))$ when only scheduled epoch endings occur. The complete proof, including the stability analysis of the adversary's internal state and the truncation argument, is provided in Appendix~\ref{app:fading}.

% ============================================================================
% ============================================================================

\vspace*{-3pt}
\subsection{Tabular Instantiation}
\vspace*{-2pt}

For finite POMGs, we can derive explicit bounds in terms of the state space size $|S|$, action space sizes $|A|$ and $|B|$, and observation space sizes $|O_A|$ and $|O_B|$.

\begin{theorem}[Tabular Bounds]\label{thm:tabular}
Consider a tabular POMG as above, with weak-revealing parameter $\kappa$,
memory $m$, and a linear adversary response of dimension $d_{\rm adv}$. Under
the assumptions of Theorem~\ref{thm:upper}, Algorithm~\ref{alg:omle} achieves,
with probability at least $1-2\delta$,
\[
\textstyle PR(T) \;=\; \tilde{O}\big(Hm\log T \;+\;
H\sqrt{\bar\beta_T^{\rm tab}\cdot d_E^{\rm tab}\cdot T}\big),
\]
where $d_E^{\rm tab}$ %=\tilde{O}\bigl(H(|S|^2|A||B|\,|O_A|^\kappa|O_B|^\kappa+ d_{\rm adv}|B|)\bigr)$ 
is the tabular Eluder bound of
Proposition~\ref{prop:eluder} and
$\bar\beta_T^{\rm tab}=\tilde O\bigl(H^2(|S|^2|A||B|
+ |S|(|O_A|{+}|O_B|) + d_{\rm adv}|B|)+HmB_0\bigr)$.
\end{theorem}

%\begin{theorem}[Tabular Bounds]\label{thm:tabular}
%Under the assumptions of Theorem~\ref{thm:upper}, with weak-revealing parameter
%$\kappa$, for a tabular POMG with $|S|$ states, $|A|$ learner actions, $|B|$ adversary actions,
%$|O_A|$ learner observations, $|O_B|$ adversary observations, horizon $H$, memory $m$,
%and linear adversary response with dimension $d_{\rm adv}$, define $d_E^{\rm tab} \!=\! \tilde{O}(H(|S|^2 |A||B| |O_A|^\kappa |O_B|^\kappa + d_{\rm adv} |B|))$ (Proposition~\ref{prop:eluder}). Then Algorithm~\ref{alg:omle} achieves:
%\[
%\textstyle PR(T) \;=\; \tilde{O} \big(Hm\log T \;+\; H\sqrt{\bar\beta_T^{\rm tab}\cdot d_E^{\rm tab}\cdot T}\big)
%\]
%w.p. at least $1-2\delta$,
%where $\bar\beta_T^{\rm tab} = \tilde O(H^2\bigl(|S|^2|A||B| +$ $|S|(|O_A|{+}|O_B|) +d_{\rm adv}|B|\bigr){+}HmB_0),$
%%\[
%%\begin{aligned}
%%\bar\beta_T^{\rm tab}
%%= \tilde O\!\bigl(
%%H^2\bigl(|S|^2|A||B|&{+}|S|(|O_A|{+}|O_B|) \\[-2pt]
%%&
%%{+}d_{\rm adv}|B|\bigr){+}HmB_0
%%\bigr),$
%%\end{aligned}
%%\]
%the middle term accounting for the unknown emission kernels and the $\tilde O$
%absorbing an additive $H\log(1/\delta)$. Under the standard
%tabular scaling $|O_A|+|O_B| = O(|S||A||B|)$, the emission term is absorbed and
%$\bar\beta_T^{\rm tab}
%= \tilde O\!\left(H^2\bigl(|S|^2|A||B|+d_{\rm adv}|B|\bigr)+HmB_0\right)$.
%\end{theorem}

In $\bar\beta_T^{\rm tab},$ the middle term accounts for the unknown emission kernels, 
and the $\tilde O$ absorbs an additive $H\log(1/\delta)$. Under the standard
tabular scaling $|O_A|+|O_B| = O(|S||A||B|)$, the emission term is absorbed leaving 
$\bar\beta_T^{\rm tab} = \tilde O(H^2\bigl(|S|^2|A||B|+d_{\rm adv}|B|\bigr)+HmB_0)$.
The outer factor \(H\) comes from the
simulation lemma (Lemma~\ref{lem:simulation}), while the tabular Eluder dimension
itself scales linearly with \(H\) because the operator class contains separate
time-step-dependent operators for \(h\in[H]\). Substituting these quantities into
Theorem~\ref{thm:upper} yields the bound. The detailed proof is in
Appendix~\ref{app:tabular}.

% ============================================================================

\vspace*{-5pt}
\subsection{Computational Considerations}
\vspace*{-3pt}

The optimistic planning step in Algorithm~\ref{alg:omle} requires solving $\max_{(\pi,\xi) \in \Pi \times \mathcal{C}_{e-1}} V^\pi(\xi)$, which is computationally challenging in general. POMDP planning is PSPACE-complete even with known parameters, and optimization over the confidence set adds another layer of difficulty.

The statistical guarantees above should therefore be interpreted independently of
computational tractability in the most general model class. For small tabular
instances or structured subclasses, the planning and confidence-set optimization
problems may be approachable by dynamic programming, discretization, or problem-specific
optimization methods. The epoch structure is still useful computationally because
the optimistic planning problem is solved only once per epoch, that is,
\(S_T\) times in total (Lemma~\ref{lem:switch-count}), rather than after every
episode, although the termination test does require a per-episode MLE
evaluation.

A possible practical alternative is to replace optimistic planning with posterior
sampling. At the start of each epoch \(e\), sample parameters \(\tilde{\xi}\) from a
posterior distribution given the cumulative data \(\mathcal{D}_{e-1}\), and then
compute the optimal policy for \(\tilde{\xi}\). This may be easier to implement in
some model classes, although proving an analogous regret guarantee for posterior
sampling would require a separate analysis. Posterior sampling has been analyzed
for POMGs in the self-play and oblivious-adversary regimes
\citep{qiu2023posterior}; extending such guarantees to policy regret against
adaptive adversaries is open.

For additional structured models, such as linear-Gaussian or control-oriented
special cases, parts of estimation or planning may admit more efficient
implementations. Extending computational tractability guarantees to broad POMG
classes remains an important direction for future work.

\vspace*{-5pt}
% ============================================================================
\section{Conclusion}
\label{sec:conclusion}
\vspace*{-3pt}
% ============================================================================

We have developed a comprehensive theoretical framework for policy regret minimization in partially observable Markov games against adaptive adversaries. Our results show that \(\sqrt{d_E T}\) policy regret is both achievable and unavoidable, up to problem-dependent and logarithmic factors; along the way, we showed that data-dependent policy switching is necessary for optimistic algorithms that update their policies only rarely. Several directions remain open, including efficient algorithms beyond tabular and linear POMGs, extensions to multi-player games, non-stationary adversaries, and continuous state-action spaces, and sharper domain-specific models of how partial observability interacts with strategic behavior. More broadly, as AI systems increasingly operate in environments with multiple adaptive agents, these foundations are essential for designing learning algorithms with reliable performance guarantees.

\section*{Acknowledgements}
 This work was supported in part by  NSF CAREER award IIS-1943251. 

\section*{Impact Statement}

This paper presents work whose goal is to advance the foundations of
machine learning in partially observable multi-agent environments. Potential
societal impacts are those broadly associated with reinforcement learning and
multi-agent decision-making, including applications in autonomous systems,
markets, and security. Our results are primarily mathematical and do not introduce
new deployed systems or datasets; we do not identify specific ethical concerns
beyond those generally associated with this area.

\balance
\bibliography{refs_arxiv}
\bibliographystyle{icml2026}

%%%%%%%%%%%%%%%%%%%%%%%%%%%%%%%%%%%%%%%%%%%%%%%%%%%%%%%%%%%%%%%%%%%%%%%%%%%%%%%
%%%%%%%%%%%%%%%%%%%%%%%%%%%%%%%%%%%%%%%%%%%%%%%%%%%%%%%%%%%%%%%%%%%%%%%%%%%%%%%
% APPENDIX
%%%%%%%%%%%%%%%%%%%%%%%%%%%%%%%%%%%%%%%%%%%%%%%%%%%%%%%%%%%%%%%%%%%%%%%%%%%%%%%
%%%%%%%%%%%%%%%%%%%%%%%%%%%%%%%%%%%%%%%%%%%%%%%%%%%%%%%%%%%%%%%%%%%%%%%%%%%%%%%
\newpage
\appendix
\onecolumn

\noindent\textbf{Organization of the appendices.}
Appendix~\ref{sec:related} surveys related work and situates our guarantees
(Tables~\ref{tab:compare} and~\ref{tab:results}).
Appendix~\ref{app:decomposition} proves the causal decomposition
(Theorem~\ref{thm:decomp}). Appendix~\ref{app:eluder-lemma} % develops the
% sign-witness lemma and from it 
proves the Eluder dimension
bounds of Proposition~\ref{prop:eluder}. Appendix~\ref{app:upper} proves
Theorem~\ref{thm:upper}, opening with an extended proof sketch
(Appendix~\ref{app:sketch}). % and including the necessity of data-dependent termination (Appendix~\ref{app:necessity}). 
Appendices~\ref{app:lower}, \ref{app:adaptive}, \ref{app:fading}, and~\ref{app:tabular} prove the lower
bound, the adaptivity property, the fading-memory extension, and the tabular
instantiation. Appendix~\ref{app:experiments} reports the empirical
validation, and Appendix~\ref{app:prior_work} gives a detailed technical
comparison with concurrent work.

\section{Related Work}
\label{sec:related}

The concept of policy regret was introduced by \citet{arora2012bandit} in the context of multi-armed bandits with adaptive adversaries. They showed that while no algorithm can achieve sublinear regret against adversaries with unbounded memory, bounded-memory adversaries admit efficient learning through mini-batching. Competing against reactive opponents with bounded memory has earlier roots in the information-theoretic literature on sequential prediction with loss functions that have memory \citep{merhav2002sequential}. \citet{cesabianchi2013online} sharpened the minimax picture for switching costs and other adaptive adversaries, and \citet{dekel2014bandits} established the matching $\widetilde\Theta(T^{2/3})$ characterization for bandits with switching costs, while
\citet{arora2019policy} extended the switching-cost analysis to bandits with feedback graphs, quantifying how side observations change the attainable rates. 
  \citet{rakhlin2013online} studied online learning with predictable sequences, showing how regret improves when the opponent's moves are partially predictable. We carry this program to partially observable sequential decision-making with structured dynamics.

\citet{arora2018policy} studied the game-theoretic implications of policy regret
and introduced the notion of policy equilibrium. They showed that when both players
use no-policy-regret algorithms, the empirical distribution of play converges to a
policy equilibrium. This result complements the earlier policy-regret literature by
showing that no-policy-regret learning has an associated equilibrium interpretation
in repeated games. Our work is different in focus. We study policy regret for a
single learner in a partially observable Markov game against an adaptive adversary,
rather than convergence of two no-policy-regret players to an equilibrium.
A complementary line of work studies the reverse question of how an opponent can
strategize against a no-regret learner \citep{deng2019strategizing}; policy
regret is the benchmark a learner adopts to be robust against precisely such
strategic responses.

More recently, \citet{nguyentang2024learning} initiated the study of policy regret
in Markov games with adaptive adversaries, identifying fundamental barriers under
unbounded memory or non-stationarity and giving positive results under structural
restrictions on the adversary. \citet{nguyentang2025policy} studied policy regret in Markov games with function approximation, obtaining \(\sqrt{T}\) policy regret under Eluder-type complexity conditions. Our work brings these ideas to the more
complex setting of partially observable Markov games. The partial-observation
setting introduces new difficulties as the learner must reason through observable
operator representations, and the adversary response is entangled with hidden world
dynamics. Our causal decomposition separates hidden-world estimation from adversary-response
estimation. The epoch structure, with geometric caps and statistical termination, then
limits scheduled policy changes to \(O(\log T)\) and prices each data-driven
termination in statistical evidence, so that comparing adversary responses
across deployed policies imposes no cost on the final confidence radius. 

\paragraph{Learning in POMDPs}
Sample-efficient learning in POMDPs has been studied extensively. Early work established computational hardness results \citep{papadimitriou1987complexity} and developed approximate planning methods \citep{cassandra1994acting, kaelbling1998planning}. More recent work has focused on sample complexity under structural assumptions. \citet{krishnamurthy2016pac} established tractability for POMDPs with rich observations, while \citet{azizzadenesheli2016reinforcement} and \citet{guo2016pac} developed spectral methods for parameter estimation. The breakthrough work of \citet{jin2020undercomplete} introduced sample-efficient algorithms for undercomplete POMDPs where the number of observations exceeds the number of latent states. 

\citet{liu2022partially} significantly generalized these results by identifying the class of weakly revealing POMDPs, showing that observability over multi-step windows suffices for efficient learning. \citet{liu2023optimistic} developed an optimistic MLE framework that underlies our approach, establishing \(\tilde O(\sqrt T)\) regret bounds under weak revealing conditions. Recent work by \citet{golowich2023planning} addresses computational aspects of planning and learning in observable POMDPs. Beyond POMDPs, predictive state representations admit sample-efficient learning \citep{zhan2023pac, uehara2022provably}, and unified complexity frameworks such as the generalized Eluder coefficient \citep{zhong2022gec} and the decision-estimation coefficient \citep{foster2021statistical} cover broad classes of partially observable decision problems. These frameworks, however, treat the environment as fixed. The data-generating process does not change in response to the learner's strategy, and that response is precisely the coupling that policy regret against an adaptive adversary must address.

\paragraph{Learning in Markov Games}
Multi-agent reinforcement learning in Markov games has been studied from both computational and statistical perspectives \citep{littman1994markov, hu2003nash, greenwald2003correlated}. For fully observable games, significant progress has been made on finding Nash equilibria and learning against adaptive opponents. \citet{bai2020provable} and \citet{xie2020learning} developed sample-efficient algorithms for finding Nash equilibria in two-player zero-sum games. \citet{jin2022power} introduced the exploiter framework for provably learning non-exploitable policies via self-play, while \citet{tian2022online} studied online learning in unknown Markov games without observing the opponent's actions, achieving sublinear regret against the minimax value. \citet{liu2022learning} studied Markov games against adversarial opponents, delineating which feedback models permit sublinear external regret, and \citet{foster2023complexity} characterized the statistical complexity of multi-agent decision making. These works measure external regret against oblivious or worst-case opponents, rather than policy regret against adaptive ones.

The partially observable case is considerably more challenging. \citet{nayyar2013decentralized} developed the common-information approach to decentralized stochastic control with partial history sharing, a structural foundation for games with asymmetric information. \citet{liu2022sample} studied POMGs under weak revealing assumptions, showing polynomial sample complexity for self-play and, against adversarial opponents, for learning the maximin policy. \citet{qiu2023posterior} obtained sublinear-regret guarantees for POMGs with general function approximation via posterior sampling, in self-play and adversarial regimes with regret measured against the game value, and \citet{liu2023partially} studied partially observable multi-agent RL under information-sharing structures that restore (quasi-)efficiency. However, the adaptive adversary setting we consider here requires fundamentally different techniques, as the adversary's response depends on the learner's policy in ways that create complex feedback loops.

\paragraph{Eluder Dimension and Function Approximation}
The Eluder dimension, introduced by \citet{russo2013eluder}, provides a complexity measure for function classes that governs the sample complexity of exploration in bandits and MDPs. \citet{osband2014model} applied Eluder-based analysis to model-based reinforcement learning, while \citet{wang2020reinforcement} extended these ideas to general value function approximation. \citet{ayoub2020model} developed model-based algorithms using Eluder dimension for structured MDPs, and \citet{zanette2020learning} studied exploration under low inherent Bellman error, a structural condition complementary to Eluder-type complexity measures. Our work builds on these ideas, using Eluder dimension to characterize the complexity of the joint world-adversary model class in POMGs.

\paragraph{Observable Operators and Spectral Methods}
The observable operator model (OOM) framework was introduced by \citet{jaeger2000observable} as an alternative to hidden Markov models, representing dynamics through matrix operations on a predictive state space. \citet{boots2011online} developed an online spectral algorithm for learning predictive-state representations of partially observable nonlinear dynamical systems. \citet{hsu2012spectral} gave a spectral, method-of-moments algorithm for learning hidden Markov models with finite-sample guarantees. \citet{song2010hilbert} extended these ideas to continuous state spaces using kernel methods. The OOM framework has proven particularly valuable for POMDPs, enabling likelihood computation and sample-efficient learning analyses under suitable observability conditions.

\paragraph{Low-Switching Reinforcement Learning}
Our algorithm deploys a single policy per epoch, ending epochs at geometric
caps or by a statistical termination test. Deploying few policies is also the
goal of low-switching-cost RL \citep{bai2019provably, qiao2022sample}, where
switches are limited for operational reasons such as deployment or
recomputation cost, and data-dependent triggers (count or determinant doubling)
achieve logarithmic switching in single-agent settings. Two aspects
differentiate our setting: against an adaptive adversary every policy switch
additionally costs warm-up episodes for the opponent's response to
re-stabilize, and, as we show in Proposition~\ref{prop:fixed-fails},
data-dependent termination is not merely convenient but \emph{necessary}. With
a fixed switching schedule, optimistic planning combined with infrequent
updates provably incurs linear policy regret (Appendix~\ref{app:necessity}).

\paragraph{Concurrent and Related Work on POMGs}
The most closely related work is a concurrent manuscript by \citet{sang2025pomg}, which also studies policy regret minimization in POMGs via optimistic MLE under a posterior-Lipschitz response condition. Our paper differs in several technical respects: (i) our decoupled formulation of the posterior-Lipschitz condition, which avoids circular dependence of the smoothness condition on the adversary being constrained; (ii) our causal decomposition tracks both players' private observation-action histories, so it applies to adversaries that act on private observations, whereas their factorization models the opponent's response as a function of the learner-side history; (iii) our one-policy-per-epoch structure with cumulative confidence sets and statistical termination, under which no posterior-Lipschitz transport term enters the final confidence radius (Remark~\ref{rem:transport-role}), whereas a flat batching schedule lets the transport term grow polynomially (Appendix~\ref{app:prior_work}); (iv) our matching lower bound, proving the optimality of the \(\tilde O(\sqrt T)\) dependence on \(T\) and the aggregate Eluder dimension; and (v) our extensions to fading memory adversaries and horizon-adaptive guarantees. In the other direction, \citet{sang2025pomg} give a counterexample showing that without a posterior-Lipschitz-type condition, sublinear policy regret is impossible even under weakly revealing observations and a memory-bounded, stationary opponent, a useful complement establishing that smoothness conditions of this kind are necessary rather than an artifact of the analysis.

Table~\ref{tab:compare} situates our guarantees among the most closely related
settings. Table~\ref{tab:results} summarizes the bounds proved in this
paper.

\begin{table*}[h]
\caption{Comparison with the most closely related work. ``$m$-memory
adaptive'' means the opponent reacts to the learner's recent policies, and
policy regret charges the learner for that reaction; PAC and regret guarantees
measured against the maximin or Nash value treat the opponent as worst-case
within episodes instead. Rates suppress horizon and problem-dependent factors;
$d_E$ is the channel-sum aggregate Eluder dimension of
Section~\ref{sec:formulation}.}
\label{tab:compare}
\vskip 0.1in
\centering
\footnotesize
\setlength{\tabcolsep}{3pt}
\begin{tabular}{lllll}
\toprule
Reference & Observability & Opponent & Benchmark & Rate \\
\midrule
\citet{arora2012bandit} & bandit feedback & $m$-memory adaptive & policy regret & $\tilde O(T^{2/3})$ \\
\citet{dekel2014bandits} & bandit feedback & switching cost & policy regret & $\widetilde\Theta(T^{2/3})$ \\
\citet{liu2022sample} & POMG, weak revealing & self-play; adversarial & PAC vs.\ maximin & poly$(1/\epsilon)$ \\
\citet{qiu2023posterior} & POMG, general classes & self-play; adversarial & regret vs.\ Nash & $\tilde O(\sqrt T)$ \\
This paper & POMG, weak revealing & $m$-memory adaptive & policy regret & $\tilde O(\sqrt{d_E T})$, matching lower bound \\
\bottomrule
\end{tabular}

\end{table*}

\begin{table*}[h]
\caption{Summary of the guarantees proved in this paper. Upper bounds hold
with probability at least $1-2\delta$ and lower bounds are on expected policy
regret; $\bar\beta_T = H\beta_T$ is the effective confidence radius and $S_T$
the deployed-policy count.}
\label{tab:results}
\vskip 0.1in
\centering
\small
\begin{tabular}{lll}
\toprule
Result & Setting & Bound \\
\midrule
Theorem~\ref{thm:upper} & $m$-memory adversary & $PR(T) = \tilde O\bigl(H\sqrt{\bar\beta_T\, d_E T} + mH\log T\bigr)$, unconditionally \\
Theorem~\ref{thm:lower} & fixed memoryless adversary & $\E[PR(T)] = \Omega\bigl(\sqrt{d_E T}\bigr)$ \\
Proposition~\ref{prop:fixed-fails} & fixed epoch schedule & $\E[PR(T)] \ge T/4$ on an instance satisfying all assumptions \\
Theorem~\ref{thm:adaptive} & unknown horizon & same rate as Theorem~\ref{thm:upper}, for all $T$ simultaneously \\
Theorem~\ref{thm:fading} & $\gamma$-fading memory & $\tilde O\bigl(H\sqrt{\bar\beta^{\gamma}_T d_E T} + H S_T \log T/(1-\gamma)\bigr)$ \\
Theorem~\ref{thm:tabular} & tabular POMG & explicit $\tilde O\bigl(H\sqrt{\bar\beta^{\rm tab}_T d^{\rm tab}_E T}\bigr) + \tilde O(Hm\log T)$ \\
\bottomrule
\end{tabular}
\end{table*}

% ============================================================================

\section{Proof of Theorem~\ref{thm:decomp} (Causal Decomposition)}
\label{app:decomposition}

We construct explicit operators \(W^\theta_h\) and \(G^{\Phi,\pi}_h\) and verify
the factorization \(J^{\xi,\pi}_h=G^{\Phi,\pi}_h\circ W^\theta_h\). Several issues
must be handled simultaneously: (i) the operators must have compatible types, which
requires an intermediate \(B\)-indexed space \(\mathcal V\); (ii) the aggregation
operator \(G^{\Phi,\pi}_h\) must be genuinely independent of \(\theta\); (iii) since the
adversary's private history cannot be inferred from the learner's history in a
general POMG, the augmented state must track both players' histories; (iv) the
construction must also respect the causal timing of the game, where observations
come first, actions are then chosen, and the world transitions afterward; and (v) 
the OOM operator represents the controlled conditional process of observations
given learner actions, rather than the joint probability of learner actions and
observations.

\begin{proof}[Proof of Theorem~\ref{thm:decomp}]
The proof proceeds in stages, constructing the two operators explicitly and
then verifying the factorization and its structural properties
(Figure~\ref{fig:decomp}).

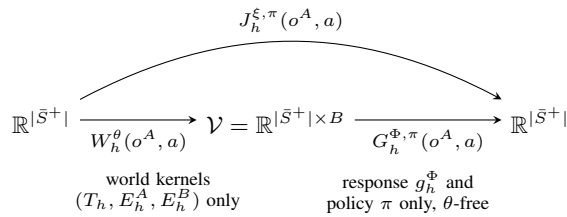
\begin{figure}[h]
\centering
\begin{tikzpicture}[>=stealth, every node/.style={font=\small}]
\node (A) at (0,0) {$\R^{|\bar S^+|}$};
\node (V) at (3.1,0) {$\mathcal V=\R^{|\bar S^+|\times B}$};
\node (B) at (6.6,0) {$\R^{|\bar S^+|}$};
\draw[->] (A) -- node[below, font=\scriptsize]{$W^{\theta}_h(o^A,a)$} (V);
\draw[->] (V) -- node[below, font=\scriptsize]{$G^{\Phi,\pi}_h(o^A,a)$} (B);
\draw[->] (A) to[bend left=28] node[above, font=\scriptsize]{$J^{\xi,\pi}_h(o^A,a)$} (B);
\node[font=\scriptsize, text width=2.6cm, align=center] at (1.55,-0.95)
  {world kernels\\ $(T_h,E^A_h,E^B_h)$ only};
\node[font=\scriptsize, text width=2.6cm, align=center] at (4.85,-0.95)
  {response $g^\Phi_h$ and\\ policy $\pi$ only, $\theta$-free};
\end{tikzpicture}
\caption{The causal decomposition of Theorem~\ref{thm:decomp}. The observable
operator factors through the $B$-indexed intermediate space; the world channel
carries all dependence on $\theta$, and the adversary aggregation carries all
dependence on the response model and the deployed policy.}
\label{fig:decomp}
\end{figure}

\paragraph{Timing convention and histories.}
We begin by fixing the timing convention. At step \(h\), both players observe
\((o^A_h,o^B_h)\), then choose actions \((a_h,b_h)\), and then the latent state
transitions to \(s_{h+1}\). This order matters because the adversary's response at
step \(h\) is allowed to depend on its current observation \(o^B_h\), but not on
the action \(b_h\) before it is chosen.

It is useful to distinguish the adversary history available before the action from
the history stored after the action. The post-action adversary history after step
\(h\) is
\[
\tau^B_h
=
(o^B_1,b_1,\ldots,o^B_{h-1},b_{h-1},o^B_h,b_h),
\]
which includes the chosen action $b_h$. 
This is the object carried into the next augmented state. The adversary's decision
history at step \(h\) is
\[
\eta^B_h
=
(\tau^B_{h-1},o^B_h)
=
(o^B_1,b_1,\ldots,o^B_{h-1},b_{h-1},o^B_h),
\]
which includes the observation-action history through step $h-1$ plus the new observation $o^B_h$. This is the information available after observing \(o^B_h\) and before choosing
\(b_h\), so \(g^\Phi_h(\cdot\mid \eta^B_h,\pi)\) is well-defined. The two histories 
are related by \(\tau^B_h=(\eta^B_h,b_h)\), i.e., \(\tau^B_h\) is \(\eta^B_h\)
extended by \(b_h\).

The learner histories are handled analogously. At the beginning of step \(h\), the
learner carries \(\tau^A_{h-1}\), and after observing \(o^A_h\) and taking action
\(a_h\), the history becomes
\[
\tau^A_h=(\tau^A_{h-1},o^A_h,a_h).
\]

\paragraph{The augmented state space.}
We now define the augmented state at the beginning of step \(h\) as
\[
\bar s_h=(s_h,\tau^A_{h-1},\tau^B_{h-1}),
\]
where both histories are
post-action observation-action histories (consistent with $\cT^A$ and $\cT^B$). The augmented state space is
\[
\bar S^+
=
S\times \mathcal T^A_{\le H}\times \mathcal T^B_{\le H}.
\]
This full-history space is finite for every fixed finite-horizon problem
instance, since all component spaces are finite (Section~\ref{sec:formulation}).
It is intentionally larger than the reduced predictive representations used later
for complexity bounds. Here the goal is exact factorization, and exact
factorization requires enough information to recover both players' histories.
We adopt the following length convention throughout the construction: all
step-\(h\) operators act as zero on coordinates whose history components are
not length-matched to the step. For \(J^{\xi,\pi}_h\) and \(W^\theta_h\),
inputs carry histories of length \(h-1\) and outputs of length \(h\);
\(G^{\Phi,\pi}_h\) acts on the \(B\)-indexed intermediate coordinates that
\(W^\theta_h\) produces, whose learner-history component already has length
\(h\). With this convention every operator below is totally defined on its
space, and the OOM products only ever traverse length-matched supports.

\paragraph{World channel $W^\theta_h$ mapping into the $B$-indexed intermediate space.} We next construct the world channel. Let
\[
\mathcal V=\mathbb R^{|\bar S^+|\times B}.
\]
An element of \(\mathcal V\) is a family \((u_b)_{b\in B}\), where each
\(u_b\in\mathbb R^{|\bar S^+|}\). The role of this space is to keep the adversary
action \(b\) as an explicit index before it is mixed according to the adversary
response. For each \(b\in B\), define the linear map
\[
W^\theta_{h,b}(o^A,a)\in
\mathcal L(\mathbb R^{|\bar S^+|},\mathbb R^{|\bar S^+|})
\]
by 
\begin{equation}
\label{eqWdef}
\begin{aligned}
\bigl[W^\theta_{h,b}(o^A,a)v\bigr]_{(s',\tau^A_h,\tau^B_h)}
&\quad =
\sum_{\substack{(s,\tau^A,\tau^B)\\(\tau^A,o^A,a)=\tau^A_h}}
\sum_{\substack{o^B\in O_B\\(\tau^B,o^B,b)=\tau^B_h}}
E^A_h(o^A\mid s)\,
E^B_h(o^B\mid s)\,
T_h(s'\mid s,a,b)\,
v_{(s,\tau^A,\tau^B)} .
\end{aligned}
\end{equation}
This formula has a simple interpretation; the product
\(E^A_h(o^A\mid s)\,E^B_h(o^B\mid s)\) uses the conditional independence of the
two observations given \(s\) (Section~\ref{sec:formulation}). Starting from augmented state
\((s,\tau^A,\tau^B)\), the world emits \(o^A\) to the learner and \(o^B\) to the
adversary. The learner action \(a\) is fixed because the controlled OOM operator is
conditioned on learner actions. The adversary action \(b\) is not yet randomized.
Instead, \(W^\theta_{h,b}\) computes the hypothetical successor mass if the
adversary were to choose this particular \(b\). The learner history becomes
\((\tau^A,o^A,a)\), the adversary history becomes \((\tau^B,o^B,b)\), and the
latent state moves to \(s'\) according to \(T_h(\cdot\mid s,a,b)\).

The history-consistency conditions in \eqref{eqWdef} only select legal predecessors. Once
\((s',\tau^A_h,\tau^B_h)\), \(o^A\), \(a\), and \(b\) are fixed, the equations
\[
\tau^A_h=(\tau^A,o^A,a),
\qquad
\tau^B_h=(\tau^B,o^B,b)
\]
determine the predecessor histories and the adversary observation whenever they
exist. The summation therefore reduces to a single genuine sum over the latent
state \(s\in S\) (the histories and the adversary observation are pinned by the
output index), which is exactly the belief-propagation sum of the controlled
forward kernel on the augmented state space.

The full world channel is the \(B\)-indexed collection
\[
W^\theta_h(o^A,a)v
=
\bigl(W^\theta_{h,b}(o^A,a)v\bigr)_{b\in B}
\in \mathcal V .
\]
By construction, \(W^\theta_h\) depends only on the world kernels
\((T_h,E^A_h,E^B_h)\). It does not depend on \(\Phi\) or on the learner policy
\(\pi\). The interpretation is that for each hypothetical adversary action $b$,
$W^\theta_{h,b}$ computes the unnormalized mass flowing to each successor augmented
state $(s',\tau^A_h,\tau^B_h)$ under world dynamics.  Crucially, $\tau^B_h = (\tau^B,o^B,b)$
includes $b$ (the adversary's action at this step), so the next step's augmented
state correctly records the full observation-action history of the adversary.  

All strategic weighting over adversary actions by response probabilities 
is deliberately postponed
to the next operator, $G^{\Phi,\pi}_h$.

\paragraph{Adversary aggregation $G^{\Phi,\pi}_h$ mapping from the intermediate space.} We now define the adversary aggregation operator. Given
\((u_b)_{b\in B}\in\mathcal V\), define $G^{\Phi,\pi}_h(o^A,a):\mathcal{V}\to\R^{|\bar{S}^+|}$ by 

\begin{equation}
\label{eqGdef}
\bigl[
G^{\Phi,\pi}_h(o^A,a)(u_b)_{b\in B}
\bigr]_{(s',\tau^A_h,\tau^B_h)}
=
g^\Phi_h(b\mid \eta^B_h,\pi)\,
[u_b]_{(s',\tau^A_h,\tau^B_h)},
\end{equation}

where the output history is written as
\(\tau^B_h=(\eta^B_h,b)\), with
\(\eta^B_h=(\tau^B_{h-1},o^B_h)\) and \(b\in B\). Coordinates whose history components are not length-matched to step \(h\), or do not parse as a post-action history, are mapped to zero, per the length convention above.

This respects causal timing as the adversary forms $\eta^B_h = (\tau^B_{h-1},o^B_h)$
after observing $o^B_h$ and before choosing $b_h$, then $g^\Phi_h(b\mid\eta^B_h,\pi)$ 
gives the probability of choosing $b$, and then $b$ is appended to form $\tau^B_h$.

The key property is that \(G^{\Phi,\pi}_h\) depends only on \((g^\Phi_h,\pi)\) and
does not involve \(\theta\). The input \((u_b)_{b\in B}\in\mathcal V\) already
contains the hypothetical successor masses computed by the world channel \(W^\theta_h\), one for each
possible adversary action. The aggregation operator then reads the adversary response weight
\(g^\Phi_h(b\mid\eta^B_h,\pi)\), where both \(b\) and \(\eta^B_h\) are determined
by the output index \(\tau^B_h\), and multiplies the corresponding coordinate of
\(u_b\) by this weight. Thus \(G^{\Phi,\pi}_h\) uses the
adversary decision history encoded in the augmented-state index and the response
model \(g^\Phi_h\), but it never refers to the world parameter \(\theta\).

It is worth clarifying how the mixture over adversary actions is represented in
this coordinatewise definition. For a fixed output index
\((s',\tau^A_h,\tau^B_h)\), the post-action history \(\tau^B_h\) already contains
the final adversary action. Writing $\tau^B_h=(\eta^B_h,b),$ where \(\eta^B_h=(\tau^B_{h-1},o^B_h)\) is the decision history and \(b\) is the
action appended to it, only the component \(u_b\) contributes to this particular
output coordinate. This does not mean that the other adversary actions are ignored.
If we fix a decision history \(\eta^B_h\) and sum over all post-action histories
\((\eta^B_h,b)\) with \(b\in B\), the aggregation recovers the full mixture
\[
\sum_{b\in B}
g^\Phi_h(b\mid\eta^B_h,\pi)\,
[u_b]_{(\cdot,\cdot,(\eta^B_h,b))}.
\]
Thus \(G^{\Phi,\pi}_h\) represents the adversary's randomized action by assigning
each possible \(b\) to the successor coordinates whose adversary history ends with
that \(b\), and weighting those coordinates by
\(g^\Phi_h(b\mid\eta^B_h,\pi)\).

\paragraph{Verifying the factorization $J^{\xi,\pi}_h = G^{\Phi,\pi}_h\circ W^\theta_h$.}
The canonical controlled forward operator of the protocol is, coordinate-wise,
\begin{align*}
&\bigl[J^{\xi,\pi}_h(o^A,a)\bigr]_{(s',\hat\tau^A,\hat\tau^B),(s,\tau^A,\tau^B)}
\\
&\quad=
\mathbf 1[\hat\tau^A=(\tau^A,o^A,a)]
\!\!\sum_{\substack{o^B,\,b:\\ \hat\tau^B=(\tau^B,o^B,b)}}\!\!
E^A_h(o^A|s)\,E^B_h(o^B|s)\,
g^\Phi_h\bigl(b\,|\,(\tau^B,o^B),\pi\bigr)\,
T_h(s'|s,a,b),
\end{align*}
on length-matched coordinates and zero otherwise; this is the operator whose
products yield the controlled OOM identity below, and the theorem asserts that
it factorizes. 

We can now verify the factorization. Combining \eqref{eqWdef} and \eqref{eqGdef},
and writing the adversary post-action history as
\[
\tau^B_h=(\eta^B_h,b),
\qquad
\eta^B_h=(\tau^B_{h-1},o^B_h),
\]
we have

\begin{align}
&
\bigl[
(G^{\Phi,\pi}_h\circ W^\theta_h)(o^A,a)v
\bigr]_{(s',\tau^A_h,\tau^B_h)}
\nonumber\\
&\quad =
g^\Phi_h(b\mid\eta^B_h,\pi)
\bigl[
W^\theta_{h,b}(o^A,a)v
\bigr]_{(s',\tau^A_h,\tau^B_h)}
\nonumber\\
&\quad =
g^\Phi_h(b\mid\eta^B_h,\pi)
\sum_{\substack{(s,\tau^A,\tau^B)\\(\tau^A,o^A,a)=\tau^A_h}}
\sum_{\substack{o^B\in O_B\\(\tau^B,o^B,b)=\tau^B_h}}
E^A_h(o^A\mid s)\,
E^B_h(o^B\mid s)\,
T_h(s'\mid s,a,b)\,
v_{(s,\tau^A,\tau^B)} .
\label{eqcomposition}
\end{align}

This is exactly the controlled conditional forward dynamics of the POMG. For each current augmented state \((s,\tau^A,\tau^B)\) with mass
\(v_{(s,\tau^A,\tau^B)}\), the expression in \eqref{eqcomposition} follows exactly
one step of the POMG protocol. The learner observation \(o^A\) and adversary
observation \(o^B\) are emitted from the same latent state \(s\), with probabilities
\(E^A_h(o^A\mid s)\) and \(E^B_h(o^B\mid s)\). The adversary observation \(o^B\)
extends the adversary's pre-step history to the decision history $\eta^B_h=(\tau^B,o^B).$ The adversary then chooses action \(b\) with probability
\(g^\Phi_h(b\mid\eta^B_h,\pi)\), and its post-action history becomes $\tau^B_h=(\tau^B,o^B,b)=(\eta^B_h,b).$ At the same step, the learner takes the fixed controlled action \(a\), so its
post-action history becomes $\tau^A_h=(\tau^A,o^A,a).$ Finally, the world transitions to \(s'\) with probability \(T_h(s'\mid s,a,b)\).
Thus the right-hand side of \eqref{eqcomposition} is precisely the
\((s',\tau^A_h,\tau^B_h)\) coordinate of the controlled conditional forward
operator \(J^{\xi,\pi}_h(o^A,a)v\). Therefore
\[
J^{\xi,\pi}_h(o^A,a)
=
G^{\Phi,\pi}_h(o^A,a)\circ W^\theta_h(o^A,a).
\]

\paragraph{The OOM identity and learner action probabilities.}
It remains to connect the factorization above to the controlled OOM convention used
in the paper. The operator \(J^{\xi,\pi}_h(o^A,a)\) is indexed by the learner
observation \(o^A\) and the learner action \(a\), and it propagates unnormalized
mass over the augmented state space conditional on that controlled observation-action
pair. Therefore, if \(q_0 = \rho_0\otimes\delta_{\emptyset}\otimes\delta_{\emptyset}\)
is the initial augmented-state distribution (mass \(\rho_0(s)\) on
\((s,\emptyset,\emptyset)\); it depends on no model parameter), then
for any fixed learner observation-action sequence
\[
(o^A_1,a_1,\ldots,o^A_h,a_h),
\]
the vector
\[
\tilde q^{\xi}_{h}
=
J^{\xi,\pi}_h(o^A_h,a_h)\cdots
J^{\xi,\pi}_1(o^A_1,a_1)q_0
\]
is the unnormalized augmented-state distribution after realizing that controlled
sequence (its normalization is the predictive state \(q^\xi_h\) of
Section~\ref{sec:formulation}). Its total mass is obtained by summing its
coordinates over \(\bar s\in\bar S^+\), namely
\[
\sum_{\bar s\in\bar S^+}
\left[
J^{\xi,\pi}_h(o^A_h,a_h)\cdots
J^{\xi,\pi}_1(o^A_1,a_1)q_0
\right]_{\bar s}
=
\mathbb P^{\xi, \pi}
\!\left(
o^A_{1:h}\mid a_{1:h}
\right),
\]
which is the controlled conditional OOM identity used in
Section~\ref{sec:oom}. The learner action probabilities are not included inside
\(J^{\xi,\pi}_h(o^A,a)\). They enter only when converting this controlled
conditional law into the probability of a learner trajectory under policy \(\pi\):
\[
P^{\pi}_{\xi}(o^A_{1:h},a_{1:h})
=
\mathbb P^{\xi,\pi}(o^A_{1:h}\mid a_{1:h})
\prod_{t=1}^h
\pi_t(a_t\mid \tau^A_{t-1},o^A_t).
\]
Thus the OOM operators describe the environment's controlled observation process,
while the learner policy contributes the separate action-selection factors in the
likelihood.

\paragraph{Non-negativity and $\ell_1$-nonexpansiveness.} 
The nonnegativity properties follow directly from the construction. Each
\(W^\theta_{h,b}\) has nonnegative entries because its coefficients are products of
transition probabilities, emission probabilities, and history-consistency
indicators. Thus \(v\ge 0\) implies \(u_b = W^\theta_{h,b}v\ge 0\) for every \(b\).
The operator \(G^{\Phi,\pi}_h\) multiplies each successor coordinate whose adversary
post-action history is \(\tau^B_h=(\eta^B_h,b)\) by the nonnegative probability
\(g^\Phi_h(b\mid\eta^B_h,\pi)\). For each fixed decision history \(\eta^B_h\),
these weights sum to one over \(b\in B\). Consequently, the aggregation preserves
nonnegativity and correctly represents the adversary's randomized action. Hence
\(J^{\xi,\pi}_h\) is nonnegative as well.

The \(\ell_1\)-nonexpansiveness asserted in the theorem holds in signed form,
which is how the operator is used in
the channel-triangle inequalities \eqref{eq:channel-triangle} and the transfer arguments of
Appendix~\ref{app:upper}. Fix any \(u\in\mathcal V\) with coordinates
\(u_{\bar s,b}\). By the construction above, the output coordinate of
\(G^{\Phi,\pi}_h u\) at an augmented state \(\bar s\) whose adversary
post-action history ends with the action \(b(\bar s)\) reads the single input
coordinate \(u_{\bar s,\,b(\bar s)}\) and scales it by
\(g^\Phi_h(b(\bar s)\mid\eta^B_h(\bar s),\pi)\in[0,1]\), and distinct output
coordinates read distinct input coordinates. Hence
\[
\bigl\|G^{\Phi,\pi}_h u\bigr\|_1
=\sum_{\bar s}g^\Phi_h\bigl(b(\bar s)\mid\eta^B_h(\bar s),\pi\bigr)\,
\bigl|u_{\bar s,\,b(\bar s)}\bigr|
\le\sum_{\bar s,\,b}\bigl|u_{\bar s,b}\bigr|
=\|u\|_{1,\mathcal V},
\]
so \(G^{\Phi,\pi}_h\) is \(\ell_1\)-nonexpansive on all of \(\mathcal V\), not
only on nonnegative inputs. For nonnegative inputs the weights
\(g^\Phi_h(\cdot\mid\eta^B_h,\pi)\), summing to one over \(b\) at each decision
history, additionally conserve the total mass routed through that history,
which is the probabilistic content of the aggregation.

\paragraph{Dimension and the weak revealing condition.} The construction above uses the full-history augmented space \(\bar S^+\), whose
dimension is exponential in \(H\). This is intentional for the exact factorization.
Under the \(\kappa\)-step weakly revealing condition, one can often replace this
full-history representation by a lower-dimensional predictive-state
representation whose dimension depends on \(\kappa\)-window statistics rather than
the full history length. In the single-agent POMDP setting, such reductions are
standard in OOM analyses of weakly revealing models
\citep{liu2022partially}. In the POMG setting considered here, the analogous
finite-window representation requires window sufficiency of the retained
statistics for both channels; Proposition~\ref{prop:eluder} imposes exactly
this as a hypothesis (policy and response classes depending on histories only
through the \(\kappa\)-window statistics), with Assumption~\ref{ass:reveal}
supplying detectability of the world channel. In the tabular
finite-window case, this gives predictive dimension on the order of
\[
O\!\left(
|S|\cdot |O_A|^{\kappa-1}\cdot |O_B|^{\kappa-1}
\right).
\]

This dimension reduction is separate from the exact causal factorization. Naively
truncating \((\tau^A,\tau^B)\) to the last \(\kappa-1\) observations on each side
does not in general preserve the factorization unless \(\pi\) and \(g^\Phi\) also
depend only on those truncated histories, or unless the \(\kappa\)-step predictive
state is used as a sufficient statistic in place of raw histories. The construction
above is always valid on the full-history space. Dimension reduction via the $\kappa$-window predictive state is an
additional step that relies on Assumption~\ref{ass:reveal} and is invoked in the
Eluder dimension bound (Proposition~\ref{prop:eluder}), not in the factorization
itself.

\paragraph{Conclusion.} We have therefore constructed the desired causal factorization. In particular:
\begin{itemize}
\item The world channel
\[
W^\theta_h(o^A,a):\mathbb R^{|\bar S^+|}\to \mathcal V
\]
depends only on the world kernels \((T_h,E^A_h,E^B_h)\). For each adversary action
\(b\), its \(b\)-indexed component computes the hypothetical unnormalized successor
mass obtained if the adversary were to choose \(b\), and it updates the adversary
history by appending \(b\) to form $\tau^B_h=(\tau^B_{h-1},o^B_h,b).$

\item The adversary aggregation operator
\[
G^{\Phi,\pi}_h(o^A,a):\mathcal V\to \mathbb R^{|\bar S^+|}
\]
depends only on the response model \(g^\Phi_h\) and the learner policy \(\pi\).
For an output history \(\tau^B_h=(\eta^B_h,b)\), it reads the decision history $\eta^B_h=(\tau^B_{h-1},o^B_h)$ from the augmented-state index and applies the response weight $g^\Phi_h(b\mid \eta^B_h,\pi).$

Thus \(G^{\Phi,\pi}_h\) is genuinely independent of the world parameter \(\theta\).

\item The composition satisfies
\[
J^{\xi,\pi}_h(o^A,a)
=
G^{\Phi,\pi}_h(o^A,a)\circ W^\theta_h(o^A,a),
\]
as verified by direct computation. This composition is exactly the controlled
conditional forward dynamics of the POMG, with learner action \(a\) fixed and the
adversary observation and action incorporated according to \(E^B_h\) and
\(g^\Phi_h\).
\end{itemize}

Both operators preserve non-negativity. The intermediate space $\mathcal V=\mathbb R^{|\bar S^+|\times B}$ makes the types compatible by storing one hypothetical successor family for each
adversary action before the adversary response weights are applied. Tracking both
players' full observation-action histories ensures that the adversary decision
history \(\eta^B_h\) is available to \(G^{\Phi,\pi}_h\) without referring to
\(\theta\), and conditioning \(g^\Phi_h\) on the post-observation, pre-action
history \(\eta^B_h\) respects the causal protocol of the POMG.
\end{proof}

\section{Eluder Dimension of $\ell_1$ Operator-Error Classes}
\label{app:eluder-lemma}

The complexity bounds of Proposition~\ref{prop:eluder} rest on the following
general estimate. For classes of $\ell_1$-norms of linear images of a
low-dimensional operator family, the scale-sensitive Eluder dimension is
controlled by the \emph{span dimension} of the family, up to a factor
logarithmic in the scale.

\begin{lemma}[Sign-witness Eluder bound]\label{lem:sign-witness}
\emph{(a) Fixed design.} Let $\mathcal A\subset\R^{D_{\rm out}\times D_{\rm in}}$ be a family of
matrices contained in a linear subspace of dimension $p$, with
$\|A\|_F\le B_A$ for all $A\in\mathcal A$, and let each query $x$ carry a
vector $v_x$ with $\|v_x\|_2\le B_v$. Consider the function class
$\mathcal F := \{x\mapsto\|Av_x\|_1 : A\in\mathcal A\}$. Then for every
$\varepsilon>0$,
\[
\dimE(\mathcal F,\varepsilon)
\;\le\;
12\,p\,\log \Bigl(1+\frac{8\,D_{\rm out}B_v^2B_A^2}{\varepsilon^2}\Bigr).
\]
\emph{(b) Query-dependent design.} Let instead each query $x$ carry a matrix
$U(x)\in\R^{D_{\rm out}\times p}$ with $\|U(x)\|_F\le B_U$, and let
$\mathcal C\subset\R^p$ with $\|c\|_2\le B_c$ for all $c\in\mathcal C$. For
$\mathcal F := \{x\mapsto\|U(x)c\|_1 : c\in\mathcal C\}$ and every
$\varepsilon>0$,
\[
\dimE(\mathcal F,\varepsilon)
\;\le\;
12\,p\,\log\!\Bigl(1+\frac{8\,D_{\rm out}B_U^2B_c^2}{\varepsilon^2}\Bigr).
\]
\end{lemma}

\begin{proof}
By the scale convention of Section~\ref{sec:formulation} it suffices to bound,
for each $\varepsilon'\ge\varepsilon$, the length of any sequence whose
elements are $\varepsilon'$-independent of their predecessors at the exact
scale $\varepsilon'$. The bound we prove is decreasing in the scale, so the
bound at scale $\varepsilon'$ together with $\varepsilon'\ge\varepsilon$ gives
the lemma. To lighten notation we write $\varepsilon$ for $\varepsilon'$.

We first prove part (a). Let $x_1,\ldots,x_N$ be a sequence in which each
$x_i$ is $\varepsilon$-independent of its predecessors, witnessed by
$A_i\in\mathcal A$; that is,
\begin{equation}\label{eq:sw-witness}
\|A_iv_{x_i}\|_1>\varepsilon
\qquad\text{and}\qquad
\sum_{k<i}\|A_iv_{x_k}\|_1^2\le\varepsilon^2 .
\end{equation}
We need to bound $N$. We begin by defining the \emph{sign-witness matrices}
\begin{equation}\label{eq:sw-matrices}
X_k := \sigma_kv_{x_k}^\top,
\qquad
\sigma_k:=\mathrm{sign}(A_kv_{x_k})\in\{\pm1\}^{D_{\rm out}},
\end{equation}
where the sign vector $\sigma_k$ is fixed at step $k$ and never changes
afterwards. These matrices turn the $\ell_1$ quantities of
\eqref{eq:sw-witness} into inner products. At its own query, each witness
aligns with its sign-witness matrix,
\begin{equation}\label{eq:sw-align}
\langle A_i,X_i\rangle=\sigma_i^\top A_iv_{x_i}=\|A_iv_{x_i}\|_1>\varepsilon .
\end{equation}
At earlier queries the inner product is dominated by the corresponding
$\ell_1$ norm, because $|\sigma_k^\top A_iv_{x_k}|\le\|A_iv_{x_k}\|_1$.
Combining this with the second inequality in \eqref{eq:sw-witness} gives
\begin{equation}\label{eq:sw-dom}
\sum_{k<i}\langle A_i,X_k\rangle^2\le\varepsilon^2 .
\end{equation}

The next step is an elliptic-potential argument. Recall that $\mathcal A$
lies in a linear subspace of matrices of dimension $p$. Fix an orthonormal
basis of this subspace. Let $\bar A_i\in\R^p$ be the coordinate vector of
$A_i$, and let $\bar X_k\in\R^p$ be the coordinate vector of the orthogonal
projection of $X_k$ onto the subspace. Since $A_i$ lies in the subspace,
$\langle A_i,X_k\rangle=\langle\bar A_i,\bar X_k\rangle$, with the Frobenius
inner product on the left and the Euclidean one on the right; the quantities
in \eqref{eq:sw-align} and \eqref{eq:sw-dom} are therefore unchanged, and
$\|\bar A_i\|_2=\|A_i\|_F$. Define the potential matrices
\begin{equation}\label{eq:sw-potential-def}
V_i:=\lambda I_p+\sum_{k\le i}\bar X_k\bar X_k^\top,
\qquad
\lambda:=\varepsilon^2/B_A^2 .
\end{equation}
By Cauchy--Schwarz in the $V_{i-1}$-geometry, followed by \eqref{eq:sw-dom}
and $\lambda\|A_i\|_F^2\le\varepsilon^2$,
\[
\varepsilon^2
<\langle A_i,X_i\rangle^2
\le\Bigl(\sum_{k<i}\langle A_i,X_k\rangle^2+\lambda\|A_i\|_F^2\Bigr)
\cdot\|\bar X_i\|^2_{V_{i-1}^{-1}}
\le 2\varepsilon^2\,\|\bar X_i\|^2_{V_{i-1}^{-1}},
\]
where the strict inequality is \eqref{eq:sw-align}. Hence
$\|\bar X_i\|^2_{V_{i-1}^{-1}}>1/2$ for every $i$. The elliptic-potential
lemma \citep[Lemma~19.4]{lattimore2020bandit} bounds the sum of these
quantities,
\begin{equation}\label{eq:sw-elliptic}
\sum_{i\le N}\min\bigl(1,\|\bar X_i\|^2_{V_{i-1}^{-1}}\bigr)
\le 2p\log\Bigl(1+\frac{N\max_k\|X_k\|_F^2}{\lambda p}\Bigr).
\end{equation}
The left-hand side of \eqref{eq:sw-elliptic} is at least $N/2$, and
$\|X_k\|_F^2=\|\sigma_k\|_2^2\,\|v_{x_k}\|_2^2\le D_{\rm out}B_v^2$.
Substituting both, together with the value of $\lambda$, yields
\begin{equation}\label{eq:sw-selfref}
N\;\le\;4p\log\Bigl(1+\frac{N\,D_{\rm out}B_v^2B_A^2}{p\,\varepsilon^2}\Bigr).
\end{equation}

Inequality \eqref{eq:sw-selfref} is self-referential in $N$, and it remains
to solve it. Write $u:=N/p$ and $C_0:=D_{\rm out}B_v^2B_A^2/\varepsilon^2$,
so that \eqref{eq:sw-selfref} reads $u\le 4\log(1+uC_0)$. If any witness
exists then $\varepsilon<\|A_1v_{x_1}\|_1\le\sqrt{D_{\rm out}}B_AB_v$, that
is, $C_0>1$; otherwise $\dimE=0$ and the bound is trivial. The elementary
inequality $1+uC_0\le(1+u)(1+C_0)$ then gives
\begin{equation}\label{eq:sw-solve}
u\;\le\;4\log(1+u)+4\log(1+C_0).
\end{equation}
The function $u\mapsto u-4\log(1+u)$ is increasing for $u\ge 3$, and at
$u=12\log(1+8C_0)\ge 12\log 9$ it exceeds $4\log(1+C_0)$. By
\eqref{eq:sw-solve} we conclude $u\le 12\log(1+8C_0)$, which is the stated
bound.

Part (b) follows by the same argument with witness \emph{vectors} in place of
witness matrices. In place of \eqref{eq:sw-matrices}, define
$w_k := U(x_k)^\top\sigma_k\in\R^p$ with $\sigma_k:=\mathrm{sign}(U(x_k)c_k)$.
The alignment \eqref{eq:sw-align} becomes
$\langle c_i,w_i\rangle=\sigma_i^\top U(x_i)c_i=\|U(x_i)c_i\|_1>\varepsilon$,
and the domination \eqref{eq:sw-dom} holds because
$|\langle c_i,w_k\rangle|=|\sigma_k^\top U(x_k)c_i|\le\|U(x_k)c_i\|_1$. The
potential \eqref{eq:sw-potential-def} runs directly in $\R^p$, with
$V_i:=\lambda I_p+\sum_{k\le i}w_kw_k^\top$ and $\lambda:=\varepsilon^2/B_c^2$,
and no projection is needed. The norm bound becomes
$\|w_k\|_2^2\le\|\sigma_k\|_2^2\,\|U(x_k)\|_{\rm F}^2\le D_{\rm out}B_U^2$,
using $\|U^\top\sigma\|_2\le\|U\|_{\rm op}\|\sigma\|_2\le\|U\|_{\rm F}\|\sigma\|_2$.
With these substitutions, the argument from \eqref{eq:sw-potential-def} to
\eqref{eq:sw-solve} applies verbatim and gives the stated bound.
\end{proof}

We now use the lemma to prove the proposition.

\begin{proof}[Proof of Proposition~\ref{prop:eluder}]
Fix a reference parameter $\xi^*=(\theta^*,\Phi^*)$. By definition
$d_E=d^{W}_{\rm agg}+d^{G}_{\rm agg}$, so it suffices to bound the Eluder
dimension of each channel class at the evaluation scale $\alpha=1/(2T)$. The
proof goes as follows. We first show that every function in the world-channel
class is an expectation of $\ell_1$ norms of linear images of one operator,
and we count the dimension of the subspace in which these operators live. We
then verify that the expectation-form queries provide sign-witness matrices
satisfying the alignment \eqref{eq:sw-align}, the domination
\eqref{eq:sw-dom}, and a Frobenius norm bound; these are the only properties
of the witness matrices that the proof of Lemma~\ref{lem:sign-witness}(a)
uses, so the lemma's conclusion applies and yields the world-channel bound.
The adversary channel is handled in the same way through
Lemma~\ref{lem:sign-witness}(b). Finally we add the two channel bounds and
specialize the count to the three model families. Throughout, the bounds we
obtain do not depend on $\xi^*$, so the suprema in $d^{W}_{\rm agg}$ and
$d^{G}_{\rm agg}$ inherit them; and every norm bound entering
Lemma~\ref{lem:sign-witness} is polynomial in the stated parameter counts and
in $T$, enters only through the logarithm, and is absorbed into
$\tilde O(\cdot)$.

We begin with the world channel. A function of
$\mathcal F^{W,\xi^*}_{\rm agg}$ is indexed by $\theta\in\Theta$ and evaluates
at a query policy $\pi$ to
\[
F^{W}_{\theta}(\pi)
=\sum_{h=1}^H\E^{\pi}\bigl[\|\Delta W_h(o^A_h,a_h)\,q^{\xi^*}_{h-1}\|_{1,\mathcal V}\bigr],
\qquad
\Delta W_h(o^A,a):=W^{\theta}_h(o^A,a)-W^{\theta^*}_h(o^A,a),
\]
where the operators are evaluated at the realized observation--action pair.
Splitting the expectation over that pair gives
\begin{equation}\label{eq:pf-FW-split}
F^{W}_{\theta}(\pi)
=\sum_{h=1}^H\sum_{o^A,a}
\E^{\pi}\Bigl[\mathbf 1\{(o^A_h,a_h)=(o^A,a)\}\,
\bigl\|\Delta W_h(o^A,a)\,q^{\xi^*}_{h-1}\bigr\|_1\Bigr].
\end{equation}
Collect the differences into the block-diagonal operator
$\bar A_\theta:=\oplus_{h,o^A,a}\Delta W_h(o^A,a)$. Every summand in
\eqref{eq:pf-FW-split} is an $\ell_1$ norm of a linear image of
$\bar A_\theta$, so the class has the form treated by
Lemma~\ref{lem:sign-witness}(a), except that the queries produce expectations
rather than single vectors $v_x$. We return to this point after counting
dimensions.

The family $\{\bar A_\theta:\theta\in\Theta\}$ lies in the span of
\emph{pattern matrices}, the coordinate matrices supported on the positions
that the window structure allows to be nonzero, and we now count those
positions. On the reduced space $\bar S^+_{\rm red}$, an entry of
$W^{\theta}_h(o^A,a)$ is indexed by a source state $\bar s$, a hypothetical
adversary action $b$, and a destination state $\bar s'$. There are
$d_W=|S|\,|O_A|^{\kappa-1}|O_B|^{\kappa-1}$ source states. The window
components of $\bar s'$ are determined by those of $\bar s$ together with the
new observations, so the free destination indices are only the new latent
state $s'$ and the new adversary-side observation $o^B_{\rm new}$. Ranging
also over the block indices $(o^A,a)$, the possibly-nonzero positions at a
fixed step $h$ number at most
\begin{equation}\label{eq:pf-count}
\underbrace{|S|\,|O_A|^{\kappa-1}|O_B|^{\kappa-1}}_{\bar s}
\cdot\underbrace{|A|}_{a}
\cdot\underbrace{|B|}_{b}
\cdot\underbrace{|O_A|}_{o^A}
\cdot\underbrace{|O_B|}_{o^B_{\rm new}}
\cdot\underbrace{|S|}_{s'}
\;=\;|S|^2|A||B|\,|O_A|^{\kappa}|O_B|^{\kappa}\;=:\;p_W .
\end{equation}
Distinct steps occupy distinct blocks of $\bar A_\theta$, so the family spans
at most $p:=H\,p_W$ dimensions.

So far we have shown that the world-channel functions are expectations of
$\ell_1$ norms of linear images of operators drawn from an
$Hp_W$-dimensional subspace. It remains to supply the sign-witness matrices.
For a query policy $\pi$ and a witness $\theta$, define the block query
matrix
\begin{equation}\label{eq:pf-query}
X_{\pi,\theta}
:=\oplus_{h,o^A,a}\,
\E^{\pi}\Bigl[\mathbf 1\{(o^A_h,a_h)=(o^A,a)\}\;
\sigma_{h,o^A,a}(\tau)\,\bigl(q^{\xi^*}_{h-1}\bigr)^{\!\top}\Bigr],
\qquad
\sigma_{h,o^A,a}(\tau):=\mathrm{sign}\bigl(\Delta W_h(o^A,a)\,q^{\xi^*}_{h-1}\bigr),
\end{equation}
with the sign vectors chosen per trajectory for the witness $\theta$ and
frozen thereafter. The alignment \eqref{eq:sw-align} holds for these
matrices. Indeed, under the witness's own signs
$\sigma^\top\Delta W_h\,q=\|\Delta W_h\,q\|_1$, so
$\langle\bar A_\theta,X_{\pi,\theta}\rangle=F^{W}_{\theta}(\pi)$ by
\eqref{eq:pf-FW-split}. The domination \eqref{eq:sw-dom} holds as well.
Indeed, $|\sigma^\top v|\le\|v\|_1$ pointwise, the expectation preserves this
inequality, and therefore
$|\langle\bar A_{\theta'},X\rangle|\le F^{W}_{\theta'}(\pi')$ for every
parameter $\theta'$ and every previously frozen query matrix $X$ built from a
policy $\pi'$. For the norm bound, within each step $h$ the indicator blocks
partition the trajectory space; using
$\|q^{\xi^*}_{h-1}\|_2\le\|q^{\xi^*}_{h-1}\|_1=1$ and
$\|\sigma\|_2^2\le D_{\rm out}$, where $D_{\rm out}:=|B|\,d_W$ is the block
output dimension, we get $\|X_{\pi,\theta}\|_F^2\le H\,D_{\rm out}$. The
conclusion of Lemma~\ref{lem:sign-witness}(a) therefore holds with
$D_{\rm out}B_v^2$ replaced by $HD_{\rm out}$ inside the logarithm and with
$p=Hp_W$, giving, at $\alpha=1/(2T)$,
\begin{equation}\label{eq:pf-W-bound}
\dimE\bigl(\mathcal F^{W,\xi^*}_{\rm agg},\alpha\bigr)
\;\le\;12\,H p_W\,\log\Bigl(1+\frac{8\,HD_{\rm out}B_A^2}{\alpha^2}\Bigr)
\;=\;\tilde O\bigl(H\,|S|^2|A||B|\,|O_A|^\kappa|O_B|^\kappa\bigr),
\end{equation}
where $B_A$ bounds the Frobenius norms of the probability-entry block
operators and is polynomial in the counts.

Next we bound the adversary channel. The argument has the same two parts, a
linear representation and a parameter count. A function of
$\mathcal F^{G,\xi^*}_{\rm agg}$ evaluates to
$F^{G}_{\Phi}(\pi)=\sum_{h}\E^{\pi}\bigl[\|(G^{\Phi,\pi}_h-G^{\Phi^*,\pi}_h)\,
W^{\theta^*}_hq^{\xi^*}_{h-1}\|_1\bigr]$. For fixed $\pi$ the operator
$G^{\Phi,\pi}_h$ is linear in $\Phi$ (Theorem~\ref{thm:decomp}), so each
summand equals $\|U_h(x)\,\mathrm{vec}(\Phi-\Phi^*)\|_1$ for a design matrix
$U_h(x)$ computable from the query; its entries are built from the components
of $W^{\theta^*}_hq^{\xi^*}_{h-1}$ and the response features, both determined
by the query and the reference parameter. In the linear response model,
$\Phi_h$ carries $(|B|-1)\,d_{\rm adv}\le|B|\,d_{\rm adv}$ free parameters
per step, so stacking the steps gives a parameter vector of dimension at most
$H|B|\,d_{\rm adv}$; a step-shared $\Phi$ only lowers this. The expectation
device of \eqref{eq:pf-query}, with the per-step designs stacked into a
single $U(x)$ and the signs chosen per trajectory and frozen with the
witness, again verifies the alignment, the domination, and the norm bound,
and Lemma~\ref{lem:sign-witness}(b) yields
\begin{equation}\label{eq:pf-G-bound}
\dimE\bigl(\mathcal F^{G,\xi^*}_{\rm agg},\alpha\bigr)
=\tilde O\bigl(H\,d_{\rm adv}|B|\bigr).
\end{equation}

Both channels are now bounded, by \eqref{eq:pf-W-bound} and
\eqref{eq:pf-G-bound}, and neither bound involves $\xi^*$. Taking suprema
over $\xi^*$ and adding,
\[
d_E=d^{W}_{\rm agg}+d^{G}_{\rm agg}
=\tilde O\bigl(H(|S|^2|A||B|\,|O_A|^\kappa|O_B|^\kappa+d_{\rm adv}|B|)\bigr),
\]
which is the tabular bullet. For the linear and low-rank families only the
span count changes. The coordinate-structure hypothesis states that each
entry of $W^{\theta}_h$ is an affine function of the parameters, so on each
of the $|O_A|^{\kappa}|O_B|^{\kappa}$ window contexts the difference family
is an affine image of the parameter space. The per-step count $p_W$ of
\eqref{eq:pf-count} is therefore at most
$d_{\rm lin}\,|O_A|^{\kappa}|O_B|^{\kappa}$ in the linear case and
$r^2(|A|+|B|)\,|O_A|^{\kappa}|O_B|^{\kappa}$ in the low-rank case. The rest
of the argument is unchanged, and the adversary channel is identical in all
three cases.
\end{proof}

\section{Proof of Theorem~\ref{thm:upper}}
\label{app:upper}

We prove that the epoch-based algorithm with statistical termination achieves policy regret $C(H\sqrt{\bar\beta_T d_E T}\,\log T + mH\,S_T)$ with high probability, where $\bar\beta_T = H\beta_T$ and $S_T$ is the number of deployed policies (Lemma~\ref{lem:switch-count}).

Figure~\ref{fig:timeline} illustrates the epoch structure that the proof
analyzes. We first give an extended proof sketch in Appendix~\ref{app:sketch}, and then the subsequent subsections contain the full proof. 

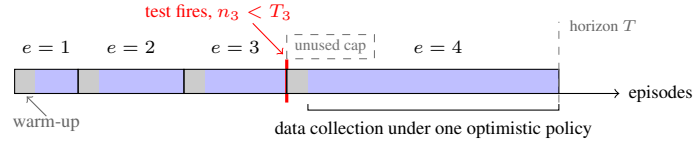
\begin{figure}[h]
\centering
\begin{tikzpicture}[font=\scriptsize]
\draw[->] (0,0) -- (8.0,0) node[right]{episodes};
% epoch 1
\fill[black!20] (0,0) rectangle (0.28,0.32);
\fill[blue!25] (0.28,0) rectangle (0.84,0.32);
% epoch 2
\fill[black!20] (0.84,0) rectangle (1.12,0.32);
\fill[blue!25] (1.12,0) rectangle (2.24,0.32);
% epoch 3, terminated early
\fill[black!20] (2.24,0) rectangle (2.52,0.32);
\fill[blue!25] (2.52,0) rectangle (3.6,0.32);
\draw[very thick, red] (3.6,-0.12) -- (3.6,0.45);
% unused cap, shown as a ghost above the bar
\draw[dashed, black!55] (3.6,0.46) rectangle (4.76,0.78);
\node[black!55, font=\tiny] at (4.18,0.62) {unused cap};
% epoch 4 begins immediately and is truncated by the horizon
\fill[black!20] (3.6,0) rectangle (3.88,0.32);
\fill[blue!25] (3.88,0) rectangle (7.2,0.32);
\draw (0,0) rectangle (0.84,0.32);
\draw (0.84,0) rectangle (2.24,0.32);
\draw (2.24,0) rectangle (3.6,0.32);
\draw (3.6,0) rectangle (7.2,0.32);
\node at (0.42,0.62) {$e=1$};
\node at (1.54,0.62) {$e=2$};
\node at (2.92,0.62) {$e=3$};
\node at (5.6,0.62) {$e=4$};
\node[red] at (2.7,1.08) {test fires, $n_3<T_3$};
\draw[red, ->, thin] (3.15,0.93) -- (3.55,0.52);
\draw[dashed, black!50] (7.2,-0.12) -- (7.2,0.95);
\node[black!60, font=\tiny, right] at (7.22,0.9) {horizon $T$};
\node at (0.42,-0.42) {\textcolor{black!60}{warm-up}};
\draw[->, thin, black!60] (0.42,-0.3) -- (0.16,-0.04);
\draw (3.88,-0.14) -- (3.88,-0.22) -- (7.2,-0.22) -- (7.2,-0.14);
\node[below] at (5.54,-0.24) {data collection under one optimistic policy};
\end{tikzpicture}
\caption{The epoch structure. Each epoch begins with $m-1$ warm-up episodes
(gray), then collects data under a single optimistic policy (blue) for at most
$T_e=2^e$ episodes. The statistical termination test can end an epoch early (red); the unused
remainder of the cap (dashed) is forgone, the next epoch begins immediately,
and terminated epochs keep their data. The final epoch is truncated by the
horizon $T$.}
\label{fig:timeline}
\end{figure}

A remark on policy transport before the proof. Because the adversary's
response changes with the learner's policy, one might expect a cost for
comparing adversary responses across deployed policies. Under
Assumption~\ref{ass:reg}(e) the likelihood conversion applies directly to the
joint operators at each data-collection policy, so no transport term enters the
final confidence radius (Remark~\ref{rem:transport-role});
posterior-Lipschitz structure instead underpins the well-behavedness of the
channel operator classes (Proposition~\ref{prop:eluder}) and the comparison
with flat-batching analyses \citep{sang2025pomg} in
Appendix~\ref{app:prior_work}.

\begin{remark}[On the choice of \(\theta\) in \(S^{\mathrm{ref}}\)]\label{rem:sref-theta} Assumption~\ref{ass:pl} imposes \eqref{eq:posterior-lipschitz} uniformly over \(\theta\in\Theta\).
This is the cleanest formulation. It makes the adversary class well-defined without knowing the true
\(\theta^*\), and it is the version needed in the regret analysis, where
confidence sets may contain many world models consistent with the data. In
parametric models, the resulting Lipschitz constant may depend on the sensitivity
of the map \(\theta\mapsto S_{\tau_B,h}^{\mathrm{ref},\theta}(\pi)\), which is
controlled when the transition and emission kernels vary continuously in \(\theta\)
and the reference policy \(\muref\) has full support.
The
assumption is satisfied, for instance, by adversaries who best-respond to a smoothed
estimate of the learner's policy under a fixed world model, or who use gradient-based
updates with bounded step sizes.
\end{remark}

\vspace*{-7pt}
\subsection{Extended Proof Sketch of Theorem~\ref{thm:upper}}
\label{app:sketch}
\vspace*{-3pt}

We outline the main argument, with the technical concentration and
Eluder summability lemmas developed in the subsections that follow. The proof follows
the same high-level template as optimistic model-based learning, but with one new
difficulty. The observable dynamics themselves depend on the learner's policy,
because the adversary's response depends on that policy.

\vspace*{-7pt}
\paragraph{From policy regret to simulation error.}
Consider an epoch \(e\), in which the algorithm deploys a single policy \(\pi_e\)
for \(n_e\le T_e=2^e\) data-collection episodes. Before collecting data, it runs
\(m-1\) warm-up episodes, so an \(m\)-memory adversary sees only \(\pi_e\) during
the data-collection phase. Thus, during the \(n_e\) data-collection episodes of epoch \(e\), the
true model \(\xi^*=(\theta^*,\Phi^*)\) induces the stationary adversary response
\(R_\infty(\pi_e)\), and
\vspace*{-2pt}
\[
V^{\pi_e}(\xi^*)=V^{\pi_e}(R_\infty(\pi_e)).
\]

\vspace*{-5pt}
The regret in a data-collection episode of epoch \(e\) is thus
\[
V^{\pi^*}(R_\infty(\pi^*))-V^{\pi_e}(R_\infty(\pi_e))
=
V^{\pi^*}(\xi^*)-V^{\pi_e}(\xi^*).
\]
Because \(\xi^*\in\mathcal C_{e-1}\) on the high-probability confidence event and
\((\pi_e,\xi_e)\) is chosen optimistically,
\[
V^{\pi_e}(\xi_e)
\ge
\max_{\pi} V^\pi(\xi^*)
\ge
V^{\pi^*}(\xi^*).
\]
Hence the regret in epoch \(e\) is not controlled by how good \(\pi_e\) is in the
optimistic model; optimism already handles that. What remains is only the amount by
which the optimistic model overestimates the value of the policy it selected:
\begin{equation}
\label{eqn:regret-ineq}
V^{\pi^*}(\xi^*)-V^{\pi_e}(\xi^*)
\le
V^{\pi_e}(\xi_e)-V^{\pi_e}(\xi^*).
\end{equation}

\vspace*{-6pt}
\paragraph{Why observable-operator errors are the right quantity.} 
Inequality~\eqref{eqn:regret-ineq} reduces the problem to a simulation question, namely how much the
value of the same policy \(\pi_e\) can differ between the optimistic model \(\xi_e\)
and the true model \(\xi^*\). For partially observable models, this difference is naturally
measured through the induced observable operators. Define
\begin{equation}
\label{eq:main-delta-e}
\Delta_e :=
\sum_{h=1}^H
\mathbb{E}_{\tau\sim P_{\xi^*}^{\pi_e}}
\left[
\left\|
J_h^{\xi_e,\pi_e}q_{h-1}^{\xi^*}
-
J_h^{\xi^*,\pi_e}q_{h-1}^{\xi^*}
\right\|_1
\right].
\end{equation}
This is the aggregate prediction error of the optimistic model along the trajectory
distribution generated by \(\pi_e\) under the true model. The simulation lemma says
that an error made at step \(h\) can affect rewards over the remaining horizon,
giving the horizon prefactor $V^{\pi_e}(\xi_e)-V^{\pi_e}(\xi^*)\le H\Delta_e.$ Thus the regret accumulated in epoch \(e\) is bounded by
\[
R_e
\le
n_e H\Delta_e+(m-1)H,
\]
where the second term is the cost of the warm-up episodes.

\vspace*{-6pt}
\paragraph{How the confidence set controls \(\Delta_e\).}
It remains to bound the cumulative size of the prediction errors \(\Delta_e\).
The confidence set at the beginning of epoch \(e\) is built from data collected
under the earlier policies \(\pi_1,\ldots,\pi_{e-1}\). Since \(\xi_e\in
\mathcal C_{e-1}\), the optimistic model must fit those past data nearly as well
as the true model. Consequently (Lemma~\ref{lem:conc}), the cumulative KL divergence under the historical
policies is small:
\begin{equation}
\label{eq:main-historical-kl}
\sum_{k<e}n_k
\mathrm{KL}\!\left(P_{\xi^*}^{\pi_k}\,\|\,P_{\xi_e}^{\pi_k}\right)
\lesssim \beta_T .
\end{equation}
The KL-to-operator conversion (Assumption~\ref{ass:reg}(e)) turns the likelihood statement~\eqref{eq:main-historical-kl} into a prediction-error statement, and weak revealing (Assumption~\ref{ass:reveal}) is what makes such a conversion available. If two models induce different observable
operators, then a \(\kappa\)-step window of learner observations detects that
difference with strength \(\alpha_\kappa\). Thus likelihood fit under past policies
implies small observable-operator error under those past policies.

There is one additional complication that is absent in single-agent POMDPs, and
it is where the epoch design earns its keep. Past consistency constrains
\(\xi_e\) only at the \emph{historical} policies; it says nothing about the
error of \(\xi_e\) at the freshly chosen \(\pi_e\). Optimism deliberately
steers toward such unconstrained directions, and under a fixed epoch schedule a
single ill-informed commitment would run for \(\Theta(T)\) episodes;
Proposition~\ref{prop:fixed-fails} shows that the fixed schedule, under a
valid tie-breaking of the planning oracle, provably
incurs \(\Omega(T)\) policy regret. The statistical termination test closes
exactly this gap. If the deployed model's prediction error at \(\pi_e\) is
large, the epoch's own data refute it at rate
\(\mathrm{KL}_e:=\mathrm{KL}(P^{\pi_e}_{\xi^*}\|P^{\pi_e}_{\xi_e})
\ge c_{\rm KL}\,s_e\) per episode, where \(s_e\) is the stepwise squared
prediction error (written \(\mathcal E_e^2\) in Definition~\ref{def:errors} below), and
the test fires; consequently every epoch,
terminated or complete, satisfies the cap
\[
n_e\Delta_e^2 \le H\, n_e s_e \le \tilde O(H\beta_T/c_{\rm KL})
\]
with high probability (Lemma~\ref{lem:epoch-cap}). The combination of past consistency and the per-epoch
cap is precisely what the capped summability lemma
(Lemma~\ref{lem:eluder-summ}) consumes; the resulting effective statistical
radius is $\bar\beta_T := H\beta_T$.

\vspace*{-6pt}
\paragraph{Eluder summability and the final regret bound.}
Each optimistic model \(\xi_e\) is consistent with all past prediction queries
up to radius \(\bar\beta_T\), and every epoch's own contribution is capped at
\(\tilde O(H\beta_T/c_{\rm KL})\) by the termination test. The aggregate Eluder dimension
then limits how often the learner can encounter a large new prediction error.
Lemma~\ref{lem:eluder-total} gives
\vspace*{-5pt}
\begin{equation}
\label{eq:main-eluder-total}
\sum_{e=1}^E n_e\Delta_e^2
\le
C_1 d_E\bar\beta_T\log^2 T .
\end{equation}
This is the central statistical estimate. It says that although an individual epoch may have nontrivial prediction error, the weighted square-sum of these errors over
all epochs is controlled by the aggregate Eluder dimension.

We now return to the regret decomposition. Using Cauchy--Schwarz,
\(\sum_{e=1}^E n_e\le O(T)\), and \eqref{eq:main-eluder-total}, we get
\vspace*{-2pt}
\[
\sum_{e=1}^E n_e\Delta_e
\le
O\!\left(\sqrt{C_1 d_E T\bar\beta_T}\,\log T\right).
\]
Finally, summing per-epoch regret \(R_e\le n_eH\Delta_e+(m-1)H\) over
the \(S_T\) epochs gives
$PR(T)
\le
H\sum_{e} (n_e\Delta_e
+
(m-1))
\le
C(H\sqrt{\bar\beta_T d_E T}\log T+mH S_T),$
and Lemma~\ref{lem:switch-count} bounds \(S_T\). The design leans on both
halves of the termination test. It must fire when the data refute the deployed
model, else a fixed schedule suffers \(\Omega(T)\)
(Proposition~\ref{prop:fixed-fails}), and it must fire rarely, each firing
requiring \(\Omega(\beta_e)\) bits of evidence, which is what keeps the
warm-up cost dominated (Lemma~\ref{lem:switch-count}).

\subsection{The Epoch-Based Algorithm Structure}
\label{app:alg-structure}

Algorithm~\ref{alg:omle} divides the $T$ episodes into epochs of geometrically increasing cap $T_e = 2^e$ for $e = 1, 2, \ldots$; because the termination test can end an epoch early, the number of epochs started by time $T$ is the (data-dependent) deployed-policy count $S_T$, not $\lceil\log_2 T\rceil$. At the beginning of epoch $e$, the algorithm:
\begin{itemize}
\item Computes the MLE $\hat{\xi}_{e-1}$ from the cumulative dataset $\mathcal{D}_{e-1}$ containing all trajectories collected in epochs $1, \ldots, e-1$
\item Forms the confidence set $\mathcal{C}_{e-1} = \{\xi : L_{e-1}(\xi) \geq L_{e-1}(\hat{\xi}_{e-1}) - \beta_e\}$ where $\beta_e = c(\log \mathcal{N}(\epsilon_e;\Xi) + \log(1/\delta_e) + mB_0)$, $\delta_e = \delta/(2e^2)$, $\epsilon_e = 2^{-\lceil\log_2(|\mathcal D_{e-1}|+2)\rceil}$
\item Selects one policy-parameter pair $(\pi_e, \xi_e)$ that maximizes value within $\mathcal{C}_{e-1}$
\item Executes $\pi_e$ for $m-1$ warm-up episodes (discarding data), then for at most $T_e$ data-collection episodes, recomputing the MLE after each episode and \emph{terminating the epoch} as soon as $L_t(\xi_e) < L_t(\hat\xi_t) - \tau_e(t)$, where $\tau_e(t) = 2\beta_e(t) + C_0B_0\log(2t(t{+}1)/\delta_e)$ and $\beta_e(t)$ is the confidence radius recomputed at the current dataset size (Appendix~\ref{app:cap-switch}); the realized number of data-collection episodes is $n_e \le T_e$
\item Updates the cumulative dataset: $\mathcal{D}_e = \mathcal{D}_{e-1} \cup \{\text{new trajectories from epoch } e\}$ (terminated epochs' data included)
\end{itemize}

The total number of data-collection episodes satisfies \(\sum_{e} n_e\le T\)
by definition of the horizon (the final epoch is truncated at time \(T\)).
Scheduled completions, meaning epochs that exhaust their cap \(T_e=2^e\),
number at most \(\log_2(2T)\), since a completion at index \(e\) requires \(2^e\le T\).
The algorithm deploys \(S_T\) policies in total (one per epoch), where \(S_T\) counts
scheduled completions, early terminations, and the final truncated epoch, and
is bounded in Lemma~\ref{lem:switch-count}. We define the horizon-level radius
\[
\beta_T := c\bigl(\log \mathcal{N}(\epsilon_T;\Xi) + 5\log(2T) + \log(1/\delta) + mB_0\bigr),
\qquad \epsilon_T := 1/(2T+4),
\]
where the scale \(\epsilon_T\) lower-bounds every dyadic scale the algorithm
can use (all realized scales are at least \(2^{-\lceil\log_2(T+2)\rceil}\ge
1/(2T+4)\), since dataset sizes never exceed \(T\)).
Every realized radius is dominated by it. Since $|\mathcal D_{e-1}|\le T$ and $e\le S_T\le T+1\le 2T$ (every started epoch except possibly the final horizon-truncated one collects at least one data-collection episode), we have $\log(1/\delta_e)=\log(2e^2/\delta)\le 3\log(2T)+\log(1/\delta)$, so $\beta_e\le\beta_T$; likewise, for the running radii of Appendix~\ref{app:cap-switch}, $\log(2t(t{+}1)/\delta_e)\le\log(32T^4/\delta)\le 5\log(2T)+\log(1/\delta)$ for $t\le T$, so $\beta_e(t)\le\beta_T$ for every reachable pair $(e,t)$. By construction $\log T\le\beta_T/c$, and $\beta_T = O(\log \mathcal{N}(\epsilon_T;\Xi) + \log T + \log(1/\delta) + mB_0)$.

\subsection{Auxiliary Lemmas}

We state the simulation lemma and likelihood concentration lemma before the main proof.

\begin{lemma}[Simulation Lemma]\label{lem:simulation}
For any policy \(\pi\) and model \(\xi\in\Xi\),
\[
|V^\pi(\xi)-V^\pi(\xi^*)|
\le
H\sum_{h=1}^H
\mathbb{E}_{\tau_{h-1}\sim P^{\pi}_{\xi^*}}
\left[
\left\|
J^{\xi,\pi}_h q^{\xi^*}_{h-1}
-
J^{\xi^*,\pi}_h q^{\xi^*}_{h-1}
\right\|_1
\right].
\]
\end{lemma}

\begin{proof}
The proof is the standard telescoping argument for controlled observable
operators. Fix \(\pi\), and let \(q^{\xi^*}_{h-1}\) be the predictive state under
the true model after the learner observation-action history up to step \(h-1\).
Replacing the true step-\(h\) operator \(J^{\xi^*,\pi}_h\) by
\(J^{\xi,\pi}_h\), while keeping the same reference predictive state, changes the
one-step predictive distribution by
\[
\left\|
J^{\xi,\pi}_h q^{\xi^*}_{h-1}
-
J^{\xi^*,\pi}_h q^{\xi^*}_{h-1}
\right\|_1 .
\]
By telescoping over \(h=1,\ldots,H\), the \(\ell_1\) distance between the
trajectory laws induced by \(\xi\) and \(\xi^*\) under the fixed policy \(\pi\),
which equals twice their total variation distance, is bounded by the sum of
these one-step prediction errors, averaged over the true-model trajectory
distribution. Since rewards are in \([0,1]\), an error in
the predictive distribution at any step can affect at most \(H\) units of remaining
cumulative reward. Therefore,
\[
|V^\pi(\xi)-V^\pi(\xi^*)|
\le
H\sum_{h=1}^H
\mathbb{E}_{\tau_{h-1}\sim P^{\pi}_{\xi^*}}
\left[
\left\|
J^{\xi,\pi}_h q^{\xi^*}_{h-1}
-
J^{\xi^*,\pi}_h q^{\xi^*}_{h-1}
\right\|_1
\right].
\]
For POMGs, the controlled observable operators are well-defined by the causal
decomposition in Theorem~\ref{thm:decomp}; the telescoping argument is otherwise
the same as in the POMDP/OOM setting~\citep{liu2023optimistic}.
\end{proof}

\begin{lemma}[Likelihood Concentration]\label{lem:conc}
Let
\[
K_{e-1}(\xi)
:=
\sum_{k<e}n_k\,
\mathrm{KL}\!\left(P^{\pi_k}_{\xi^*}\,\|\,P^{\pi_k}_\xi\right).
\]
Under Assumption~\ref{ass:reg}(a)--(c) and (f), with probability at least
\(1-\delta_e/2\), uniformly over \(\xi\in\Xi\), the following two bounds hold
simultaneously:
\begin{itemize}[leftmargin=18pt,itemsep=0pt]
\item[(i)] \emph{(validity; absolute constants)}
\(L_{e-1}(\xi)-L_{e-1}(\xi^*) \le \beta_e/2\);
\item[(ii)] \emph{(evidence accumulation)}
\(L_{e-1}(\xi)-L_{e-1}(\xi^*) \le -\frac12 K_{e-1}(\xi)+C_B\,\beta_e/2\)
for a constant \(C_B\ge1\) depending only on \(B_0\).
\end{itemize}
Consequently, if \(\xi\in\mathcal C_{e-1}\), then
\(K_{e-1}(\xi)\le C\beta_e\) for a constant \(C>0\) depending only on
\(B_0\).
\end{lemma}

\begin{proof}
Fix an epoch \(e\) and a model \(\xi\in\Xi\). For a trajectory \(\tau\) collected
under policy \(\pi_k\), define
\begin{equation}
\label{eq:conc-Z-def}
Z_k(\xi;\tau)
:=
\log\frac{P_{\xi}^{\pi_k}(\tau)}{P_{\xi^*}^{\pi_k}(\tau)}
+
\mathrm{KL}\!\left(P_{\xi^*}^{\pi_k}\,\|\,P_{\xi}^{\pi_k}\right).
\end{equation}
Conditioned on the filtration before this trajectory is sampled, the policy
\(\pi_k\) is fixed and the trajectory is distributed according to
\(P_{\xi^*}^{\pi_k}\). Hence
\begin{equation}
\label{eq:conc-Z-centered}
\mathbb E_{\xi^*}\!\left[Z_k(\xi;\tau)\mid\mathcal F_{k-1}\right]=0,
\end{equation}
so the sequence of \(Z_k\)'s over all trajectories collected before epoch \(e\) is a
martingale difference sequence.

By Assumption~\ref{ass:reg}(c), log-likelihood ratios are uniformly bounded:
\begin{equation}
\label{eq:conc-logratio-bound}
\left|
\log\frac{P_{\xi}^{\pi_k}(\tau)}{P_{\xi^*}^{\pi_k}(\tau)}
\right|
\le B_0 .
\end{equation}
Moreover, the KL term is also bounded by \(B_0\), so
\begin{equation}
\label{eq:conc-Z-bound}
|Z_k(\xi;\tau)|\le 2B_0 .
\end{equation}
For 
$K_{e-1}(\xi)
:=
\sum_{k<e}n_k\,
\mathrm{KL}\!\left(P_{\xi^*}^{\pi_k}\,\|\,P_{\xi}^{\pi_k}\right),$ 
the conditional variance satisfies
\begin{equation}
\label{eq:conc-variance}
\sum_{k<e}\sum_{i=1}^{n_k}
\mathbb E_{\xi^*}\!\left[
Z_{k,i}(\xi)^2\mid\mathcal F_{k,i-1}
\right]
\le
C_{\rm var} K_{e-1}(\xi),
\end{equation}
for a constant \(C_{\rm var}=O(B_0)\) depending only on the likelihood-ratio
bound (recall \(B_0\ge1\)). This is the
standard Bernstein/Freedman variance control for bounded likelihood ratios.

Freedman's inequality in its fixed-parameter exponential-supermartingale
form \citep{freedman1975tail}, which applies to the adapted policy sequence and
the stopping-time epoch lengths, together with \eqref{eq:conc-Z-bound} and
\eqref{eq:conc-variance}, gives that for any \(u>0\), with probability at least
\(1-e^{-u}\),
\begin{equation}
\label{eq:conc-freedman}
\sum_{k<e}\sum_{i=1}^{n_k} Z_{k,i}(\xi)
\le
\frac12 K_{e-1}(\xi)+C u,
\end{equation}
where \(C=O(B_0)\) depends only on \(B_0\) and \(C_{\rm var}\). By the definition of \(Z_k\) in
\eqref{eq:conc-Z-def}, this is equivalent to
\begin{equation}
\label{eq:conc-one-model}
L_{e-1}(\xi)-L_{e-1}(\xi^*)
+
K_{e-1}(\xi)
\le
\frac12 K_{e-1}(\xi)+C u,
\end{equation}
or equivalently,
\begin{equation}
\label{eq:conc-one-model-final}
L_{e-1}(\xi)-L_{e-1}(\xi^*)
\le
-\frac12 K_{e-1}(\xi)+C u .
\end{equation}

To make the bound uniform over \(\Xi\), we must account for the fact that
the scale \(\epsilon_e=2^{-\ell_e}\), \(\ell_e:=\lceil\log_2(|\mathcal
D_{e-1}|+2)\rceil\), is data-dependent, so the realized net is not known in
advance. The dyadic grid makes the candidate nets a deterministic family. Since
\(|\mathcal D_{e-1}|\le\sum_{k<e}2^k\le 2^e\), the realized level satisfies
\(\ell_e\le e+1\). For each candidate level \(\ell\in\{1,\ldots,e+1\}\), let
\(\mathcal N^{(\ell)}\) be a fixed \(2^{-\ell}\)-net of \(\Xi\) in the operator
norm from Assumption~\ref{ass:reg}, with
\begin{equation}
\label{eq:conc-net-size}
|\mathcal N^{(\ell)}|
\le
\mathcal N(2^{-\ell};\Xi,\|\cdot\|_{\rm op}).
\end{equation}
Apply \eqref{eq:conc-one-model-final} to every \(\xi\in\mathcal N^{(\ell)}\)
and every level \(\ell\) with
\begin{equation}
\label{eq:conc-u-choice}
u_\ell=\log|\mathcal N^{(\ell)}|+\log\bigl(4(\ell{+}1)(\ell{+}2)/\delta_e\bigr).
\end{equation}
A union bound over the net at each level, then over levels
(\(\sum_{\ell\ge1}\tfrac{1}{(\ell+1)(\ell+2)}=\tfrac12\)), gives
\eqref{eq:conc-one-model-final} simultaneously at every level with probability
at least \(1-\delta_e/8\); on this event the bound holds in particular at the
realized level \(\ell_e\). The level overhead is affordable:
\(2\log\bigl(4(\ell_e{+}1)(\ell_e{+}2)\bigr)\le 2\log(1/\delta_e)+7\)
for every \(\ell_e\le e+1\), since \((\ell_e{+}1)(\ell_e{+}2)\le(e{+}2)(e{+}3)\le
12e^2\) for \(e\ge1\) and \(\log(1/\delta_e)\ge\log(2e^2)\). Hence
\(Cu_{\ell_e}\le C_B\,\beta_e/2\) for a constant \(C_B\ge1\) depending only on
\(B_0\), since \(\beta_e\ge
c(\log|\mathcal N^{(\ell_e)}|+\log(1/\delta_e)+mB_0)\) with \(mB_0\ge1\). This
yields part (ii) at the realized net.

Part (i) does not require Freedman and carries absolute constants. For a fixed
\(\xi\), conditionally on the filtration each factor
\(\E_{\xi^*}[\exp(\tfrac12\log(P_\xi/P_{\xi^*})(\tau))\mid\mathcal F]\)
equals the Hellinger affinity of the two trajectory laws, which is at most
\(1\); hence \(\exp(\tfrac12(L_{e-1}(\xi)-L_{e-1}(\xi^*)))\) is a
supermartingale with initial value \(1\), and Markov's inequality gives
\(\Prob[L_{e-1}(\xi)-L_{e-1}(\xi^*)\ge 2u]\le e^{-u}\). A union bound over
\(\mathcal N^{(\ell)}\) and the levels \(\ell\le e+1\), with \(u_\ell\) as in
\eqref{eq:conc-u-choice}, gives part (i) at the realized net with probability
at least \(1-\delta_e/8\) and deviation
\(2u_{\ell_e}\le\beta_e/2\), using the level-overhead estimate above together
with \(c\ge16\) and the floor \(mB_0\ge1\).

The Lipschitz regularity in Assumption~\ref{ass:reg}(b) extends both
inequalities from the net to all \(\xi\in\Xi\). Replacing a model by its nearest net point changes each of the \(|\mathcal D_{e-1}|\) per-trajectory log-likelihoods by at most \(\epsilon_e L_{\rm lik}\), where \(L_{\rm lik}\) denotes the Lipschitz constant of Assumption~\ref{ass:reg}(b). Since the realized scale satisfies \(\epsilon_e\le 1/(|\mathcal D_{e-1}|+2)\), the total discretization error in \(L_{e-1}\) and \(K_{e-1}\) is at most \(L_{\rm lik}\), which the normalization in Assumption~\ref{ass:reg}(c) bounds by \(B_0\), hence by the \(mB_0\) floor; it is absorbed by the same \(c\ge16\) slack, keeping \(c\) an absolute constant. Uniformly over
\(\xi\in\Xi\), with probability at least \(1-\delta_e/2\), both bounds hold:
\begin{equation}
\label{eq:likelihood-conc-kl}
L_{e-1}(\xi)-L_{e-1}(\xi^*)
\le
\min\Bigl(\;\frac{\beta_e}{2},\;
-\frac12 K_{e-1}(\xi)+C_B\,\frac{\beta_e}{2}\Bigr).
\end{equation}

It remains to verify the stated consequence.
If \(\xi\in\mathcal C_{e-1}\), then by definition of
\(\mathcal C_{e-1}\) and the optimality of \(\hat\xi_{e-1}\),
\[
L_{e-1}(\xi)
\ge
L_{e-1}(\hat\xi_{e-1})-\beta_e
\ge
L_{e-1}(\xi^*)-\beta_e .
\]
Combining this with part (ii) of \eqref{eq:likelihood-conc-kl} gives
\[
-\beta_e
\le
L_{e-1}(\xi)-L_{e-1}(\xi^*)
\le
-\frac12 K_{e-1}(\xi)+C_B\,\frac{\beta_e}{2}.
\]
Therefore \(K_{e-1}(\xi)\le (2+C_B)\beta_e\), and the claimed bound
\(K_{e-1}(\xi)\le C\beta_e\) holds with \(C:=2+C_B\), a constant depending
only on \(B_0\).
\end{proof}

\subsection{Cumulative Confidence Sets Contain True Parameters}

\begin{lemma}[Confidence Set Validity]\label{lem:confidence}
With probability at least $1-\delta$, we have $\xi^* \in \mathcal{C}_{e-1}$ for all epochs $e$ simultaneously.
\end{lemma}

\begin{proof}
Consider epoch \(e\). The confidence set uses threshold \(\beta_e\):
\[
\mathcal{C}_{e-1}
=
\left\{\xi \in \Xi :
L_{e-1}(\xi) \geq L_{e-1}(\hat{\xi}_{e-1}) - \beta_e
\right\}.
\]

On the event of Lemma~\ref{lem:conc}, part (i) holds uniformly over
\(\xi\in\Xi\), and hence also for the data-dependent MLE \(\hat\xi_{e-1}\):
\[
L_{e-1}(\hat\xi_{e-1})-L_{e-1}(\xi^*)
\le
\frac{\beta_e}{2}.
\]
Thus \(L_{e-1}(\xi^*)\ge L_{e-1}(\hat\xi_{e-1})-\beta_e\), so
\(\xi^*\in\mathcal C_{e-1}\) with probability at least \(1-\delta_e/2\). Equivalently,
\[
\Prob_{\xi^*}\left[
L_{e-1}(\hat{\xi}_{e-1}) - L_{e-1}(\xi^*) > \beta_e
\right]
\leq \frac{\delta_e}{2}.
\]
Setting \(\delta_e=\delta/(2e^2)\) and applying a union bound over all epochs gives
\[
\Prob[\xi^* \in \mathcal{C}_{e-1} \text{ for all epochs } e]
\geq
1-\sum_{e=1}^{\infty}\frac{\delta_e}{2}
=
1-\frac{\delta}{2}\sum_{e=1}^\infty\frac{1}{2e^2}
\geq
1-\delta .
\]
\end{proof}

For the remainder of this proof, we condition on the high-probability event that $\xi^* \in \mathcal{C}_{e-1}$ for all epochs $e$.

\subsection{Per-Epoch Regret Decomposition}

Fix an epoch $e$ where the algorithm executes policy $\pi_e$ for $n_e \le T_e = 2^e$ data-collection episodes (after warm-up). Let $\xi_e = (\theta_e, \Phi_e)$ be the optimistically selected model. For each episode $t$ in this epoch (after the warm-up phase), the instantaneous regret can be decomposed as:
\begin{align}
&V^{\pi^*}(R_\infty(\pi^*)) - V^{\pi_e}(R_\infty(\pi_e)) \nonumber\\
&\quad = \underbrace{[V^{\pi^*}(\xi^*) - V^{\pi_e}(\xi_e)]}_{\text{optimism gap}} + \underbrace{[V^{\pi_e}(\xi_e) - V^{\pi_e}(\xi^*)]}_{\text{simulation error}}, \label{eq:decomp-epoch}
\end{align}
where we use the key equality $V^\pi(\xi^*) = V^\pi(R_\infty(\pi))$ (the model $\xi^*$ captures the adversary's stationary response to any fixed policy). This equality holds because, by Assumption~\ref{ass:reg}(f), after the \(m-1\) warm-up episodes, every data-collection episode in epoch \(e\) is generated under the stationary response \(R_\infty(\pi_e)\). The value $V^{\pi_e}(\xi^*)$ is defined as the expected cumulative reward when the learner plays $\pi_e$ and the adversary responds according to $\Phi^*$; after warm-up, this equals $V^{\pi_e}(R_\infty(\pi_e))$ since the adversary is in its stationary regime for all $n_e$ data episodes.

\textbf{Optimism gap:} By the optimistic selection procedure, since $\xi^* \in \mathcal{C}_{e-1}$ (by Lemma~\ref{lem:confidence}) and we maximize over all policies within the confidence set:
\[
V^{\pi_e}(\xi_e) \geq \max_{\pi \in \Pi} V^\pi(\xi^*) \geq V^{\pi^*}(\xi^*) = V^{\pi^*}(R_\infty(\pi^*)).
\]
Therefore, the optimism gap is non-positive and contributes zero regret.

\textbf{Warm-up cost:} The $m-1$ warm-up episodes allow the adversary to stabilize to $R_\infty(\pi_e)$. Each warm-up episode incurs at most $H$ regret, contributing $(m-1)H$ per epoch.

\textbf{Simulation error:} This measures how much the optimistic model $\xi_e$ overestimates the value of $\pi_e$ compared to the true model $\xi^*$. This is the key term we must bound carefully.

Define the \emph{operator error} for epoch $e$ as
\begin{equation}
\label{eq:appendix-delta-e}
\Delta_e := \sum_{h=1}^H \mathbb{E}_{\tau_{h-1}\sim P^{\pi_e}_{\xi^*}}\!\left[\|J^{\xi_e,\pi_e}_h q^{\xi^*}_{h-1} - J^{\xi^*,\pi_e}_h q^{\xi^*}_{h-1}\|_1\right].
\end{equation}
By the simulation lemma (Lemma~\ref{lem:simulation}):
\[
V^{\pi_e}(\xi_e) - V^{\pi_e}(\xi^*) \;\leq\; H\,\Delta_e.
\]

The total regret in epoch $e$ is thus:
\[
R_e \;\le\; n_e \cdot H\,\Delta_e + (m-1)H,
\]
where $(m-1)H$ accounts for warm-up.

\subsection{Bounding Operator Errors via Eluder Dimension}
It remains to control the cumulative prediction errors \(\Delta_e\) of the optimistic
models selected across epochs. This is the main statistical step of the proof, and it
rests on three ingredients. First, the confidence-set construction ensures that each
optimistic model \(\xi_e\) remains consistent with the data collected under earlier
epoch policies (\emph{past consistency}). Second, and this is where the statistical
termination test is essential, past consistency alone cannot control the error
of \(\xi_e\) at the freshly selected policy \(\pi_e\).
Proposition~\ref{prop:fixed-fails} shows that without the termination test the
algorithm, under a valid tie-breaking of its planning oracle, provably suffers
linear policy regret. The termination test supplies the missing ingredient, a \emph{per-epoch cap}
on each epoch's contribution (Lemma~\ref{lem:epoch-cap}). Third, a capped weighted
summability argument (Lemma~\ref{lem:eluder-summ}) combines the two constraints with
the aggregate Eluder dimension to give
\begin{equation}
\label{eq:appendix-eluder-summary}
\sum_{e} n_e \Delta_e^2
\le
\tilde O(d_E\bar\beta_T),
\qquad
\bar\beta_T=H\beta_T .
\end{equation}

\subsection{Why Data-Dependent Termination Is Necessary}
\label{app:necessity}

The following proposition shows that the statistical termination test cannot be
dispensed with. Under the fixed schedule \(T_e=2^e\) alone, optimism can commit
to a direction the data have never constrained for an entire epoch, and the
final epoch has length \(\Theta(T)\).

\begin{proposition}[Failure of the fixed epoch schedule]\label{prop:fixed-fails}
For every \(T\ge2\) there is a POMG instance and model class satisfying
Assumptions~\ref{ass:pl}, \ref{ass:reveal}, and~\ref{ass:reg}, namely the
lower-bound family of Appendix~\ref{app:lower} with a constant adversary
(\(m=1\), \(L=0\), \(\kappa=1\), \(H=2\)) and \(d=\lceil\log_2 T\rceil+1\)
arms, so that \(d_E=\Theta(\log T)\), on which Algorithm~\ref{alg:omle}
\emph{without} the termination test (executing every epoch for its full
scheduled length \(T_e=2^e\), with any maximizer selection that prefers
unplayed arms under exact ties, breaking remaining ties by largest index; the
algorithm specification leaves such ties unresolved, so this is a valid
instantiation of the planning oracle) incurs
\[
\E[PR(T)] \;\ge\; T/4 .
\]
\end{proposition}

\begin{proof}
Take the construction of Appendix~\ref{app:lower} with \(d\) arms and a
constant adversary, true means \(\mu^*_1=3/4\) and \(\mu^*_i=1/2\) for
\(i\ge 2\), and model class \(\Theta=[1/8,\,7/8]^d\). All outcome probabilities
are bounded away from \(0\) and \(1\), giving Assumption~\ref{ass:reg}(c);
Assumptions~\ref{ass:reg}(a)--(f), \ref{ass:pl} (with \(L=0\)), and
\ref{ass:reveal} (with \(\kappa=1\)) hold with absolute constants exactly as
verified in Appendix~\ref{app:lower} (the policy class is again the \(d\)
routing policies). The adversary channel is trivial, so \(d_E=d^{W}_{\rm agg}\);
repeated queries are dependent at every scale, so \(d_E\le d\), and the \(d\)
single-arm queries are \(1/(2T)\)-independent (each arm's mean varies over an
interval of length \(3/4\), so witness values are of constant order, above the
evaluation scale \(1/(2T)\) for every \(T\ge2\)), giving
\(d_E=\Theta(d)=\Theta(\log T)\). The likelihood factorizes
over arms, and an arm that has never been played contributes no likelihood
terms; hence for any unplayed arm \(i\), every value \(\mu_i\in[1/8,7/8]\)
attains the maximum likelihood, and the optimistic value of the policy
``commit to arm \(i\)'' is \(7/8\). Any arm's optimistic value is at most
\(7/8\), and arm \(1\)'s true value is \(3/4<7/8\). Consequently a maximizer always exists among the unplayed arms, and under the
stated tie rule optimistic planning selects one whenever an unplayed arm
exists. Without the termination
test the algorithm runs at most \(\lceil\log_2 T\rceil\) epochs, deploying one
arm per epoch; since \(d>\lceil\log_2 T\rceil\), an unplayed arm exists at
every epoch start, and the largest-index tie rule never reaches arm \(1\).
Every episode is therefore spent on an arm \(i\ge2\) and incurs instantaneous
policy regret \(\mu^*_1-\mu^*_i=1/4\), so \(\E[PR(T)]\ge T/4\).
\end{proof}

The defect this exposes is not specific to our setting. Holding one function
fixed across a block of \(2^e\) identically weighted queries invalidates
weighted analogues of the standard Eluder summability bound (the arm class
above satisfies the past-consistency hypothesis with any radius, yet
\(\sum_e T_e f_e(x_e)^2=\Theta(T)\)). It also explains why data-dependent
update triggers are ubiquitous in low-switching reinforcement learning
\citep{bai2019provably,qiao2022sample}. Fixed schedules do not mix with
optimistic planning that is updated only infrequently.

\subsection{Per-Epoch Cap and Switch Count}
\label{app:cap-switch}

Throughout this subsection, condition on the event of
Lemma~\ref{lem:conc} at every epoch boundary (which contains the event of
Lemma~\ref{lem:confidence} and, through part (i), bounds the MLE-to-truth gap by
\(\beta_e/2\)), strengthened as follows. For an episode boundary
\(t\) of epoch \(e\), let \(N_{e,t}:=|\mathcal D_{e-1}|+t\) denote the current
dataset size and define the running radius
\[
\beta_e(t)
:=
c\Bigl(\log \mathcal N\bigl(2^{-\lceil\log_2(N_{e,t}+2)\rceil};\Xi\bigr)
+\log\bigl(2t(t{+}1)/\delta_e\bigr)+mB_0\Bigr),
\]
which is nondecreasing in \(t\) and satisfies \(\beta_e(t)\ge\beta_e\) (the
dyadic scale is finer and the per-boundary budget smaller) as well
as \(\beta_e(t)\le\beta_T\) for every reachable \(t\), by the definition of
\(\beta_T\) in Appendix~\ref{app:alg-structure}. Within epoch \(e\), condition
on the epoch-start \(\sigma\)-field. The dataset \(\mathcal D_{e-1}\) is then
fixed, so the dyadic covering net at scale
\(2^{-\lceil\log_2(N_{e,t}+2)\rceil}\) is deterministic for every \(t\), and
the epoch's data-collection episodes are the only remaining randomness. For
each \(\xi\) in the boundary-\(t\) net, the conditional Hellinger-affinity
supermartingale over the epoch's episodes (as in part (i) of
Lemma~\ref{lem:conc}) bounds the epoch-local gap
\(G_t(\xi):=L_t(\xi)-L_t(\xi^*)-\bigl(L_{e-1}(\xi)-L_{e-1}(\xi^*)\bigr)\) by
\(2u_t\) with conditional failure probability \(e^{-u_t}\); a union bound over
the net with budget \(\delta_e/(2t(t{+}1))\) at boundary \(t\), integrated over
the conditioning, sums to at most \(\delta_e/2\) per epoch. Combining with the
uniform epoch-boundary bound \(L_{e-1}(\xi)-L_{e-1}(\xi^*)\le\beta_e/2\) of
Lemma~\ref{lem:conc}(i) and the net-extension step, on the resulting event, for
every epoch \(e\) and every episode \(t\) within it,
\(L_t(\hat\xi_t)-L_t(\xi^*)\le\beta_e/2+2u_t+2B_0\le\beta_e/2+\beta_e(t)/8\),
the last inequality because
\(2u_t+2B_0\le 2\bigl(\log|\mathcal N_t|+\log(2t(t{+}1)/\delta_e)+mB_0\bigr)
\le(2/c)\,\beta_e(t)\le\beta_e(t)/8\) for \(c\ge16\) (the net size is counted
by the covering term of \(\beta_e(t)\), and \(B_0\le mB_0\)). The termination
test of Algorithm~\ref{alg:omle} uses the threshold
\(\tau_e(t):=2\beta_e(t)+C_0B_0\log(2t(t{+}1)/\delta_e)\), where \(C_0\) is the
absolute constant from \eqref{eq:cap-freedman} below; we fix \(c\ge16\)
throughout (as in the proof of Lemma~\ref{lem:conc}). Write
\(\mathrm{KL}_e := \mathrm{KL}(P^{\pi_e}_{\xi^*}\,\|\,P^{\pi_e}_{\xi_e})\) and
\(s_e := \sum_{h=1}^H\E_{\pi_e,\xi^*}\bigl[\|J^{\xi_e,\pi_e}_hq^{\xi^*}_{h-1}-J^{\xi^*,\pi_e}_hq^{\xi^*}_{h-1}\|_1^2\bigr]\)
(this is exactly the stepwise error \(\mathcal E_e^2\) of
Definition~\ref{def:errors}), so that \(\mathrm{KL}_e\ge c_{\rm KL}s_e\) by the
first clause of Assumption~\ref{ass:reg}(e), and
\(\Delta_e^2\le Hs_e\) by Cauchy--Schwarz over \(h\).

\begin{lemma}[Per-epoch error cap]\label{lem:epoch-cap}
With probability at least \(1-\delta_e\) over epoch \(e\)'s data, the number of
data-collection episodes \(n_e\) of epoch \(e\), when \(n_e\ge1\), satisfies
\[
n_e\,\mathrm{KL}_e \;\le\; C\bigl(\beta_e(n_e) + B_0\log\bigl(2n_e(n_e{+}1)/\delta_e\bigr)\bigr) \;=:\; \tilde\beta_e
\;=\; O(B_0\,\beta_T),
\]
for an absolute constant \(C\), however the epoch ends (by the test, the
cap, or the horizon). The last estimate uses \(n_e\le T\), \(\beta_e(n_e)\le\beta_T\),
and \(\log T\le\beta_T/c\) (both by the definition of \(\beta_T\)), so that the
additive logarithm is absorbed at the cost of a factor depending on \(B_0\). Consequently
\(n_e s_e\le\tilde\beta_e/c_{\rm KL}\) and
\(n_e\Delta_e^2\le H\tilde\beta_e/c_{\rm KL}\).
\end{lemma}

\begin{proof}
The pair \((\pi_e,\xi_e)\) is fixed at the start of the epoch (measurable with
respect to the pre-epoch \(\sigma\)-field), and after warm-up the
data-collection episodes are i.i.d.\ with law \(P^{\pi_e}_{\xi^*}\)
(Assumption~\ref{ass:reg}(f)). Let
\(Z_s := \log\bigl(P_{\xi_e}/P_{\xi^*}\bigr)(\tau_s)+\mathrm{KL}_e\) over the
epoch's data-collection episodes \(s\); as in Lemma~\ref{lem:conc}, \((Z_s)\)
is a martingale difference sequence with \(|Z_s|\le 2B_0\) and conditional
variance at most \(O(B_0)\,\mathrm{KL}_e\) (the variance bound
\eqref{eq:conc-variance}, using \(B_0\ge1\)). Freedman's inequality with a union bound over the episodes \(s\) of the epoch
(budget \(\delta_e/(2s(s{+}1))\) at episode \(s\), summing to \(\delta_e/2\))
gives, with probability \(\ge 1-\delta_e/2\), simultaneously for all \(s\),
\begin{equation}
\label{eq:cap-freedman}
G_s := \sum_{s'\le s}\log\frac{P_{\xi_e}(\tau_{s'})}{P_{\xi^*}(\tau_{s'})}
\;\le\; -\tfrac12\,s\,\mathrm{KL}_e + C_0B_0\log\bigl(2s(s{+}1)/\delta_e\bigr),
\end{equation}
where \(C_0\) is the absolute constant used in the termination threshold
\(\tau_e(t)\) (taken large enough to cover both tails).
Let \(g_{\rm pre}:=L_{\rm pre}(\xi_e)-L_{\rm pre}(\xi^*)\) denote the pre-epoch
cumulative gap. At selection, \(\xi_e\in\mathcal C_{e-1}\) and
\(L_{\rm pre}(\hat\xi)\ge L_{\rm pre}(\xi^*)\) give \(g_{\rm pre}\ge-\beta_e\);
in the other direction, \(L_{\rm pre}(\xi_e)\le L_{\rm pre}(\hat\xi)\) and
Lemma~\ref{lem:conc}(i) at the epoch boundary give
\(g_{\rm pre}\le\beta_e/2\le\beta_e(n)/2\) for every \(n\ge1\). Consider the
last episode \(n\) at which the termination test had not yet fired (\(n=n_e\)
if the test never fired, whether the epoch exhausted its cap or was truncated
by the horizon; \(n=n_e-1\ge n_e/2\) if the test fired with
\(n_e\ge2\); for \(n_e=1\) the claim is
trivial since \(\mathrm{KL}_e\le2B_0\le\beta_e\), using \(c\ge16\)). Not having fired means
\(L_n(\xi_e)\ge L_n(\hat\xi_n)-\tau_e(n)\ge L_n(\xi^*)-\tau_e(n)\), i.e.\
\(g_{\rm pre}+G_n\ge-\tau_e(n)\), whence
\(G_n\ge-\tau_e(n)-\beta_e(n)/2\). Combining with \eqref{eq:cap-freedman}:
\(\tfrac12 n\,\mathrm{KL}_e\le\tau_e(n)+\tfrac12\beta_e(n)+C_0B_0\log(2n(n{+}1)/\delta_e)
\le \tfrac52\beta_e(n)+2C_0B_0\log(2n(n{+}1)/\delta_e)\), and
\(n_e\mathrm{KL}_e\le 2n\,\mathrm{KL}_e+2B_0\) yields the claim after adjusting
the absolute constant. The consequences follow from
Assumption~\ref{ass:reg}(e) and \(\Delta_e^2\le Hs_e\).
\end{proof}

\begin{lemma}[Switch count]\label{lem:switch-count}
Let \(S_T\) denote the number of epochs started by time \(T\), and let
\(N_{\rm term}\) be the number ended by the termination test. On the same
event, augmented by the lower-tail analogue of \eqref{eq:cap-freedman} (same
per-episode union budget, remaining \(\delta_e/2\)),
\[
S_T \;\le\; \log_2 (2T) + N_{\rm term} + 1,
\]
and
\[
N_{\rm term}
\;\le\;
\frac{32}{3c\,(1+mB_0)}\sum_{e} n_e\,\mathrm{KL}_e
\;\le\;
\frac{32B_0}{3c\,(1+mB_0)}\sqrt{T\,\textstyle\sum_{e}n_e\Delta_e^2}
\;=\;
\tilde O\!\left(\frac{B_0}{1+mB_0}\sqrt{d_E\,\bar\beta_T\,T}\right).
\]
Consequently the total warm-up cost satisfies, with no condition relating
\(m\), \(B_0\), and \(T\),
\((m-1)H\,N_{\rm term}\le\tilde O\bigl(H\sqrt{d_E\bar\beta_T T}\bigr)\) and
\((m-1)H\,S_T\le \tilde O\bigl(mH\log T + H\sqrt{d_E\bar\beta_T T}\bigr)\).
\end{lemma}

\begin{proof}
Scheduled completions number at most \(\log_2(2T)\) (Appendix~\ref{app:alg-structure}), and at most
one epoch is truncated by the horizon. Consider an epoch \(e\) ended by the
termination test at episode \(n=n_e\). At the firing time,
\(L_n(\xi_e)<L_n(\hat\xi_n)-\tau_e(n)\le
L_n(\xi^*)+\beta_e/2+\beta_e(n)/8-\tau_e(n)\le
L_n(\xi^*)+\tfrac58\beta_e(n)-\tau_e(n)\) on
the strengthened confidence event, using \(\beta_e\le\beta_e(n)\). Subtracting
the pre-epoch gap, which
satisfies \(g_{\rm pre}\ge-\beta_e\ge-\beta_e(n)\), the epoch-local gap obeys
\(G_{n}<-\tau_e(n)+\tfrac{13}{8}\beta_e(n)\). The lower-tail Freedman bound for the
same martingale (same union budget) gives
\(G_n\ge-2n\,\mathrm{KL}_e-C_0B_0\log(2n(n{+}1)/\delta_e)\) for all \(n\)
simultaneously with probability \(\ge1-\delta_e/2\); therefore
\[
2n_e\mathrm{KL}_e
\;\ge\;
\tau_e(n)-\tfrac{13}{8}\beta_e(n)-C_0B_0\log(2n(n{+}1)/\delta_e)
\;=\;
\tfrac38\beta_e(n) ,
\]
by the definition of \(\tau_e\), so \(n_e\mathrm{KL}_e\ge\tfrac{3}{16}\beta_e(n)\). Since
\(\log(1/\delta_e)\ge\log 2\ge\tfrac12\) and \(\beta_e(n)\) carries the
\(mB_0\) floor, \(\beta_e(n)\ge(c/2)(1+mB_0)\), so each terminated epoch contributes
\(n_e\mathrm{KL}_e\ge\tfrac{3c}{32}(1+mB_0)\); summing over terminated
epochs gives the first inequality. For the second, bounded log-ratios give
\(\mathrm{KL}_e\le 2B_0\,\mathrm{TV}(P^{\pi_e}_{\xi^*},P^{\pi_e}_{\xi_e})
\le B_0\,\Delta_e\)
(total variation of the trajectory laws is at most half the sum of one-step
prediction errors, by the telescoping in the proof of
Lemma~\ref{lem:simulation}), so by Cauchy--Schwarz
\(\sum_en_e\mathrm{KL}_e\le B_0\sqrt{\sum_en_e}\sqrt{\sum_en_e\Delta_e^2}
\le B_0\sqrt{T\sum_en_e\Delta_e^2}\), using \(\sum_en_e\le T\). The final estimate substitutes
\(\sum_en_e\Delta_e^2=\tilde O(d_E\bar\beta_T)\) from
Lemma~\ref{lem:eluder-total}, whose proof does not use this lemma. The warm-up
consequences follow by multiplying by \((m-1)H\) and using
\((m-1)B_0/(1+mB_0)\le 1\).
\end{proof}

\subsection{The Role of the Posterior-Lipschitz Transport Term}
\label{sec:gamma}

\begin{remark}[Role of the transport term]\label{rem:transport-role}
Because the adversary aggregation operator \(G^{\Phi,\pi}_h\) depends on the
deployed policy, one might expect the analysis to pay a cost for transporting
adversary-aggregation errors from historical policies to the current one.
Under Assumption~\ref{ass:reg}(e), which imposes the KL-to-operator
conversion on the \emph{joint} operators, no such cost arises. Past consistency
is established directly at the data-collection policies, and the query shift to
\(\pi_e\) is handled by the capped summability argument, so the radius
\(\bar\beta_T=H\beta_T\) carries no transport term. The transport viewpoint
remains instructive for two reasons. First, it quantifies what changes if
Assumption~\ref{ass:reg}(e) is weakened to a world-channel-only conversion, in
which case adversary-aggregation errors must be compared across deployed
policies and each comparison costs \(O(H^2)\) at the aggregate squared-error
scale (Lemma~\ref{lem:transport} below). Second, it is the lens through which
the flat-batching analysis of \citet{sang2025pomg} is best understood
(Appendix~\ref{app:prior_work}). With \(K=\Theta(\sqrt{d_ET})\) deployed
policies, the accumulated transport term grows polynomially in \(T\) and
obstructs a direct \(\tilde O(\sqrt T)\) guarantee by that route.
\end{remark}

\begin{lemma}[Posterior-Lipschitz transport]\label{lem:transport}
Suppose Assumption~\ref{ass:pl} holds. Fix two learner policies \(\pi,\nu\), an
adversary model \(\Phi\), and a step \(h\).

Let \(q_h^{\rm mid}\in\mathcal V\) denote any
nonnegative intermediate predictive vector of the \(B\)-indexed form
produced by the world-channel update \(W_h^\theta\), before the adversary action is
aggregated by \(G_h^{\Phi,\pi}\). 
Then
\[
\left\|
\bigl(G_h^{\Phi,\pi}-G_h^{\Phi^*,\pi}\bigr)q_h^{\rm mid}
-
\bigl(G_h^{\Phi,\nu}-G_h^{\Phi^*,\nu}\bigr)q_h^{\rm mid}
\right\|_1
\le
C_{\rm tr}\,\|q_h^{\rm mid}\|_1 ,
\]
where \(C_{\rm tr}\) depends on the posterior-Lipschitz constant \(L\) but not on
\(T\). Consequently, at the aggregate squared-error scale,
\[
\left(
\sum_{h=1}^H
\left\|
\bigl(G_h^{\Phi,\pi}-G_h^{\Phi^*,\pi}\bigr)q_h^{\rm mid}
-
\bigl(G_h^{\Phi,\nu}-G_h^{\Phi^*,\nu}\bigr)q_h^{\rm mid}
\right\|_1
\right)^2
\le
C_{\rm tr}^2\, H
\sum_{h=1}^H \|q_h^{\rm mid}\|_1^2 .
\]
\end{lemma}

\begin{proof}
The aggregation operator \(G_h^{\Phi,\pi}\) differs across policies only through
the adversary response probabilities \(g_h^\Phi(\cdot\mid\eta^B,\pi)\). For any
intermediate vector \(q_h^{\rm mid}\), linearity of \(G_h^{\Phi,\pi}\) and the
definition of the \(\ell_1\) operator norm give
\[
\left\|
\bigl(G_h^{\Phi,\pi}-G_h^{\Phi,\nu}\bigr)q_h^{\rm mid}
\right\|_1
\le
\sup_{\eta^B}
\left\|
g_h^\Phi(\cdot\mid\eta^B,\pi)
-
g_h^\Phi(\cdot\mid\eta^B,\nu)
\right\|_1
\cdot
\|q_h^{\rm mid}\|_1 .
\]
By Assumption~\ref{ass:pl}, uniformly over \(\theta\in\Theta\),
\[
\left\|
g_h^\Phi(\cdot\mid\eta^B,\pi)
-
g_h^\Phi(\cdot\mid\eta^B,\nu)
\right\|_1
\le
L
\left\|
S_{\eta^B,h}^{\rm ref,\theta}(\pi)
-
S_{\eta^B,h}^{\rm ref,\theta}(\nu)
\right\|_1 .
\]
The two posterior-predictive learner action distributions are probability
distributions, so their \(\ell_1\) distance is at most \(2\). Hence
\[
\left\|
\bigl(G_h^{\Phi,\pi}-G_h^{\Phi,\nu}\bigr)q_h^{\rm mid}
\right\|_1
\le
2L\|q_h^{\rm mid}\|_1 .
\]
The same bound holds with \(\Phi^*\) in place of \(\Phi\). Applying the triangle
inequality,
\begin{align*}
&\left\|
\bigl(G_h^{\Phi,\pi}-G_h^{\Phi^*,\pi}\bigr)q_h^{\rm mid}
-
\bigl(G_h^{\Phi,\nu}-G_h^{\Phi^*,\nu}\bigr)q_h^{\rm mid}
\right\|_1 \\
&\qquad\le
\left\|
\bigl(G_h^{\Phi,\pi}-G_h^{\Phi,\nu}\bigr)q_h^{\rm mid}
\right\|_1
+
\left\|
\bigl(G_h^{\Phi^*,\pi}-G_h^{\Phi^*,\nu}\bigr)q_h^{\rm mid}
\right\|_1
\le
4L\|q_h^{\rm mid}\|_1 .
\end{align*}
Absorbing constants into \(C_{\rm tr}\) proves the first claim. The aggregate
squared bound follows from Cauchy--Schwarz:
\[
\left(\sum_{h=1}^H a_h\right)^2
\le
H\sum_{h=1}^H a_h^2,
\]
with
\[
a_h
=
\left\|
\bigl(G_h^{\Phi,\pi}-G_h^{\Phi^*,\pi}\bigr)q_h^{\rm mid}
-
\bigl(G_h^{\Phi,\nu}-G_h^{\Phi^*,\nu}\bigr)q_h^{\rm mid}
\right\|_1 .
\]
\end{proof}

Applied with \(\pi=\pi_e\) and \(\nu=\pi_k\), and using
\(\|q^{\rm mid}_h\|_1\le|B|\) along trajectories of the true model (the
\(B\)-indexed components carry the unit predictive mass once per hypothetical
adversary action), each historical policy contributes at most
\begin{equation}
\label{eq:appendix-gamma-e}
C_{\rm tr}^2H\sum_{h=1}^H\mathbb E\|q^{\rm mid}_h\|_1^2 \;=\; O(H^2)
\end{equation}
at the aggregate squared-error scale (the constant absorbs
\(C_{\rm tr}^2|B|^2\)). After \(S\) deployed policies the accumulated transport
term is therefore
\begin{equation}
\label{eq:appendix-gamma-total}
\Gamma_{S}=O(SH^2),
\end{equation}
which is \(O(H^2\log T)\) for \(S=O(\log T)\) scheduled epochs, but
\(\tilde O(H^2\sqrt{d_ET})\) under a flat \(\sqrt{d_ET}\)-batch schedule;
this is the quantitative content of the comparison in
Appendix~\ref{app:degraded_bound}. We emphasize that none of this enters the
proof of Theorem~\ref{thm:upper}.

\subsection{Eluder Dimension Bound}

We first state the weighted Eluder summability result, then derive the main bound.
The key issue is that $\Delta_e$ is an aggregate over $H$ steps and a trajectory
distribution, not a single scalar function value. We handle this by working with the
\emph{stepwise squared error} $\mathcal{E}_e^2$ as the primary object for the Eluder
argument, then converting back to $\Delta_e$ via Jensen's inequality.

\begin{definition}[Stepwise and channel prediction errors]\label{def:errors}
For epoch $e$, define the \emph{stepwise squared prediction error}
\[
\mathcal{E}_e^2 :=
\sum_{h=1}^H \mathbb{E}_{\tau \sim P^{\pi_e}_{\xi^*}}\!\left[
\|J^{\xi_e,\pi_e}_h q^{\xi^*}_{h-1} - J^{\xi^*,\pi_e}_h q^{\xi^*}_{h-1}\|_1^2
\right],
\]
and recall the \emph{aggregate prediction error} $\Delta_e$
of \eqref{eq:appendix-delta-e}.
By Jensen's inequality (or Cauchy--Schwarz over $h$),
$\Delta_e^2 \leq H\,\mathcal{E}_e^2$. Define also the \emph{channel errors}
\[
\Delta^W_e := \sum_{h=1}^H \mathbb{E}_{\tau \sim P^{\pi_e}_{\xi^*}}\!\left[
\|(W^{\theta_e}_h - W^{\theta^*}_h) q^{\xi^*}_{h-1}\|_{1,\mathcal V}
\right],
\qquad
\Delta^G_e := \sum_{h=1}^H \mathbb{E}_{\tau \sim P^{\pi_e}_{\xi^*}}\!\left[
\|(G^{\Phi_e,\pi_e}_h - G^{\Phi^*,\pi_e}_h) W^{\theta^*}_h q^{\xi^*}_{h-1}\|_1
\right],
\]
which are the values $F^{W}_{\theta_e}(\pi_e)$ and $F^{G}_{\Phi_e}(\pi_e)$ of
the channel classes of Section~\ref{sec:eluder}; by
the channel-triangle inequalities \eqref{eq:channel-triangle}, $\Delta_e\le\Delta^W_e+\Delta^G_e$ and
$\Delta^G_e\le\Delta_e+\Delta^W_e$.
\end{definition}

\begin{lemma}[Capped Weighted Eluder Summability]\label{lem:eluder-summ}
Let $\mathcal{F}$ be a class of functions taking values in $[0,b]$, and suppose
epochs $e=1,\ldots,E$ produce triples $(f_e\in\mathcal F,\,x_e,\,n_e\in\mathbb N)$
with $\sum_e n_e\le 2T$ such that
\emph{(i) past consistency:} $\sum_{k<e} n_k\, f_e(x_k)^2 \leq R$ for every $e$; and
\emph{(ii) per-epoch cap:} $n_e\, f_e(x_e)^2 \leq \tilde\beta$ for every $e$.
Then, with $\alpha := \sqrt{\tilde\beta/(2T)}$,
\[
\sum_{e=1}^{E} n_e\, f_e(x_e)^2
\;\leq\;
40\,\dimE(\mathcal{F},\alpha/2)\,(R+\tilde\beta)\,
\Bigl\lceil\log_2\tfrac{2b}{\alpha}\Bigr\rceil\,\bigl\lceil\log_2 4T\bigr\rceil
\;+\;2\tilde\beta .
\]
Without hypothesis (ii), no bound of this form is possible, since a function
that vanishes on all past queries but is large at its own query satisfies (i)
with any $R\ge0$, yet contributes up to $n_eb^2$
(cf.\ Proposition~\ref{prop:fixed-fails}).
\end{lemma}
\begin{proof}
Epochs with $f_e(x_e)^2\le\tilde\beta/(2T)$ contribute at most
$\sum_e n_e\,\tilde\beta/(2T)\le\tilde\beta$ in total; set them aside. Bucket
the remaining epochs by the dyadic scales of value and weight. Bucket
$(j,\ell)$ contains the epochs with
$f_e(x_e)\in(\varepsilon_{j+1},\varepsilon_j]$, where $\varepsilon_j:=b2^{-j}$
and $j$ ranges over at most $\lceil\log_2(b/\alpha)\rceil$ values before
reaching the floor scale $\alpha$, and $n_e\in[2^\ell,2^{\ell+1})$ with
$\ell\le\log_2 2T$. Fix a bucket, list its epochs in time order, and write $N$ for the bucket
size.

\emph{Claim: $N\le d\,\bigl(R/(2^\ell\varepsilon_{j+1}^2)+2\bigr)$, where
$d:=\dimE(\mathcal F,\varepsilon_{j+1})$ and $\varepsilon_{j+1}\ge\alpha/2$ for
every bucket in range.} Process
the bucket's queries in order, maintaining a collection of disjoint
subsequences, as follows. Append the arriving query to some maintained
subsequence on whose current content it is
$\varepsilon_{j+1}$-\emph{in}dependent if one exists, and otherwise start a new
subsequence. Every query is placed, and by the definition of the Eluder
dimension no subsequence receives more than $d$ elements, so at the end there
are $M\ge\lceil N/d\rceil$ subsequences. Consider the query $x_{e}$ that opened
the $M$-th subsequence. At its arrival it was
$\varepsilon_{j+1}$-\emph{dependent} on each of the $M-1$ previously opened,
disjoint, nonempty subsequences. Let $f:=f_{e}$ be that epoch's function. Since
$f(x_{e})>\varepsilon_{j+1}$, dependence on each such set $S$ forces
$\sum_{z\in S}f(z)^2>\varepsilon_{j+1}^2$, and since every bucket predecessor
carries weight at least $2^\ell$,
$\sum_{k\in S}n_kf(x_k)^2>2^\ell\varepsilon_{j+1}^2$. Summing over the $M-1$
disjoint sets and invoking hypothesis (i) gives
$R\ge(M-1)\,2^\ell\varepsilon_{j+1}^2\ge(N/d-1)\,2^\ell\varepsilon_{j+1}^2$,
which rearranges to the claim.

\emph{Bucket damage.} Each bucket epoch has
$n_ef_e(x_e)^2\le\min\bigl(2^{\ell+1}\varepsilon_j^2,\ \tilde\beta\bigr)
=\min\bigl(8\cdot2^\ell\varepsilon_{j+1}^2,\ \tilde\beta\bigr)$, so the
bucket's total damage is at most
\[
N\cdot\min\bigl(8\cdot2^\ell\varepsilon_{j+1}^2,\tilde\beta\bigr)
\;\le\;
d\Bigl(\frac{R}{2^\ell\varepsilon_{j+1}^2}+2\Bigr)\cdot 8\cdot2^\ell\varepsilon_{j+1}^2
\;\le\;
d\,\bigl(8R+16\tilde\beta\bigr),
\]
using $2^\ell\varepsilon_{j+1}^2\le n_ef_e(x_e)^2\le\tilde\beta$ for any bucket
member in the last step. The value scales run from $b$ down to at least
$\alpha/2$ and the weight scales cover $1\le n_e\le 2T$, so there are at most
$\lceil\log_2(2b/\alpha)\rceil\cdot\lceil\log_2 4T\rceil$ buckets; summing the
per-bucket damage and adding the floor contribution proves the lemma
(constants consolidated).
\end{proof}

\begin{lemma}[Eluder Bound Across All Epochs]\label{lem:eluder-total}
Suppose Assumptions~\ref{ass:pl}, \ref{ass:reveal}, and~\ref{ass:reg} hold, and
condition on the high-probability events of Lemma~\ref{lem:confidence}
(strengthened as in Appendix~\ref{app:cap-switch}) and
Lemma~\ref{lem:epoch-cap} for all epochs. With the channel classes of
Section~\ref{sec:formulation}, $d_E = d^{W}_{\rm agg}+d^{G}_{\rm agg}$
evaluated at scale $1/(2T)$:
\begin{equation}\label{eq:eluder-agg}
\sum_{e} n_e\,\Delta_e^2
\;\leq\;
\tilde{O}\!\left(d_E\,H\beta_T/c_{\rm KL}\right)
\;=\;
\tilde{O}\!\left(d_E\,\bar\beta_T\right).
\end{equation}
Here and below, $\tilde O(\cdot)$ additionally absorbs constants depending on
$B_0$ and $c_{\rm KL}$, as reflected in the constant-dependence statement of
Theorem~\ref{thm:upper}.
\end{lemma}

\begin{proof}
We establish past consistency and per-epoch caps for each channel, apply the
capped summability lemma channel by channel, and combine.

\paragraph{Past consistency, per channel.}
Since \(\xi_e\in\mathcal C_{e-1}\), Lemma~\ref{lem:conc} gives
\[
\sum_{k<e}n_k\,
\mathrm{KL}\!\left(P^{\pi_k}_{\xi^*}\,\|\,P^{\pi_k}_{\xi_e}\right)
\le
C\beta_e
\le
C\beta_T .
\]
(The concentration argument is unaffected by the data-dependent epoch lengths,
because the martingale runs over episodes in temporal order, with each
episode's policy adapted to the filtration.) By the two clauses of
Assumption~\ref{ass:reg}(e) and Cauchy--Schwarz over $h$ (as in
Definition~\ref{def:errors}), writing $F_{\xi_e}$, $F^{W}_{\theta_e}$,
$F^{G}_{\Phi_e}$ for the composed and channel aggregate errors of the
optimistic model at past policies,
\[
\sum_{k<e} n_k\,F_{\xi_e}(\pi_k)^2 \;\leq\; \frac{CH\beta_T}{c_{\rm KL}} \;=:\; R ,
\qquad
\sum_{k<e} n_k\,F^{W}_{\theta_e}(\pi_k)^2 \;\leq\; R ,
\]
and, by the channel-triangle transfer \eqref{eq:channel-triangle} together with
$(a+b)^2\le 2a^2+2b^2$,
\[
\sum_{k<e} n_k\,F^{G}_{\Phi_e}(\pi_k)^2
\;\le\;
2\sum_{k<e} n_k F_{\xi_e}(\pi_k)^2 + 2\sum_{k<e} n_k F^{W}_{\theta_e}(\pi_k)^2
\;\le\; 4R .
\]

\paragraph{Per-epoch cap, per channel.}
Lemma~\ref{lem:epoch-cap} covers every data-collection episode of every epoch,
terminated or complete; via the two clauses of Assumption~\ref{ass:reg}(e) and
Cauchy--Schwarz over $h$,
\[
n_e\Delta_e^2
\le
H\tilde\beta_e/c_{\rm KL}
\;\le\;
\tilde O(H\beta_T/c_{\rm KL})
\;=:\;
\tilde\beta ,
\qquad
n_e(\Delta^W_e)^2 \le \tilde\beta,
\qquad
n_e(\Delta^G_e)^2 \le 4\tilde\beta,
\]
the last by the triangle transfer, and the middle estimates using
\(\tilde\beta_e=O(B_0\beta_T)\) with the \(B_0\)-dependent factor absorbed.

\paragraph{Summability, per channel.}
The functions $F^{W}_{\theta_e}$ belong to $\mathcal F^{W,\xi^*}_{\rm agg}$ and
the $F^{G}_{\Phi_e}$ to $\mathcal F^{G,\xi^*}_{\rm agg}$. The world-channel
class takes values in $[0,2|B|H]$, since each $b$-component of $W^\theta_h q$ carries
mass at most one, so $\|W^\theta_h q\|_{1,\mathcal V}\le|B|$ and each
$h$-summand is at most $2|B|$. The adversary-channel class takes values in
$[0,2H]$, since $G^{\Phi,\pi}_h\circ W^{\theta^*}_h = J^{(\theta^*,\Phi),\pi}_h$ is
an observable operator of the model $(\theta^*,\Phi)\in\Xi$, whose image of a
normalized state has mass at most one. The ranges enter
Lemma~\ref{lem:eluder-summ} only through $\lceil\log_2(2b/\alpha)\rceil$, so
the extra $\log|B|$ is absorbed into $\tilde O$. Apply
Lemma~\ref{lem:eluder-summ} to the world channel with $(R,\tilde\beta)$ and to
the adversary channel with $(4R,4\tilde\beta)$; the lemma's floor scale for the
adversary application is $2\alpha$ with $\alpha=\sqrt{\tilde\beta/(2T)}$, and
in both applications the dimension is evaluated at a scale at least
$\alpha/2\ge1/(2T)$. Noting that
$\tilde\beta\ge cH/c_{\rm KL}\ge 2/T$ (taking $c_{\rm KL}\le1$ without loss of
generality, since Assumption~\ref{ass:reg}(e) only weakens as $c_{\rm KL}$
decreases) gives $\alpha/2\ge1/(2T)$, so that by scale monotonicity
$\dimE(\mathcal F^{W}_{\rm agg},\alpha/2)\le d^{W}_{\rm agg}$ and
$\dimE(\mathcal F^{G}_{\rm agg},\alpha/2)\le d^{G}_{\rm agg}$, we obtain
\[
\sum_e n_e\,(\Delta^W_e)^2
\le
\tilde O\bigl(d^{W}_{\rm agg}\,(R+\tilde\beta)\bigr),
\qquad
\sum_e n_e\,(\Delta^G_e)^2
\le
\tilde O\bigl(d^{G}_{\rm agg}\,(R+\tilde\beta)\bigr).
\]

\paragraph{Combining.}
By $\Delta_e\le\Delta^W_e+\Delta^G_e$ and $(a+b)^2\le2a^2+2b^2$,
\[
\sum_e n_e\,\Delta_e^2
\le
2\sum_e n_e(\Delta^W_e)^2+2\sum_e n_e(\Delta^G_e)^2
\le
\tilde O\bigl((d^{W}_{\rm agg}+d^{G}_{\rm agg})\,H\beta_T/c_{\rm KL}\bigr),
\]
which is \eqref{eq:eluder-agg}.
\end{proof}

\subsection{Applying Cauchy--Schwarz and Summing Over Epochs}

Using the channel-sum Eluder dimension $d_E = d^{W}_{\rm agg}+d^{G}_{\rm agg}$,
Lemma~\ref{lem:eluder-total} gives (absorbing constant and problem-dependent
factors into $C_1$, and a $\log H$ into the $\log^2 T$ factor, using
$H\le\mathrm{poly}(T)$):
\[
\sum_{e} n_e\,\Delta_e^2 \;\leq\; C_1\,d_E\,\bar\beta_T\,\log^2 T.
\]
Applying Cauchy--Schwarz with $\sum_{e} n_e \le T$:
\begin{align*}
\sum_{e} n_e \cdot \Delta_e &\leq \sqrt{\left(\sum_{e} n_e\right) \cdot \left(\sum_{e} n_e \cdot \Delta_e^2\right)}
\;\leq\; O\!\left(\sqrt{C_1\,d_E\,T\,\bar\beta_T}\,\log T\right).
\end{align*}
The warm-up cost across all epochs is $(m-1)H\,S_T$, with the number of
deployed policies $S_T$ bounded in Lemma~\ref{lem:switch-count}. Combining with
the factor $H$ from $R_e \le n_e H\Delta_e + (m-1)H$:
\begin{align*}
PR(T) &\le H\sum_{e} n_e\,\Delta_e + (m-1)H\,S_T \\
&\leq C\!\left(H\sqrt{d_E\,\bar\beta_T\,T}\,\log T + mH\,S_T\right),
\end{align*}
which is the bound of Theorem~\ref{thm:upper}.
Lemma~\ref{lem:switch-count} gives, unconditionally,
\[
mH\,S_T
\;\le\;
O(mH\log T) + \tilde O\!\left(\frac{mB_0}{1+mB_0}\,H\sqrt{d_E\bar\beta_T T}\right)
\;\le\;
O(mH\log T) + \tilde O\!\bigl(H\sqrt{\bar\beta_T\,d_E\,T}\bigr),
\]
so the warm-up term is dominated by the leading term plus $mH\log T$, with no
condition relating $m$, $B_0$, and $T$; this is precisely what the $mB_0$
floor in $\beta_e$ ensures. The total failure probability is at most $2\delta$: $\delta$ for the
confidence events (Lemma~\ref{lem:confidence}, strengthened as in
Appendix~\ref{app:cap-switch}) and $\delta$ for the per-epoch events of
Lemmas~\ref{lem:epoch-cap} and~\ref{lem:switch-count}
($\sum_e\delta_e\le\delta$).

\begin{remark}[Sharper bound under aggregate KL conversion]\label{rem:aggregate-kl}
If Assumption~\ref{ass:reg}(e) is strengthened to the aggregate form
\[
\mathrm{KL}(P^\pi_{\xi^*}\|P^\pi_\xi) \;\geq\; c_{\rm agg}\left(\sum_{h=1}^H\mathbb{E}[\|J^{\xi,\pi}_h q^{\xi^*}_{h-1} - J^{\xi^*,\pi}_h q^{\xi^*}_{h-1}\|_1]\right)^{\!2},
\]
then the past-consistency step yields $\sum_{k<e}n_k F_{\xi_e}(\pi_k)^2 \leq C\beta_T/c_{\rm agg}$ directly (no Cauchy--Schwarz factor $H$), and the per-epoch cap of Lemma~\ref{lem:epoch-cap} similarly loses its $H$ factor; strengthening the world-channel clause in the same way transfers the improvement to both channels. The final theorem then holds with the sharper $\bar\beta_T = \beta_T$ in place of $H\beta_T$. This aggregate KL condition holds for tabular/finite classes with $c_{\rm agg}$ absorbing an $H$-dependent constant from the weak-revealing argument.
\end{remark}

\begin{remark}[Stepwise alternative]\label{app:stepwise}
One can instead run the summability argument on per-step channel classes,
whose queries are $(\pi,h,\tau_{h-1},o^A_h,a_h)$, converting back through
$\Delta_e^2\le H\,\mathcal E_e^2$ (Definition~\ref{def:errors}); this costs an
extra $\sqrt H$ in the final regret bound. We do not pursue it. For the
parametric classes of Proposition~\ref{prop:eluder}, the per-step and aggregate
channel dimensions admit the same span-dimension upper bounds
(Lemma~\ref{lem:sign-witness} and the proof of
Proposition~\ref{prop:eluder} in Appendix~\ref{app:eluder-lemma}), so the
aggregate route is never worse.
\end{remark}

\begin{remark}[Time accounting]\label{rem:time} Algorithm~\ref{alg:omle} executes $m-1$ warm-up episodes per epoch in addition to the data-collection episodes. If \(T\) denotes the number of data-collection episodes actually counted in the regret, then the total number of calendar episodes is \(T+O(m\,S_T)\). The regret bound $PR(T)$ is stated in terms of the $T$ data-collection episodes; warm-up episodes contribute the additional $O(mH\,S_T)$ term already included above (which the $mB_0$ floor reduces to $O(mH\log T)$ plus a dominated term). If one wishes to count $T$ as the total calendar budget, one should reduce each epoch's data-collection length by the warm-up length accordingly; this changes constants but not the asymptotic rate.

\end{remark}

\section{Proof of Theorem~\ref{thm:lower} (Lower Bound)}
\label{app:lower}

We prove that any algorithm must incur $\Omega(\sqrt{dT})$ policy regret by embedding the standard $d$-armed stochastic bandit lower bound into a two-step POMG ($H=2$) that fits within the POMG protocol.

\begin{proof}[Proof of Theorem~\ref{thm:lower}]
\leavevmode 
\paragraph{Construction.}
Fix $d \leq T$ and $\Delta = c_0\sqrt{d/T}$. The POMG has:
\begin{itemize}
\item Latent states $S = \{s_1,\ldots,s_d\}$; initial state $s_1$ (the identity of the initial state is irrelevant since step 1 is only a routing step).
\item Learner actions $A = \{a_1,\ldots,a_d\}$; single adversary action $B = \{b_{\text{coop}}\}$ (adversary is trivially memoryless, $m=1$, $L=0$).
\item Horizon $H=2$.
\item \textbf{Step $h=1$:} All states emit the same observation $o_0$. The learner
chooses $a_i$; the state transitions deterministically to $s_i$. Reward $r_1=0$.

\item \textbf{Step $h=2$:} In state $s_i$, the learner observes $o = (i, y)$, where $y \sim \mathrm{Bernoulli}(\mu_i)$. The reward is $r_2(i,y) = y$.
\end{itemize}

Although all states share the same step-1 observation, the initial state is fixed
and known to be \(s_1\), so there is no hidden-state ambiguity at step \(h=1\). 
The first coordinate \(i\) of the step-2 observation \(o=(i,y)\) reveals the latent
state at step \(h=2\); together, this makes the construction fully revealing with
\(\kappa=1\) and a constant signal parameter independent of \(\Delta\).
The reward is the second coordinate \(y\), so the learner both observes the routed
state and receives the corresponding Bernoulli reward sample.

\paragraph{Policy regret equals bandit regret.}
With $H=2$, $m=1$, and a fixed adversary, the expected  policy regret of any learner algorithm against the best fixed policy $\pi^* = a_{i^*}$ satisfies:
\begin{equation}
\label{eq:lower-bandit-regret}
\E[PR(T)]
=
\sum_{t=1}^T \!\left(\tfrac{1}{2}+\Delta - \E[r_2^{(t)}]\right)
=
\Delta \cdot \E\!\left[\sum_{t=1}^T \mathbf{1}\{A_t \neq i^*\}\right],
\end{equation}
which is exactly the expected cumulative regret of the learner against arm $i^*$ in a $d$-armed Bernoulli bandit with gap $\Delta$.

\paragraph{Policy class, ambient model class, and Eluder dimension.}
The learner's policy class for this instance is $\Pi=\{a_1,\ldots,a_d\}$, the
deterministic routing policies; the learner may still randomize its choice
across episodes, and the bandit lower bound below applies to arbitrary
algorithms. Define the ambient model class as $\mathcal{M} = \{\mu \in [1/2,\,
1/2+\Delta]^d\}$, parameterizing the step-2 emission means. The adversary
channel is trivial (a single constant response), so $d^{G}_{\rm agg}=0$ and
$d_E = d^{W}_{\rm agg}$. Since $\Pi$ has $d$ elements, any repeated query is
dependent on its predecessors at every scale, giving $d_E \le d$. Conversely,
for $dT \geq C_{\rm lb}$, with $C_{\rm lb}$ an absolute constant determined by $c_0$, the $d$
single-arm queries form a $1/(2T)$-independent sequence. The witness for query
$a_j$ sets $\mu_j = 1/2+\Delta$ and matches the reference elsewhere, so its
aggregate error at $a_j$ is exactly $\Delta$ (the step-1 operators coincide
across the class, and the step-2 emission difference carries $\ell_1$ mass
$|\mu_j-\tfrac12|=\Delta$ at the realized observation), while its error at
every other query is zero; and $\Delta = c_0\sqrt{d/T} > 1/(2T)$ exactly when
$dT > 1/(4c_0^2)$. Hence $d_E = \Theta(d)$ for
$dT\geq C_{\rm lb}:=\lfloor 1/(4c_0^2)\rfloor+1$. The
hard-instance prior uses the subclass where exactly one $\mu_{i^*} = 1/2+\Delta$
and the rest equal $1/2$; this subclass is contained in the ambient class
\(\mathcal M\), and the ambient class is the model class for which
\(d_E=\Theta(d)\).

\paragraph{Minimax lower bound.}
By the standard $d$-armed stochastic bandit minimax lower bound
\citep[\S15]{lattimore2020bandit}, in its Bernoulli form (the change-of-measure
step uses $\mathrm{KL}(\mathrm{Ber}(1/2)\,\|\,\mathrm{Ber}(1/2+\Delta))\le
4\Delta^2$ for $\Delta\le1/4$, and requires $d\ge2$), for $\Delta =
c_0\sqrt{d/T}$ with $c_0 > 0$ sufficiently small, every algorithm satisfies
\begin{equation}
\label{eq:lower-bandit-minimax}
\max_{i^*\in[d]}\,\mathbb{E}_{i^*}\!\left[\Delta\sum_{t=1}^T \mathbf{1}\{A_t \neq i^*\}\right] \geq c\sqrt{dT},
\end{equation}
for a universal constant $c > 0$. Since the policy regret of the constructed POMG equals the bandit regret in \eqref{eq:lower-bandit-regret}, the same lower bound holds for policy regret.

\paragraph{Assumption verification.}
\begin{itemize}
\item \textbf{Posterior-Lipschitz (Assumption~\ref{ass:pl}):} The adversary is constant ($|B|=1$), so $L=0$. 
\item \textbf{Weak revealing (Assumption~\ref{ass:reveal}):} $\kappa=1$; at step $h=1$ the next observation $o=(i,y)$ has first coordinate $i$ equal to the routed state index, so world-channel differences are detected with a constant signal parameter (independent of $\Delta$); at the final step the condition is not imposed, per the window convention of Assumption~\ref{ass:reveal}, and the operative conversion is Assumption~\ref{ass:reg}(e), verified directly in the next item. The $\Delta$-dependence belongs to the reward/emission mean estimation problem, not to state revelation. 
\item \textbf{Statistical regularity (Assumption~\ref{ass:reg}):} Bernoulli likelihoods are bounded, the model class $\mathcal{M}$ is compact and has finite covering number, predictive states are finite-dimensional probability vectors with unit $\ell_1$-norm, and both clauses of the KL-to-operator conversion hold exactly, via $\mathrm{KL}(\mathrm{Ber}(p)\|\mathrm{Ber}(q))\ge2(p-q)^2$ (the adversary channel is trivial). 
\end{itemize}

Since $\max_{i^*}\E[PR(T)] \geq c\sqrt{dT}$ and $d_E = \Theta(d)$, we have \(\sup_{\text{instance}\in\mathcal M}\E[PR(T)] \geq c'\sqrt{d_E T}\) for a universal constant \(c'>0\), completing the proof.
\end{proof}

% % % ============================================================================

\section{Proof of Theorem~\ref{thm:adaptive} (Adaptive Horizon)}
\label{app:adaptive}

The epoch-based structure of Algorithm~\ref{alg:omle} naturally provides adaptation to unknown time horizons. The key insight is that each epoch is self-contained. Epoch $e$ determines its behavior from its cap $T_e = 2^e$, the data collected so far, and the termination test, while building its confidence set from all historical data. The algorithm never needs to know the total horizon $T$ in advance.

\begin{proof}[Proof of Theorem~\ref{thm:adaptive}]
The algorithm operates through epochs \(e=1,2,\ldots\), where epoch \(e\) has
length at most \(T_e=2^e\) (the termination test may end it earlier). The important point is that the algorithm does not
need to know \(T\) in advance. All choices made in epoch \(e\) depend only on
the epoch index \(e\) and on the data collected up to the current episode.

At the start of epoch \(e\), the algorithm performs the following steps.
First, it computes the MLE \(\hat\xi_{e-1}\) from the cumulative dataset
\(\mathcal D_{e-1}\) containing all trajectories from epochs \(1,\ldots,e-1\).
Second, it forms the confidence set
\[
\mathcal C_{e-1}
=
\left\{
\xi:
L_{e-1}(\xi)
\ge
L_{e-1}(\hat\xi_{e-1})-\beta_e
\right\},
\]
where
\[
\delta_e=\frac{\delta}{2e^2},
\qquad
\epsilon_e = 2^{-\lceil\log_2(|\mathcal D_{e-1}|+2)\rceil},
\qquad
\beta_e
=
c\left(\log \mathcal{N}(\epsilon_e;\Xi)+\log(1/\delta_e)+mB_0\right).
\]
Thus \(\beta_e\) depends only on the current epoch index and the data collected
so far, and not on the eventual stopping time. Third, the algorithm selects an optimistic pair
\[
(\pi_e,\xi_e)
\in
\argmax_{(\pi,\xi)\in\Pi\times\mathcal C_{e-1}}
V^\pi(\xi).
\]
It then executes \(\pi_e\) for at most \(T_e=2^e\) data-collection episodes,
subject to the statistical termination test, after the \(m-1\) warm-up
episodes needed for the \(m\)-memory adversary to stabilize, and adds the
resulting trajectories to the cumulative dataset.

The confidence analysis is also horizon-free. By construction,
\(\sum_{e=1}^\infty \delta_e\le\delta\). Therefore, applying the likelihood
concentration and confidence-set validity arguments of Lemma~\ref{lem:confidence} 
with failure probability \(\delta_e\) in epoch \(e\), and union bounding over all
\(e\ge1\), gives a single high-probability event of probability at least \(1-\delta\)
on which all epoch confidence sets are valid simultaneously. This event is defined
without knowing how many epochs will eventually be run.

Now fix any final horizon \(T\). Each epoch \(e\) runs for \(n_e\)
data-collection episodes (ended by its cap, the termination test, or the
horizon), and
\[
\sum_{e}n_e \le T ,
\]
with at most \(\log_2(2T)\) scheduled completions and \(N_{\rm term}\)
terminations bounded as in Lemma~\ref{lem:switch-count}.

Define, as in Appendix~\ref{app:alg-structure},
\begin{align*}
\beta_T&:= c\bigl(\log\mathcal N(\epsilon_T;\Xi)+5\log(2T)+\log(1/\delta)+mB_0\bigr)
= O\bigl(\log\mathcal N(\epsilon_T;\Xi)+\log T+\log(1/\delta)+mB_0\bigr),
\\
\epsilon_T&:=1/(2T+4),
\qquad
\bar\beta_T:=H\beta_T ,
\end{align*}
so that \(\beta_e\le\beta_T\) and \(\beta_e(t)\le\beta_T\) for every epoch and
episode boundary reached by time \(T\).
On the simultaneous confidence event, the proof of Theorem~\ref{thm:upper} applies
verbatim to this prefix of epochs, with the final epoch truncated at the
horizon. 

In particular, if \(\Delta_e\) denotes the aggregate observable-operator prediction
error in epoch \(e\), the simulation lemma gives the epoch regret bound
\[
R_e
\le
n_e H\Delta_e+(m-1)H .
\]
The capped summability argument gives
\[
\sum_{e} n_e \Delta_e^2
\le
\tilde O(d_E\bar\beta_T).
\]
Since \(\sum_{e} n_e\le T\), Cauchy--Schwarz yields
\[
\sum_{e} n_e \Delta_e
\le
\sqrt{\left(\sum_{e} n_e \right)
      \left(\sum_{e} n_e \Delta_e^2\right)}
\le
\tilde O\!\left(\sqrt{d_E T\bar\beta_T}\right).
\]

Combining the data-collection regret with the \((m-1)H\,S_T\) warm-up cost
(Lemma~\ref{lem:switch-count}; all quantities involved are prefix-stable),
we obtain
\begin{equation}\label{eq:adaptive-anytime}
PR(T)
\le
\tilde O\!\left(
H\sqrt{\bar\beta_T d_E T}
+
mH\,S_T
\right).
\end{equation}

This bound holds for every final horizon \(T\) simultaneously on the same
high-probability event. 
The algorithm's adaptation comes from three features:
geometric epochs, epoch-wise confidence parameters defined without reference to the final
horizon, and a union bound over all possible future epochs. The per-epoch cap and switch-count arguments
(Lemmas~\ref{lem:epoch-cap} and~\ref{lem:switch-count}) are likewise
prefix-stable, since they treat each epoch separately; their per-epoch failure
probabilities consume a second budget \(\sum_{e\ge1}\delta_e\le\delta\), also
defined without reference to a final horizon, so all statements hold
simultaneously for every \(T\) with probability at least \(1-2\delta\). Finally, using the warm-up consequence of Lemma~\ref{lem:switch-count} we bound $(m-1)H\,S_T\le\tilde O\bigl(mH\log T+H\sqrt{d_E\bar\beta_T T}\bigr)$; substituting this into \textcolor{red}{\eqref{eq:adaptive-anytime}} gives the form stated in Theorem~\ref{thm:adaptive}.
\end{proof}

\section{Proof of Theorem~\ref{thm:fading} (Fading Memory Adversaries)}
\label{app:fading}

When the adversary has fading memory, meaning its response depends on an exponentially weighted average of past policies rather than a fixed window, we can still apply our epoch-based algorithm with one straightforward modification: extend the warm-up phase at the beginning of each epoch to allow sufficient time for the adversary's internal state to stabilize. The proof shows that (1) we can truncate the adversary's infinite memory to a finite window with negligible error, and (2) the warm-up cost of stabilizing this truncated adversary remains modest because scheduled epochs are logarithmic in number and every statistical termination is priced by evidence against the deployed model (Lemma~\ref{lem:switch-count}).

\begin{proof}[Proof of Theorem~\ref{thm:fading}]
Consider an adversary whose internal state evolves according to the \emph{normalized} recursion
\begin{equation}
\label{eq:fading-recursion}
z_t = \gamma z_{t-1} + (1-\gamma)\phi(\pi^t),
\end{equation}
where $\gamma \in (0,1)$, $\|\phi(\pi)\|_2 \leq 1$ for all $\pi$, and $g^t = g(z_t)$
for an $L_g$-Lipschitz map $g$. Whenever $\|z_0\|_2 \leq 1$, we have $\|z_t\|_2 \leq 1$
for all $t$ (the normalized recursion preserves the unit ball). Unrolling gives:
\[
z_t = (1-\gamma)\sum_{i=0}^{t-1} \gamma^i \phi(\pi^{t-i}) + \gamma^t z_0.
\]

\paragraph{Truncation analysis.} For truncation length $m_{\rm trunc} \geq 1$, define
\[
\tilde{z}_t^{(m_{\rm trunc})} := (1-\gamma)\sum_{i=0}^{m_{\rm trunc}-1} \gamma^i \phi(\pi^{t-i}).
\]
The approximation error is:
\[
\|z_t - \tilde{z}_t^{(m_{\rm trunc})}\|_2 \leq (1-\gamma)\sum_{i=m_{\rm trunc}}^\infty \gamma^i + \gamma^t\|z_0\|_2 \leq \gamma^{m_{\rm trunc}} + \gamma^t.
\]
For $t \geq m_{\rm trunc}$, we have $\gamma^t \leq \gamma^{m_{\rm trunc}}$, hence $\|z_t - \tilde{z}_t^{(m_{\rm trunc})}\|_2 \leq 2\gamma^{m_{\rm trunc}}$. Since the warm-up step runs $m_{\rm warmup}-1$ episodes, every data-collection episode is at least the $m_{\rm warmup}$-th consecutive episode of the deployed policy with $m_{\rm warmup} = m_{\rm trunc}$, so this is the only regime used in the regret analysis.

Choosing
\begin{equation}
\label{eq:fading-window}
m_{\rm trunc}=m_{\rm warmup}
=
\left\lceil \frac{\log(2T)}{1-\gamma}\right\rceil
\end{equation}
ensures \(\gamma^{m_{\rm trunc}}\le 1/(2T)\),
giving $\|z_t - \tilde{z}_t\|_2 \leq 1/T$ for all data-collection steps. Since \(g\) is \(L_g\)-Lipschitz, the per-step response error is
\(O(L_g/T)\) in total variation; over an episode this perturbs the trajectory law
by \(O(HL_g/T)\) and hence the episode value (which lies in \([0,H]\)) by
\(O(H^2L_g/T)\), for a total truncation contribution of \(O(H^2L_g)\) over \(T\)
episodes, which is absorbed into lower-order terms.

\paragraph{Stabilization under constant policy.} The stationary state under constant
policy $\pi$ is $z_\infty = \phi(\pi)$ (the unique fixed point of $z = \gamma z + (1-\gamma)\phi(\pi)$
is $z = \phi(\pi)$, and $\|\phi(\pi)\|_2 \leq 1$). After $k$ steps of constant $\pi$
from any $z_0$ with $\|z_0\|_2 \leq 1$:
\[
\|z_k - z_\infty\|_2 = \gamma^k \|z_0 - z_\infty\|_2 \leq 2\gamma^k.
\]
Choosing the warm-up length as in \eqref{eq:fading-window} gives
\[
\|z_{m_{\rm warmup}} - z_\infty\|_2 \leq 1/T,
\]
and hence an \(O(L_g/T)\) error in the adversary's response per step; as above,
this costs \(O(H^2L_g)\) in total value over \(T\) episodes, a lower-order term.

\paragraph{Robustness of the confidence analysis to residual non-stationarity.}
The concentration argument (Lemma~\ref{lem:conc}) is stated for data generated
exactly under the stationary response. During data collection the adversary's
response deviates from stationarity by at most
\(\delta^{\rm ns}_t \le 2L_g\gamma^{m_{\rm warmup}} \le L_g/T\) in total
variation per step, so each episode's trajectory law deviates by at most
\(H\delta^{\rm ns}_t\). Since log-likelihood increments are uniformly bounded by
\(B_0\) (Assumption~\ref{ass:reg}(c)), replacing each episode's law by its
stationary counterpart shifts every conditional mean in the martingale argument
by at most \(2B_0H\delta^{\rm ns}_t\), and hence all cumulative quantities by at
most \(\sum_{t\le T} 2B_0HL_g/T=O(B_0HL_g)\), a \(T\)-independent constant.
Enlarging \(\beta_e\) and the termination threshold \(\tau_e(t)\) by a
sufficiently large absolute multiple of this constant restores every statement
of Lemmas~\ref{lem:conc},
\ref{lem:confidence}, \ref{lem:epoch-cap}, and~\ref{lem:switch-count}
(their martingale arguments shift by the same conditional-mean bound, and the
multiple preserves the firing margins of Lemma~\ref{lem:switch-count}). This
enlargement is the \(B_0HL_g\) term inside \(\beta^\gamma_T\) in the theorem
statement; it enters the leading term through the radius, not as additive
regret.

\paragraph{Regret decomposition.} The total regret has three components:

\textbf{Component 1 (Truncation):} $O(H^2L_g)$, independent of $T$ and absorbed into lower-order terms; the residual non-stationarity is accounted inside the enlarged radius $\beta^\gamma_T$ (robustness paragraph above), not as additive regret.

\textbf{Component 2 (Warm-up):} Each epoch requires $m_{\rm warmup} = O(\log T/(1-\gamma))$
warm-up episodes. With $S_T$ epochs started (Lemma~\ref{lem:switch-count}):
\[
\text{Warm-up regret} = S_T \cdot m_{\rm warmup} \cdot H = O\!\left(\frac{H\,S_T\log T}{1-\gamma}\right),
\]
which is $O(H\log^2T/(1-\gamma))$ when only scheduled epoch endings occur.

\textbf{Component 3 (Exploration-exploitation):} Once stabilized, the analysis is identical
to Theorem~\ref{thm:upper} with the same \(d_E\) and the enlarged radius
$\bar\beta^\gamma_T = H\beta^\gamma_T$ (which carries the $B_0m_{\rm warmup}$
floor and the $B_0HL_g$ robustness constant):
\[
\text{Statistical regret}
=
\tilde O\!\left(H\sqrt{d_E\,\bar\beta^\gamma_T\,T}\right).
\]

Combining all three yields the result,
\[
PR(T) \;\leq\; \tilde O\!\left(H\sqrt{\bar\beta^\gamma_T\,d_E T}\right) + O\!\left(\frac{H\,S_T\log T}{1-\gamma}\right).
\]

\paragraph{Domination of the warm-up term.} The algorithm runs with memory
parameter \(m=m_{\rm warmup}\), so the floor in \(\beta_e\) is
\(B_0m_{\rm warmup}\) and Lemma~\ref{lem:switch-count} applies verbatim with
this \(m\). Every termination requires
\(n_e\,\mathrm{KL}_e\ge\tfrac{3}{16}\beta_e(n)\ge\tfrac{3c}{32}(1+B_0m_{\rm warmup})\), so
\(N_{\rm term}\le\tfrac{32}{3c(1+B_0m_{\rm warmup})}\sum_e n_e\mathrm{KL}_e\), and
since \(m_{\rm warmup}B_0/(1+B_0m_{\rm warmup})\le1\), the terminated-epoch
warm-up cost obeys
\(m_{\rm warmup}H\,N_{\rm term}\le\tilde O\bigl(H\sqrt{d_E\bar\beta^\gamma_T\,T}\bigr)\).
Scheduled completions and the final epoch contribute at most
\(m_{\rm warmup}H\cdot O(\log T)=O(H\log^2T/(1-\gamma))\). Hence the total
warm-up cost is unconditionally dominated by the leading statistical term at
the radius \(\bar\beta^\gamma_T\), plus \(O(H\log^2T/(1-\gamma))\), as discussed after Theorem~\ref{thm:fading}. Choosing the warm-up per epoch as
\(\lceil\log(2^{e+1})/(1-\gamma)\rceil\) instead removes the dependence on
\(T\) with the same rate; we state the fixed-\(T\) version for simplicity.
\end{proof}

The extreme cases illustrate the dependence on the effective memory.

\textbf{Case 1: $\gamma = 0$ (memoryless adversary).}
The state becomes $z_t = \phi(\pi^t)$, and the warm-up may be set to zero
(the stated algorithm still runs $\lceil\log(2T)\rceil-1$ warm-up episodes per
epoch; setting them to zero is valid at $\gamma=0$ since the state carries no
memory). With zero warm-up the bound reduces to
\[
PR(T) = \tilde{O}\!\left(H\sqrt{\bar\beta_T\,d_E\,T}\right),
\]
matching the memoryless case up to logarithmic factors.

\textbf{Case 2: $\gamma \to 1$ (long memory).}
The warm-up length grows as $m_{\rm warmup} = O(\log T/(1-\gamma))$, and the total warm-up cost over the $S_T$ epochs is
\[
O\!\left(\frac{H\,S_T\log T}{1-\gamma}\right),
\]
capturing the increasing cost of stabilizing an adversary with slowly decaying memory.

\textbf{Case 3: Intermediate $\gamma$.}
For fixed $\gamma\in(0,1)$, the additional fading-memory cost is $O(H\,S_T\log T/(1-\gamma))$; as we discuss above, it is $O(H\log^2 T/(1-\gamma))$ outright when no early terminations occur, and in general the $B_0m_{\rm warmup}$ floor makes it dominated by the leading statistical term $\tilde{O}(H\sqrt{\bar\beta^\gamma_T\,d_E\,T})$ plus the scheduled residual $O(H\log^2T/(1-\gamma))$.

\begin{remark}[On the $1/(1-\gamma)$ dependence]
We expect the $1/(1-\gamma)$ dependence to be unavoidable, because the
adversary's internal state converges to stationarity only at rate $\gamma^t$, so
distinguishing adversaries whose stationary responses differ by $\epsilon$
should require on the order of $\log(1/\epsilon)/(1-\gamma)$ episodes of play.
Turning this heuristic into a formal policy-regret lower bound, which must also
account for the learner's ability to exploit transient responses, is left open.
\end{remark}

\section{Proof of Theorem~\ref{thm:tabular} (Tabular Bounds)}
\label{app:tabular}

For finite POMGs, we can instantiate the abstract quantities in Theorem~\ref{thm:upper}, specifically, the Eluder dimension $d_E$ and confidence parameter $\beta$, in terms of the concrete problem parameters: state space size $|S|$, action spaces $|A|$ and $|B|$, observation spaces $|O_A|$ and $|O_B|$, and so forth. This gives us an explicit regret bound that practitioners can evaluate numerically for specific problem instances.

\begin{proof}[Proof of Theorem~\ref{thm:tabular}]
Consider a tabular POMG with $|S|$ states, $|A|$ learner actions, $|B|$ adversary actions, $|O_A|$ learner observations, $|O_B|$ adversary observations, horizon $H$, adversary memory $m$, and a linear adversary response function with dimension $d_{\text{adv}}$.

\paragraph{Parameterizing the World Model} The world model consists of two components,  transition dynamics and emission distributions. At each time step $h \in [H]$, we need to specify:

\begin{enumerate}
    \item 
\textbf{Transition kernels} $T_h(s'|s,a,b)$: For each combination of current state $s$, learner action $a$, and adversary action $b$, we have a distribution over next states $s'$. This requires specifying $|S|$ probabilities for each of the $|S| \times |A| \times |B|$ combinations (with normalization removing one degree of freedom per distribution). The total number of transition parameters is thus $\Theta(|S|^2 |A| |B| H)$.

\item \textbf{Emission kernels} $E^A_h(o|s)$ and $E^B_h(o^B|s)$: For each state $s$, we have distributions over $O_A$ and over $O_B$. This requires $|S|(|O_A|+|O_B|)$ parameters per time step, giving $\Theta(|S|(|O_A|+|O_B|) H)$ total emission parameters.

\end{enumerate}

The world model thus carries $\Theta\bigl((|S|^2|A||B|+|S|(|O_A|+|O_B|))H\bigr)$ parameters; under the standard tabular scaling $|O_A|+|O_B| = O(|S||A||B|)$ the transition parameters dominate.

  \paragraph{Bounding the World Model Eluder Dimension.} In the finite-window representation invoked by
  Proposition~\ref{prop:eluder} (the exact factorization of
  Theorem~\ref{thm:decomp} lives on the full-history space), the observable
  operators act on the reduced augmented state space
  $\bar{S}^+_{\text{red}} = S \times O_A^{\kappa-1} \times O_B^{\kappa-1}$ of dimension
  $d_W = |S| \cdot |O_A|^{\kappa-1} \cdot |O_B|^{\kappa-1}$, which tracks the information needed from the last \(\kappa-1\) learner and
adversary observations. The \(|O_B|^{\kappa-1}\) factor is needed because the
adversary aggregation operator \(G^{\Phi,\pi}_h\) must be expressible without
referring to the world parameter \(\theta\).   The world model must also parameterize the adversary's emission kernel $E^B_h(o^B|s)$.

Following the span-dimension route of Proposition~\ref{prop:eluder} (cf.\ the
covering-dimension calculation of \citet{liu2022partially} for the single-agent
analogue), and summing the per-step operator complexity over
\(h\in[H]\), the tabular world-channel contribution satisfies: 
  \[
  d^{[\kappa]}_\Theta = O\!\left(H \cdot |S|^2 |A| |B| \cdot |O_A|^\kappa \cdot |O_B|^\kappa
  \cdot \log(|S| |A| |B| |O_A| |O_B| H T)\right),
  \]
  where the logarithm is the sign-witness scale factor at the evaluation scale
  $1/(2T)$. 
  The factors
\(O_A^\kappa\) and \(O_B^\kappa\) come from the \(\kappa\)-step observation
windows used by the weak-revealing representation, while the factor \(H\) comes
from summing over the \(H\) time-dependent operator classes.

\paragraph{Parameterizing and Bounding the Adversary Model.} For a linear adversary response model, the response function takes the form:
\[
g_h(b|\tau_B, \pi) = [\Phi^* w_h(\tau_B, \pi)]_b,
\]
where $\Phi^*_h \in \mathbb{R}^{|B| \times d_{\text{adv}}}$ is a column-stochastic matrix (each column sums to 1 and has non-negative entries; one matrix per step, a shared matrix only tightening the bounds) and $w_h: \mathcal{T}^B_{h-1} \times O_B \times \Pi^m \to \mathbb{R}^{d_{\text{adv}}}$ is a feature function mapping adversary decision histories and recent learner policies to features. The parameter space for $\Phi^*$ has dimension $(|B|-1) \cdot d_{\text{adv}}$ due to column-stochasticity (one degree of freedom per column is fixed by the normalization constraint).

For fixed \(\pi\) and feature map \(w_h\), the adversary aggregation operator
\(G^{\Phi,\pi}_h\) is linear in \(\Phi\). The Eluder dimension of a bounded linear
class is controlled by its parameter dimension up to logarithmic factors
\citep{russo2013eluder}. Summing over \(h\in[H]\) gives
  \[
  d^{[\kappa]}_\Psi = O(H \cdot d_{\text{adv}} \cdot |B|).
  \]
  
This reflects the fact that identifying $\Phi^*_h$ requires observing its
effect on $O(d_{\text{adv}} \cdot |B|)$ linearly independent feature vectors.

\paragraph{Joint Eluder dimension.}
The channel-sum Eluder dimension of Section~\ref{sec:formulation} is the sum of
the world-channel and adversary-channel contributions:
\[
d_E^{\rm tab}
=
d^{[\kappa]}_\Theta+d^{[\kappa]}_\Psi
=
\tilde O\!\left(
H\bigl(|S|^2|A||B||O_A|^\kappa |O_B|^\kappa+d_{\rm adv}|B|\bigr)
\right).
\]
This agrees with Proposition~\ref{prop:eluder}. The factor \(H\) reflects the
sum of the per-step operator dimensions over \(h\in[H]\). Typically, the
world-channel term dominates unless the adversary response model has very large
feature dimension.

\paragraph{Confidence radius.}
The confidence radius \(\beta_T^{\rm tab}\) must be chosen so that the true
parameter remains in the confidence set with high probability. This requires
\(\beta_T^{\rm tab}\) to scale with the logarithm of the covering number of the
joint model class.

For tabular models, the parameters are the probabilities in the transition and
emission kernels, together with the adversary-response parameters. We discretize
the corresponding probability simplices to precision \(\epsilon\), chosen small
enough that the discretization error is absorbed into the confidence radius. Taking
\(\epsilon=1/(2T+4)\), the horizon-level scale \(\epsilon_T\) matching the
dyadic schedule of Algorithm~\ref{alg:omle}, suffices for the stated bound. The world-model
covering number satisfies
\[
\log \mathcal N(\epsilon;\Theta)
=
O\!\left(
H(|S|^2|A||B|+|S||O_A|+|S||O_B|)\log(1/\epsilon)
\right)
=
\tilde O\bigl(H(|S|^2|A||B|+|S|(|O_A|{+}|O_B|))\bigr),
\]
where the emission terms, \(|S||O_A|\) and \(|S||O_B|\), are retained; under
the standard tabular scaling \(|O_A|+|O_B|=O(|S||A||B|)\) they are dominated by
the transition term. Similarly, for
the adversary response class,
\[
\log \mathcal N(\epsilon;\Psi)
=
O\!\left(
Hd_{\rm adv}|B|\log(1/\epsilon)
\right)=
\tilde O\!\left(Hd_{\rm adv}|B|\right).
\]

Combining these covering terms with the union bound over
all epochs and the failure probability \(\delta\) gives
\[
\beta_T^{\rm tab}
=
\tilde O\!\left(
H(|S|^2|A||B|+|S|(|O_A|{+}|O_B|)+d_{\rm adv}|B|)+\log(1/\delta)+mB_0
\right).
\]
Substituting this tabular confidence radius into the general effective radius
\(\bar\beta_T=H\beta_T\) gives
\[
\bar\beta_T^{\rm tab}
=
\tilde O\!\left(
H^2(|S|^2|A||B|+|S|(|O_A|{+}|O_B|)+d_{\rm adv}|B|)+HmB_0
\right),
\]
which is the quantity used in Theorem~\ref{thm:tabular}, with the additive
$H\log(1/\delta)$ from $H\beta_T^{\rm tab}$ absorbed into the $\tilde O$; under the standard
tabular scaling \(|O_A|+|O_B|=O(|S||A||B|)\) the emission term is absorbed and
the familiar two-term form is recovered.

\paragraph{Applying the Main Theorem.} Substituting \(d_E^{\rm tab}\) and \(\bar\beta_T^{\rm tab}\) into
Theorem~\ref{thm:upper} gives
\begin{align*}
PR(T)
&\le
\tilde O\!\left(
H\sqrt{\bar\beta_T^{\rm tab}\,d_E^{\rm tab}\,T}
+
mH\log T
\right)  \\
&=
\tilde O\!\Bigl(
H\bigl(H^2\bigl(|S|^2|A||B|{+}|S|(|O_A|{+}|O_B|){+}d_{\rm adv}|B|\bigr){+}HmB_0\bigr)^{1/2}
\\
&\qquad\cdot
\bigl(H\bigl(|S|^2|A||B||O_A|^\kappa |O_B|^\kappa+d_{\rm adv}|B|\bigr)\bigr)^{1/2}
\sqrt{T}
+
mH\log T
\Bigr).
\end{align*}

\paragraph{Understanding the scaling.}
The tabular bound reflects several sources of difficulty. The state and action
spaces enter through both the transition parameterization and the aggregate Eluder
dimension. In particular, the transition model contributes factors of
\(|S|^2|A||B|\), while the adversary-response class contributes
\(d_{\rm adv}|B|\), and the emission kernels contribute the additive
\(|S|(|O_A|+|O_B|)\) term in the confidence radius, absorbed under the
standard tabular scaling. The observation-window factors
\(|O_A|^\kappa |O_B|^\kappa\) appear in \(d_E^{\rm tab}\) and capture the cost of
partial observability. Larger \(\kappa\) means that longer observation windows are
needed to reveal the predictive state, increasing the effective dimension.

The horizon appears in several places. The outer factor \(H\) comes from the
simulation lemma, the factor \(H\) inside \(d_E^{\rm tab}\) comes from summing the
time-dependent operator dimensions over \(h\in[H]\), and the effective confidence
radius \(\bar\beta_T^{\rm tab}\) also carries horizon dependence. The adversary
memory \(m\) appears through the warm-up cost \((m-1)H\,S_T\), which the
\(mB_0\) floor in \(\beta_e\) makes unconditionally dominated by the leading
term plus \(O(mH\log T)\) (Lemma~\ref{lem:switch-count}), and through the
additive \(mB_0\) term in \(\beta_T^{\rm tab}\). Thus longer-memory adversaries
require longer stabilization periods, but \(m\) does not increase the Eluder
dimension, and its effect on the statistical term is only the additive
\(mB_0\) floor in the confidence radius.

\paragraph{Computational implications.}
The theorem is primarily a statistical guarantee. The optimistic planning step
requires solving
\[
\max_{(\pi,\xi)\in \Pi\times \mathcal C_{e-1}} V^\pi(\xi),
\]
which can be computationally expensive in tabular POMGs. In the worst case, the
history-dependent policy class has size exponential in \(H\), and planning in a
known finite-horizon POMG is generally intractable in \(H\). The advantage of the
epoch structure is that this optimistic planning problem is solved only once
per deployed policy, \(S_T\) times in total (Lemma~\ref{lem:switch-count}),
rather than after every episode; the per-episode work of the termination test
is a single MLE re-fit, not a new planning problem. Computationally efficient
implementations for special tabular or structured subclasses are an important
direction beyond the statistical scope of this theorem.

\paragraph{Comparison with POMDP results.}
It is instructive to compare our POMG bound with the corresponding single-agent
POMDP scaling. When the adversary component is absent, equivalently when the
adversary has a single trivial action and no strategic response, the model reduces
to a POMDP. Existing weakly revealing POMDP analyses give regret bounds of the form
\(\tilde O(\sqrt{d_{\rm POMDP}T})\), with
\[
d_{\rm POMDP}
=
\tilde O\!\left(H|S|^2|A||O_A|^\kappa\right),
\]
under the same convention of summing the operator dimension over time steps.
Our POMG bound adds:
\begin{itemize}
\item an additional factor of \(|B|\) in the world-model Eluder dimension, since
adversary actions affect transitions;
\item an adversary-observation factor \(|O_B|^\kappa\) in the predictive
representation;
\item the adversary model dimension
\(d^{[\kappa]}_\Psi=\tilde O(Hd_{\rm adv}|B|)\);
\item the warm-up cost \((m-1)H\,S_T\) from adversary stabilization, dominated by the leading term plus \(O(mH\log T)\).
\end{itemize}

Thus the POMG bound reduces to the familiar weakly revealing POMDP scaling when
the adversary component is absent, while separating the additional costs due to
adversary actions, adversary observations, response complexity, and memory.
\end{proof}

\section{Empirical Validation}
\label{app:experiments}

We report simulations that exercise the paper's three main predictions that  
\begin{enumerate}
\item data-dependent termination is necessary (Experiment 1), 
\item regret and the number of deployed policies scale as the theory says they should (Experiment
2), and 
\item the full algorithm behaves as intended against an adaptive
adversary whose environment statistics come from real data (Experiment 3).
\end{enumerate}

The code, a self-contained Python notebook, and the exact seeds are
available at \url{https://github.com/ramanarora/pomg-policy-regret}; the full
suite runs in a few CPU minutes. Following standard
practice, the schedule constants are set to practical values, with $c=1$ and
termination threshold $2\beta_e(t)$ in Experiments 1 and 2; Experiment 3 uses
the variant stated in Section~\ref{app:exp-real}. The theoretical constants are
pessimistic and affect none of the qualitative predictions.

\paragraph{Metric.} Policy regret is the quantity defined in
Section~\ref{sec:adapt}:
$PR(T)=\sum_{t\le T}\bigl[V^{\pi^*}(R_\infty(\pi^*))-v_t\bigr]$, where $v_t$
is the expected reward of the policy played in episode $t$ under the
adversary's current response and $\pi^*$ is the best fixed policy under its
own stationary response. In Experiments 1 and 2 the environment is the
bandit-embedded POMG family of Appendix~\ref{app:lower}, whose value function
is available in closed form, so both terms are computed exactly rather than
estimated from realized rewards. Figure~\ref{fig:exp-necessity} plots the
per-episode average $PR(T)/T$; all other figures plot cumulative quantities,
as their axis labels indicate. Every curve averages independent seeds (20 in
Experiments 1 and 2, 10 in Experiment 3) and shaded bands show two standard
errors of the mean, i.e., $\pm2\hat\sigma/\sqrt{n_{\rm seeds}}$
with $\hat\sigma$ the across-seed standard deviation, approximately a 95\%
confidence band for the mean curve.

\paragraph{Exact implementation of the algorithm.} For a product Bernoulli
class the optimistic MLE step of Algorithm~\ref{alg:omle} is computable in
closed form, so Experiments 1 and 2 run the algorithm itself rather than an
approximation. The MLE is the clipped empirical mean of each arm. The
optimistic value of arm $i$ is its KL upper confidence value, the largest $q$
in the class with $n_i\,\mathrm{kl}(\hat\mu_i,q)\le\beta_e$, computed by
bisection; an arm with no data attains the top of the class. Ties among
unplayed arms are broken in favor of fresh arms, which is exactly the valid
tie-breaking under which Proposition~\ref{prop:fixed-fails} operates. The
termination test also follows Algorithm~\ref{alg:omle}. The
epoch records the
deployed arm's promised optimistic value $\mu_{\rm dep}$ at selection time,
and after every data-collection episode it recomputes the running radius
$\beta_e(t)$ and ends the epoch as soon as the arm's own data certify
refutation, $n_{\rm arm}\,\mathrm{kl}(\hat\mu_{\rm arm},\mu_{\rm dep})>
2\beta_e(t)$.

\subsection{Experiment 1: Necessity of the Termination Test}

\emph{Setting.} This experiment runs the exact construction behind
Proposition~\ref{prop:fixed-fails}. There are $d=\lceil\log_2T\rceil+1$ arms. One
arm has mean $3/4$ and all others have mean $1/2$. The model class is 
the widened product family $[1/8,7/8]^d$. The horizon ranges over
$T\in\{2^8,\ldots,2^{15}\}$, so the instance itself grows with $T$. We
compare two algorithms that differ in exactly one line: the epoch-based
optimistic MLE with the termination test, and the identical algorithm with
the test removed, so that every epoch runs to its scheduled cap $2^e$.

\emph{Why the instance is hard.} The widened class is what drives the failure
mode. An arm that has never been played can optimistically be a $7/8$ arm,
which is better than anything already observed, so an optimistic learner
always prefers a fresh arm. Without the test, each such preference commits an
entire cap-$2^e$ epoch to a bad arm. With
$d=\lceil\log_2T\rceil+1$ arms, an unplayed arm exists at the start of every
epoch, so every episode is spent on an arm of true mean $1/2$ while the best
arm has mean $3/4$. The per-episode regret is therefore pinned at the gap
$1/4$, which is exactly the proposition's bound $\E[PR(T)]\ge T/4$.

\emph{Findings.} Figure~\ref{fig:exp-necessity} shows the per-episode average
$PR(T)/T$. The fixed schedule sits at $0.250$ across all horizons, matching
the proposition's $T/4$ exactly. The identical algorithm with the termination
test refutes each bad commitment within a few episodes and its average regret
falls to $0.100$ by $T=2^{15}$. The blue curve is not monotone because the
instance changes with $T$, as $d$ grows with the horizon.

\begin{figure}[h]
\centering
\includegraphics[width=0.4\textwidth]{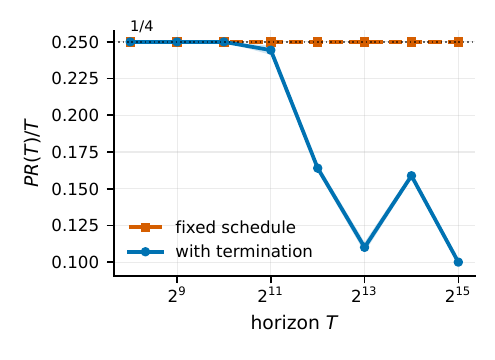}
\caption{Experiment 1: necessity of the termination test on
the hard instance of Proposition~\ref{prop:fixed-fails}, with
$d=\lceil\log_2 T\rceil+1$ arms and 20 seeds. The fixed schedule stays pinned
at average regret $1/4$; the termination test drives the average down.}
\label{fig:exp-necessity}
\end{figure}

\subsection{Experiment 2: Scaling in the Horizon, the Dimension, and the Memory}
\label{app:exp-scaling}

\emph{Setting.} All four panels of Figure~\ref{fig:exp-scaling} use the
bandit-embedded family with the full algorithm. In panels (a) and (b) there
are $d=8$ arms and the horizon sweeps $T\in\{2^8,\ldots,2^{17}\}$. The best
arm's gap is set to $0.6\sqrt{d/T}$, shrinking with the horizon. The reason
is that a fixed gap would eventually let the algorithm identify the best arm,
after which regret grows only logarithmically; shrinking the gap like
$\sqrt{d/T}$ keeps every horizon in the worst-case regime that the $\sqrt T$
theory describes. Panel (c) fixes $T=2^{14}$ and sweeps the number of arms
$d\in\{2,4,8,16,32\}$ with the same gap rule. Panels (a)--(c) use log-log axes, on which a power law
$y\propto x^k$ is a straight line of slope $k$; the fitted slopes are
least-squares fits in these coordinates. Panel (d) fixes $d=8$,
$T=2^{14}$, and a gap of $0.15$, and sweeps the adversary memory
$m\in\{1,2,4,8,16\}$. After every policy switch the next
$m-1$ episodes are warm-up and are excluded from the dataset, exactly as in
Algorithm~\ref{alg:omle}. During these episodes the simulated adversary is
still re-stabilizing, and the instance models that transient as a fixed
reduction of the learner's reward by $0.3$; the value $0.3$ is a design
parameter of the simulation (twice the gap), chosen so that warm-up episodes
carry a visible cost and the warm-up component of the regret can be measured. 

\begin{figure*}[h]
\centering
\includegraphics[width=0.9\textwidth]{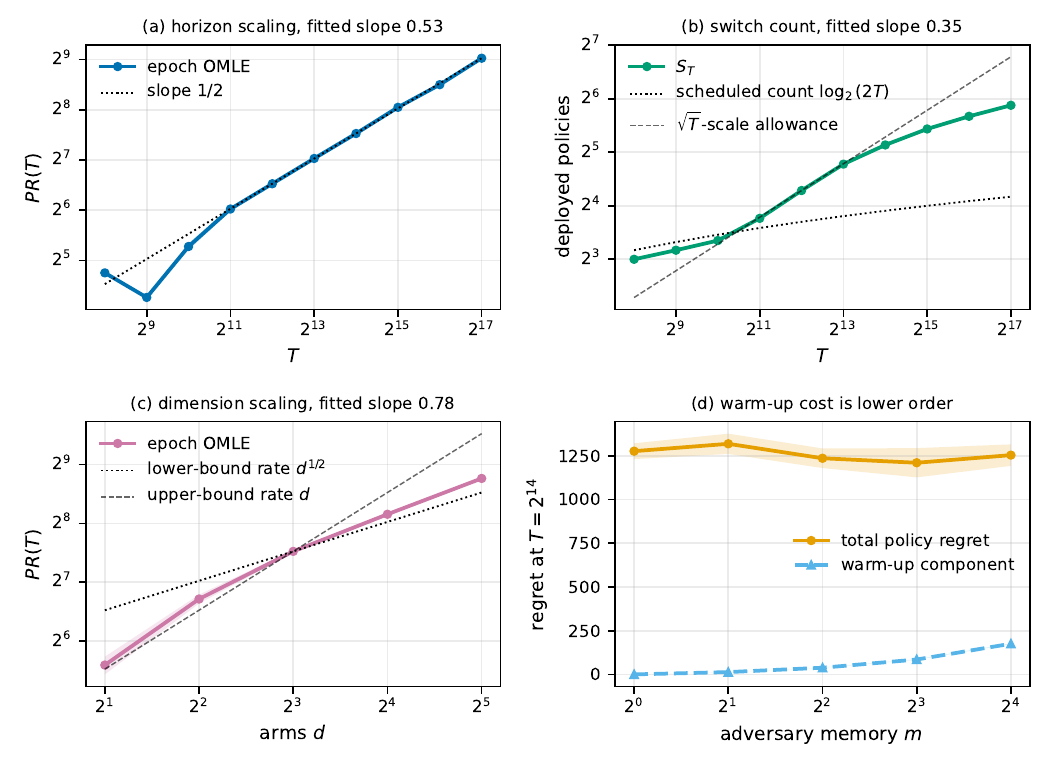}
\caption{Experiment 2: scaling behavior of the epoch-based
optimistic MLE on the bandit-embedded family (20 seeds; bands are two
standard errors). (a) Cumulative policy regret against the horizon on
minimax-regime instances with gap $0.6\sqrt{d/T}$; the fitted slope $0.53$
matches the $\sqrt T$ theory. (b) Deployed-policy count $S_T$ between the
scheduled count $\log_2(2T)$ and the $\sqrt T$-scale allowance of
Lemma~\ref{lem:switch-count}. (c) Cumulative policy regret against the number
of arms at $T=2^{14}$; the fitted slope $0.78$ lies between the reference
rates $\sqrt d$ and $d$. (d) Total policy regret and its warm-up component
against the adversary memory $m$ at $T=2^{14}$.}
\label{fig:exp-scaling}
\end{figure*}

\paragraph{Horizon scaling (panel a).} We plot cumulative policy regret
against the horizon on log-log axes, where a straight line of slope $1/2$
corresponds to $\sqrt T$ growth. The least-squares slope over the full range
is $0.53$, in close agreement with Theorem~\ref{thm:upper}.

\paragraph{Number of deployed policies (panel b).} This panel counts the
policies $S_T$ the algorithm deploys. Two reference curves frame the
measurement. The lower one is the scheduled count $\log_2(2T)$, which would
be exact if every epoch ran to its cap. The upper one grows at the $\sqrt T$
scale and is the allowance of Lemma~\ref{lem:switch-count}, which prices
every early termination in statistical evidence. The measured count exceeds
$\log_2(2T)$ once $T\ge2^{11}$, because on these shrinking-gap instances the
test keeps firing, and it grows with fitted slope $0.35$, well inside the
allowance. Both regimes described by the lemma are visible at once: scheduled
switching at small horizons, evidence-priced termination at large ones.

\paragraph{Dimension scaling (panel c).} Cumulative regret grows with the
number of arms at fitted slope $0.78$. The two natural reference rates are
$\sqrt d$, the dependence of the lower bound in Theorem~\ref{thm:lower}, and
$d$, the dependence the upper bound implies for this class, whose covering
radius satisfies $\beta_T=\Theta(d)$. The measurement falls between the two.
The gap between the reference rates reflects the known looseness of
covering-based radii in the dimension, not a defect of either bound.

\paragraph{Warm-up cost (panel d).} Total policy regret is flat as the memory
$m$ grows, while the warm-up component, the regret incurred during warm-up
episodes and tallied separately, grows linearly in $m$ and remains lower
order, about $14\%$ of the total at $m=16$. This is the
unconditional-domination behavior that the $mB_0$ floor guarantees through
Lemma~\ref{lem:switch-count}.

\subsection{Experiment 3: An Adaptive Adversary on Real Data}
\label{app:exp-real}

\emph{Task.} The learner repeatedly classifies handwritten digits against an
opponent that controls which digits it sees. In every episode the adversary
chooses a distribution over the ten digit classes, a class is drawn from it,
an image of that class is observed through a deliberately coarse channel, and
the learner's currently deployed decision rule predicts the class. The
learner is rewarded when the prediction is correct. Its choice is which of
$K=6$ fixed decision rules to deploy, and its goal is to compete with the
best single rule under the class distribution that this rule itself would
eventually induce, which is exactly policy regret against the stationary
response.

\emph{Data and observation channel.} Latent states are the ten digit classes
of the scikit-learn handwritten digits corpus (1{,}797 images of $8\times8$
pixels). The observation is deliberately coarse. Each image is
projected onto the first two principal components of the corpus, each of the
two coordinates is split into four bins at its empirical quartiles, and the
learner observes only the index of the resulting grid cell, an integer in
$\{1,\ldots,16\}$. The learner never sees the image itself. The emission
$P(o\mid s)$ is the empirical distribution of grid cells within class $s$. Partial observability therefore
comes from the genuine confusion structure of the corpus, overlapping classes
in the projection, rather than from injected noise. Panel (a) of
Figure~\ref{fig:exp-real-setup} shows the projected data and the grid.

\emph{Decision rules.} Each rule is a fixed lookup table
$h_k:\{1,\ldots,16\}\to\{0,\ldots,9\}$ from observed cells to predicted
classes; on observing cell $o$, the deployed rule predicts $h_k(o)$. The
tables are built once, before learning begins, and never change. Rule $k$ is
fit on its own random half of the corpus by assigning to each cell the
majority class among the training images falling in it, except that with
probability $0.15\,(k-1)$ the cell's assignment is replaced by a uniformly
random class (a cell with no training images also receives a random class).
Rule $1$ is thus a plain majority-vote prototype classifier and later rules
are progressively corrupted copies, so the six rules have genuinely different
accuracy profiles across classes, with mean accuracies from $0.57$ down to
about $0.24$. Panel (b) of Figure~\ref{fig:exp-real-setup} shows the full
accuracy matrix $\mathrm{acc}(k,s)$, the probability that rule $k$ correctly
classifies an image of class $s$ drawn through the channel. 

\begin{figure*}[h]
\centering
\includegraphics[width=6.2in]{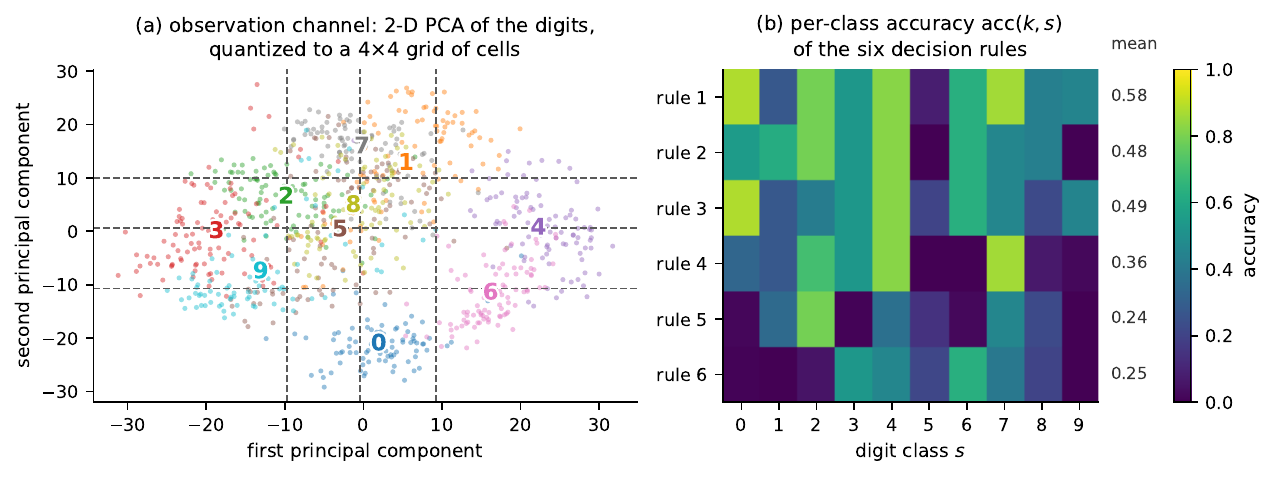}
\caption{The setup of Experiment 3. (a) The observation
channel: the digits corpus projected onto its first two principal components,
with points colored and labeled by class and the dashed lines showing the
$4\times4$ quantization grid; the learner observes only the
index of the cell containing the projected image, so
overlapping classes are genuinely confusable. (b) The per-class accuracy
matrix of the six decision rules (row means at right); the adversary tilts
the class distribution toward the columns where the learner's recent rules
are weakest.}
\label{fig:exp-real-setup}
\end{figure*}

\emph{The adversary.} A memory-$4$ adversary watches the learner's last four
deployed rules and shifts the class distribution toward the classes those
rules misclassify most, through a smooth softmax tilt,
$\rho(s)\propto\exp(\lambda\,\overline{\mathrm{err}}(s))$ with $\lambda=2$,
where $\overline{\mathrm{err}}(s)$ is the average error of the four recent
rules on class $s$. The tilt is a smooth function of the recent policies, an
instance of the posterior-Lipschitz reaction rules of
Section~\ref{sec:model}. Whenever the four recent rules are not identical the
adversary is in transit between tilts, and the learner additionally pays a
re-stabilization penalty of $0.15$; this is the transient that warm-up
exists to absorb.

\emph{Metric.} Regret is accumulated in expectation rather
than from realized rewards. Let $v^*$ be the best rule's value under its own
stationary tilt. In episode $t$ the deployed rule earns its expected accuracy
under the adversary's current class distribution, reduced by the $0.15$
penalty whenever the recent window is not constant, and the episode's regret
is $v^*$ minus this quantity. 

\emph{Algorithms.} Four algorithms are compared over $T=2^{14}$ episodes:
\begin{itemize}
\item \textbf{Termination}: the epoch-based optimistic MLE of
Algorithm~\ref{alg:omle}. Epochs are capped at $2^e$; the optimistic value of
each rule is its KL upper confidence value computed from realized rewards;
the deployed rule is abandoned by a one-sided refutation test that fires,
after at least four stable episodes, when
$n\,\mathrm{kl}(\hat\mu,\mu_{\rm dep})>2\beta_e$ with $\hat\mu$ the rule's
in-epoch mean reward and $\mu_{\rm dep}$ its promised optimistic value.
\item \textbf{Fixed}: the identical epoch algorithm with the test disabled,
so every epoch runs to its cap.
\item \textbf{$\epsilon$-greedy}: plays the rule with the best empirical mean
and explores a uniformly random rule with probability $\epsilon=0.1$ in every
episode. It has no epoch structure, so it switches constantly.
\item \textbf{Uniform}: plays a uniformly random rule in every episode, as a
floor.
\end{itemize}
The practical schedule for this experiment uses $c=0.5$ with a unit covering
term and no warm-up discard; the fixed baseline uses the identical radius, so
the comparison isolates the termination rule.

\vspace*{-10pt}
\begin{figure}[h]
\centering
\includegraphics[width=0.9\textwidth]{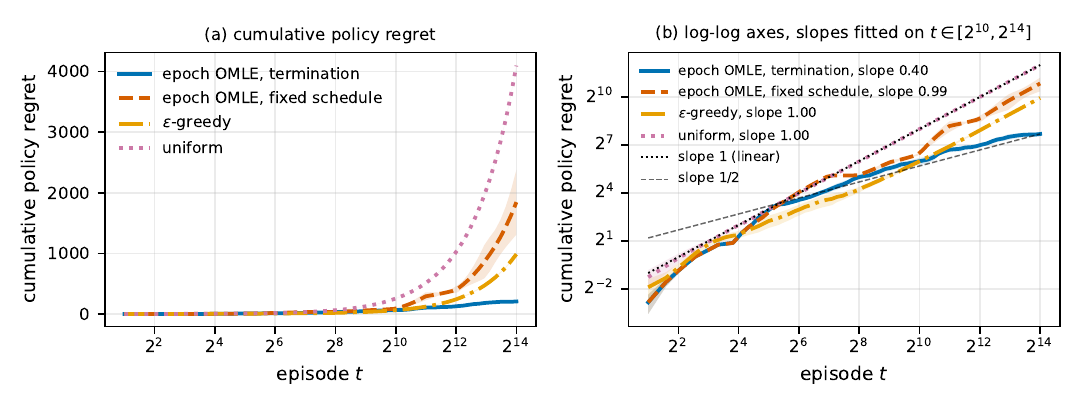}
\vspace*{-15pt}
\caption{Experiment 3: policy regret on the digits POMG, with
a memory-$4$ softmax-tilt adversary and re-stabilization transients (10
seeds; bands are two standard errors of the mean). (a) Cumulative policy
regret of the four algorithms. (b) The same curves on log-log axes, with
least-squares slopes fitted over $t\in[2^{10},2^{14}]$; the three baselines
grow linearly (slope $\approx1$), while the epoch algorithm with termination
grows at fitted slope $0.40$ and flattens as it settles on the best rule.}
\label{fig:exp-real}
\end{figure}

\emph{Findings.} Panel (a) of Figure~\ref{fig:exp-real} shows cumulative policy regret.
The epoch algorithm with termination ends at about $207$, against $1844$ for
the fixed schedule, $987$ for $\epsilon$-greedy, and $4097$ for uniform play.
The unit of policy regret here is one expected correct
classification, so the numbers have a concrete reading. Over the $16{,}384$
episodes, the termination algorithm forgoes about $207$ expected correct
predictions relative to always playing the best fixed rule under its own
stationary tilt, about $1.3\%$ of the horizon; the fixed schedule forgoes
about $11\%$, $\epsilon$-greedy about $6\%$, and uniform play about $25\%$. 
The ordering matches the theory. The fixed schedule overcommits to poorly
fitting rules for entire epochs. The $\epsilon$-greedy learner switches in
almost every episode, so the adversary never settles and the learner pays the
transient penalty nearly continuously. The epoch structure with termination
stabilizes the opponent, and still abandons refuted commitments quickly.

\emph{Trend.} Panel (b) of Figure~\ref{fig:exp-real} repeats the comparison on
log-log axes, where a power law is a straight line whose slope is the
exponent, with least-squares slopes fitted over the last four octaves
$t\in[2^{10},2^{14}]$. The three baselines grow at slope $\approx 1$, that
is, linearly, because each keeps paying a per-episode cost that never decays:
refuted commitments for the fixed schedule, transient penalties for
$\epsilon$-greedy, and both for uniform play. The termination algorithm's
fitted slope over the same window is $0.40$, below the $1/2$ of the
$\sqrt T$ theory, and its curve visibly flattens once the algorithm settles
on the best rule.

\paragraph{Scope.} Two limitations are inherent. Strategic-response data does
not exist in logged form, so the adversary is necessarily simulated even when
the environment statistics come from real data. And the experiments
instantiate the algorithm exactly only where the optimistic MLE is computable
in closed form; scaling the oracle steps to rich model classes is the
computational open problem discussed in Section~\ref{sec:extensions}.

\section{Technical Comparison with Concurrent Work}
\label{app:prior_work}

In this appendix, we provide a detailed technical analysis of the concurrent work by \citet{sang2025pomg} on policy regret minimization in POMGs. We identify fundamental issues in their problem formulation and analysis. Most critically, we show that their flat-batching analysis allows the historical
transport term to grow polynomially with \(T\), which obstructs a direct
\(\tilde O(\sqrt T)\) guarantee via the Eluder/Cauchy--Schwarz route. We explain how our work resolves these issues.

% ============================================================================

\subsection{Issue 1: Circularity in the Posterior-Lipschitz Condition}
\label{app:circularity}

The posterior-Lipschitz condition in \citet{sang2025pomg} is intended to identify
adversary responses that vary smoothly with the learner's policy. In their
formulation, for policies \(\pi,\nu\) and adversary history \(\tau_B\), the
condition takes the form
\begin{equation}
\label{eq:sang-circular-lipschitz}
\norm{g(\cdot\mid \tau_B,\pi)-g(\cdot\mid \tau_B,\nu)}_1
\le
L\max_h
\norm{S_{\tau_B}(\pi_h)-S_{\tau_B}(\nu_h)}_1,
\end{equation}
where the posterior-predictive learner policy is defined by
\begin{equation}
\label{eq:sang-circular-posterior}
S_{\tau_B}(\pi_h)
:=
\E_{\tau_A\sim P^{\pi,g,\theta}(\cdot\mid \tau_B)}
\left[
\pi_h(\cdot\mid \tau_A)
\right].
\end{equation}
The difficulty is that the right-hand side of
\eqref{eq:sang-circular-lipschitz} is computed using the same adversary response
\(g\) that the condition is trying to constrain. Thus the metric with respect to
which \(g\) is required to be Lipschitz is itself induced by \(g\). This makes the adversary class circular because, before one can check whether
\(g\) is admissible, one must already use \(g\) to define the posterior-predictive quantities appearing in
the admissibility condition.

This circularity is not merely cosmetic. The posterior-predictive quantity
\(S_{\tau_B}(\pi_h)\) is used to compare the adversary responses induced by
different learner policies, and this comparison is then invoked in the Lipschitz
transfer and constraint-propagation steps of the analysis. If the comparison
metric itself depends on the candidate adversary response, the resulting
regularity condition does not define a fixed model class over which one can form
confidence sets and perform uniform concentration.

Our formulation removes this dependence by computing posterior-predictive learner
policies under a fixed reference measure \(\muref\). Specifically, we use
\begin{equation}
\label{eq:our-reference-posterior}
S_{\tau_B,h}^{\mathrm{ref},\theta}(\pi)
:=
\E_{\tau_A\sim P^{\pi,\muref,\theta}(\cdot\mid \tau_B)}
\left[
\pi_h(\cdot\mid \tau_A)
\right],
\end{equation}
and state posterior-Lipschitzness with
\(S_{\tau_B,h}^{\mathrm{ref},\theta}\) in place of \(S_{\tau_B}\). The reference
measure \(\muref\) is fixed independently of the adversary response being
constrained, so the right-hand side of the Lipschitz condition is well-defined
before choosing \(g\). The adversary class is therefore specified non-circularly,
while still measuring how the learner's policy appears from the adversary's
information state.
% ============================================================================

\subsection{Issue 2: A $T^{3/4}$-Scale Obstruction from the Batched Structure}
\label{app:degraded_bound}

Beyond the circularity issue, the flat-batching structure in
\citet{sang2025pomg} allows the historical transport term to grow polynomially
with \(T\). The point is not merely that many batches are used, but that the
posterior-Lipschitz transport term accumulates over previously deployed policies.
When this term is carried through the same weighted Eluder/Cauchy--Schwarz route
used to convert prediction errors into regret, it creates a \(T^{3/4}\)-scale
contribution.

Their algorithm uses a flat batching schedule with
\begin{equation}
\label{eq:sang-batch-schedule}
K=\sqrt{d_E T},
\qquad
n=\frac{T}{K}=\sqrt{\frac{T}{d_E}},
\end{equation}
where \(K\) is the number of batches, \(n\) is the number of data-collection
episodes per batch, and we write \(d_E\) for the corresponding complexity
parameter of their analysis.

A key quantity in their transport analysis is the historical accumulation term
\begin{equation}
\label{eq:sang-gamma-def}
\Gamma_{j-1}
:=
\sum_{k=1}^{j-1}
\sum_{h=1}^H
\E\left[
\norm{q^{\mathrm{mid},*}_h(\pi_k)}_1^2
\right].
\end{equation}
This is the analogue of the transport term in our analysis. It measures how much
historical predictive mass must be carried when an adversary-channel error is
transported from the historical policies \(\pi_k\) to the current batch policy.
In their Lemma 10, this term appears in the adversary-channel constraint as
\begin{equation}
\label{eq:sang-transport}
\sum_{k=1}^{j-1}
\sum_{h=1}^H
\E\left[
\norm{
\left[
G^{\Phi_j,\pi}_h
-
G^{\Phi^*,\pi}_h
\right]
q^{\mathrm{mid},*}_h
}_1^2
\right]
\le
C\beta
+
C'\Delta_\sigma(\pi,\nu_j)^2\Gamma_{j-1}.
\end{equation}
Since \(q^{\mathrm{mid},*}_h\) is a predictive state, its \(\ell_1\)-norm is
\(O(1)\). Therefore
\begin{equation}
\label{eq:sang-gamma-growth}
\Gamma_j
=
\sum_{k=1}^{j}
\sum_{h=1}^H
\E\left[
\norm{q^{\mathrm{mid},*}_h(\pi_k)}_1^2
\right]
=
O(jH).
\end{equation}
At the final batch, using \eqref{eq:sang-batch-schedule},
\begin{equation}
\label{eq:sang-gamma-final}
\Gamma_K
=
O(KH)
=
O\!\left(H\sqrt{d_E T}\right).
\end{equation}
Thus the historical transport term grows as \(\sqrt{T}\), rather than
polylogarithmically.

In the final assembly of their analysis, this transport contribution is absorbed
into a constant described as depending on ``the covering radius used to select a
nearest historical policy''; no argument is given that optimistically selected
policies form such a net, so we track the term explicitly below.

We next spell out how this growth affects the regret conversion. For batch \(j\),
define the per-episode adversary-channel prediction error
\begin{equation}
\label{eq:sang-Aj-def}
A_j^2
:=
\sum_{h=1}^H
\E\left[
\norm{
\left[
G^{\Phi_j,\pi_j}_h
-
G^{\Phi^*,\pi_j}_h
\right]
q^{\mathrm{mid},*}_h
}_1^2
\right].
\end{equation}
Equivalently, the batch-weighted squared error is \(nA_j^2\). This matches the
episode-weighted quantities used in the regret analysis, since batch \(j\)
contributes \(n\) data-collection episodes under the same deployed policy
\(\pi_j\).

The adversary-channel part of the value error in one episode of batch \(j\) is
bounded by \(H\sqrt H\, A_j\), where the factor \(H\) is the simulation-lemma
horizon factor and the \(\sqrt H\) is the Cauchy--Schwarz conversion over the
\(H\) steps, since \(A_j\) is the root of a stepwise \emph{squared} sum. Hence the
adversary-channel contribution to regret satisfies
\begin{equation}
\label{eq:sang-regret-before-cs}
\mathrm{Regret}_G
\le
H^{3/2}\sum_{j=1}^K nA_j .
\end{equation}
Applying Cauchy--Schwarz in the weighted form gives
\begin{equation}
\label{eq:sang-weighted-cs}
\sum_{j=1}^K nA_j
\le
\sqrt{\sum_{j=1}^K n}\,
\sqrt{\sum_{j=1}^K nA_j^2}
=
\sqrt{T}\,
\sqrt{\sum_{j=1}^K nA_j^2},
\end{equation}
where we used \(Kn=T\).

Thus the relevant quantity for regret is the weighted square-error sum
\(\sum_j nA_j^2\). Granting their weighted Eluder summability step (which, as
Proposition~\ref{prop:fixed-fails} shows, itself requires a per-batch cap that
a fixed batch length does not provide) and carrying the transport term in
\eqref{eq:sang-transport} through it gives a bound of the form
\begin{equation}
\label{eq:sang-weighted-eluder}
\sum_{j=1}^K nA_j^2
\le
\widetilde O\!\left(d_E(\beta+\Gamma_K)\right).
\end{equation}
Substituting \eqref{eq:sang-gamma-final} into
\eqref{eq:sang-weighted-eluder}, and ignoring the lower-order confidence term
\(\beta\), yields
\begin{equation}
\label{eq:sang-weighted-error-final}
\sum_{j=1}^K nA_j^2
\le
\widetilde O\!\left(d_E\cdot H\sqrt{d_E T}\right).
\end{equation}
Combining \eqref{eq:sang-regret-before-cs},
\eqref{eq:sang-weighted-cs}, and \eqref{eq:sang-weighted-error-final}, we obtain
\begin{align}
\mathrm{Regret}_G
&\le
H^{3/2}\sqrt{T}\,
\sqrt{
\widetilde O\!\left(d_E\cdot H\sqrt{d_E T}\right)
}
\label{eq:sang-regret-degraded-start}
\\
&=
\widetilde O\!\left(
H^{2}d_E^{3/4}T^{3/4}
\right).
\label{eq:sang-regret-degraded-final}
\end{align}
Therefore, once the historical transport term \(\Gamma_K\) is retained in the
weighted square-error bound, the final Cauchy--Schwarz step produces a
\(T^{3/4}\)-scale contribution.

To summarize, the calculation above isolates the obstruction caused by flat batching. With
\(K=\sqrt{d_E T}\) batches, the transport term satisfies
\(\Gamma_K=O(H\sqrt{d_E T})\) by \eqref{eq:sang-gamma-final}. Substituting this
growth into the weighted Eluder bound \eqref{eq:sang-weighted-eluder} and then
applying the regret conversion \eqref{eq:sang-weighted-cs} gives the
\(T^{3/4}\)-scale contribution in \eqref{eq:sang-regret-degraded-final}. Thus the
difficulty is not a logarithmic artifact. The flat-batching structure allows the
historical transport term to grow polynomially with \(T\), which obstructs a
direct \(\widetilde O(\sqrt T)\) guarantee by this analysis route.

\subsection{Our Solution: One-Policy-Per-Epoch Achieves $\tilde{O}(\sqrt{T})$}
\label{app:our_solution}

We now summarize how our algorithmic design avoids the two difficulties discussed
above. The first issue is handled at the modeling level by defining the
posterior-predictive learner policy with respect to the fixed reference measure
\(\muref\), as in \eqref{eq:our-reference-posterior}. This makes the
posterior-Lipschitz condition a non-circular condition on the adversary response
class. The second issue is handled algorithmically by using epochs rather than a
flat batch schedule.

Our algorithm caps epochs at geometrically increasing lengths
\(T_e=2^e\) and, crucially, terminates an epoch early as soon as the collected
data refute the deployed optimistic model. At the beginning of each epoch, the
learner forms a confidence set from all previously collected data and selects
one optimistic policy \(\pi_e\), which it plays until the cap or the
termination test ends the epoch. Concretely, after every data-collection episode the learner recomputes the running MLE and terminates the epoch as soon as the deployed model's log-likelihood falls more than \(\tau_e(t)=2\beta_e(t)+C_0B_0\log(2t(t{+}1)/\delta_e)\) below it, so a refuted commitment is abandoned within the episodes needed to certify the refutation. Two consequences follow. First, under the
joint KL-to-operator conversion (Assumption~\ref{ass:reg}(e)) our analysis
requires no policy-transport term at all. Past consistency is established
directly at the data-collection policies, the termination test caps every
epoch's contribution at \(\tilde O(H\beta_T)\)
(Lemma~\ref{lem:epoch-cap}), and the Eluder summability argument
(Lemma~\ref{lem:eluder-total}) yields
\begin{equation}
\label{eq:our-eluder-total}
\sum_{e} n_e\Delta_e^2
\le
\widetilde O(d_E\bar\beta_T),
\qquad
\bar\beta_T := H\beta_T .
\end{equation}
Applying Cauchy--Schwarz with \(\sum_e n_e\le T\),
\begin{equation}
\label{eq:our-cs}
\sum_{e} n_e\Delta_e
\le
\sqrt{
\left(\sum_e n_e\right)
\left(\sum_e n_e\Delta_e^2\right)}
\le
\widetilde O\!\left(\sqrt{d_E T\bar\beta_T}\right).
\end{equation}
The simulation lemma contributes the outer factor \(H\), and the warm-up
episodes contribute \((m-1)H\,S_T\), which the \(mB_0\) floor in \(\beta_e\)
makes unconditionally dominated by the leading term plus \(O(mH\log T)\)
(Lemma~\ref{lem:switch-count}). Therefore,
\begin{equation}
\label{eq:our-final-regret}
PR(T)
\le
\widetilde O\!\left(
H\sqrt{d_E T\bar\beta_T}
+
mH\log T
\right).
\end{equation}
Second, even from the transport viewpoint relevant to the flat-batching route
(Remark~\ref{rem:transport-role}), our \emph{scheduled} switches number only
\(O(\log T)\), so any transport accounting over them stays logarithmic; the
data-driven terminations are not free in that accounting, but our proof never
needs that accounting, since the per-epoch cap replaces transport entirely.
This is the substantive difference from a flat schedule with
\(K=\sqrt{d_ET}\) batches, where every one of the \(K\) switches is structural
and the transport term necessarily accumulates to
\(\Gamma_K=O(H\sqrt{d_ET})\).

The essential difference is that a flat schedule with polynomially many
deployed policies both forgoes the per-epoch cap (each batch runs its full
length regardless of what the data say) and, on the transport route, lets the
historical term grow polynomially in \(T\), as shown in
Appendix~\ref{app:degraded_bound}.

In summary, our reference posterior-predictive formulation removes the
circularity in the posterior-Lipschitz condition, and our epoch design, with
geometric caps and statistical termination, supplies the per-epoch cap that
makes infrequently updated optimistic planning sound while keeping policy
switches rare. These changes are
what allow the optimistic MLE analysis to recover the
\(\widetilde O(\sqrt T)\) policy-regret rate.

\end{document}